\documentclass{article} % For LaTeX2e
\usepackage{iclr2026_conference,times}

\usepackage{amsmath,amsfonts,bm}

\def\eqref#1{equation~\ref{#1}}
\def\1{\bm{1}}

\DeclareMathAlphabet{\mathsfit}{\encodingdefault}{\sfdefault}{m}{sl}
\SetMathAlphabet{\mathsfit}{bold}{\encodingdefault}{\sfdefault}{bx}{n}

\newcommand{\E}{\mathbb{E}}

\newcommand{\R}{\mathbb{R}}

\newcommand{\KL}{D_{\mathrm{KL}}}
\newcommand{\JS}{D_{\mathrm{JS}}}

\DeclareMathOperator*{\argmax}{arg\,max}

\usepackage[utf8]{inputenc} % allow utf-8 input
\usepackage[T1]{fontenc}    % use 8-bit T1 fonts
\usepackage{microtype}
\usepackage{times}
\usepackage{graphicx}
\usepackage{amsfonts,amsmath,amssymb,amsthm,mathtools,systeme}
\usepackage{bm}
\usepackage{mathrsfs}
\usepackage{acronym}
\usepackage{enumitem}
\usepackage[pagebackref,breaklinks,colorlinks,citecolor=blue]{hyperref}
\usepackage{xspace}
\usepackage{setspace}
\usepackage{subcaption}
\usepackage[skip=0pt,font=small]{caption}
\usepackage{wrapfig}      
\usepackage[dvipsnames]{xcolor}
\usepackage[capitalise]{cleveref}
\usepackage{booktabs,tabularx,colortbl,multirow,array,makecell}
\usepackage{float}
\usepackage{appendix}
\usepackage{algorithm}
\usepackage{algorithmic}

\makeatletter
\DeclareRobustCommand\onedot{\futurelet\@let@token\@onedot}
\def\@onedot{\ifx\@let@token.\else.\null\fi\xspace}
\def\eg{\emph{e.g}\onedot} 

\def\ie{\emph{i.e}\onedot}

\def\etc{\emph{etc}\onedot}

\makeatother

\DeclareMathOperator{\ELBO}{\mathrm{ELBO}}

\DeclareMathOperator{\LL}{\mathcal{L}}

\newcommand{\I}{\mathbf{I}}
\newcommand{\N}{\mathcal{N}}

\newcommand{\x}{\mathbf{x}}

\newcommand{\z}{\mathbf{z}}
\newcommand{\cD}{\mathcal{D}}
\newcommand{\cX}{\mathcal{X}}
\newcommand{\cJ}{\mathcal{J}}

\newcommand{\bA}{\boldsymbol{\alpha}}
\newcommand{\bB}{\boldsymbol{\beta}}

\newcommand{\bM}{\boldsymbol{\mu}}
\newcommand{\bP}{\boldsymbol{\phi}}
\newcommand{\bPs}{\boldsymbol{\psi}}
\newcommand{\bS}{\boldsymbol{\sigma}}
\newcommand{\bT}{\boldsymbol{\theta}}
\newcommand{\bTau}{\boldsymbol{\tau}}

\usepackage{aliascnt}

\newtheorem{theorem}{Theorem}[section]
\newtheorem{assumption}[theorem]{Assumption}
\newtheorem{lemma}[theorem]{Lemma}

\newaliascnt{proposition}{theorem}
\newtheorem{proposition}[proposition]{Proposition}
\aliascntresetthe{proposition}

\crefname{theorem}{Theorem}{Theorems}
\Crefname{theorem}{Theorem}{Theorems}
\crefname{proposition}{Proposition}{Propositions}
\Crefname{proposition}{Proposition}{Propositions}
\crefname{lemma}{Lemma}{Lemmas}
\Crefname{lemma}{Lemma}{Lemmas}
\crefname{assumption}{Assumption}{Assumptions}
\Crefname{assumption}{Assumption}{Assumptions}

\usepackage[framemethod=TikZ]{mdframed}
\definecolor{lightgraybox}{rgb}{0.95,0.95,0.95}

\mdfdefinestyle{lighttheobox}{
  backgroundcolor=lightgraybox,
  linecolor=black,
  linewidth=0.5pt,
  roundcorner=3pt,
  skipabove=\baselineskip,
  skipbelow=\baselineskip,
  innertopmargin=6pt,
  innerbottommargin=6pt,
  innerleftmargin=8pt,
  innerrightmargin=8pt
}

\surroundwithmdframed[style=lighttheobox]{theorem}
\surroundwithmdframed[style=lighttheobox]{proposition}
\surroundwithmdframed[style=lighttheobox]{lemma}

\makeatletter
\renewcommand{\paragraph}{%
  \@startsection{paragraph}{4}%
  {\z@}{0ex \@plus 0ex \@minus 0ex}{-1em}%
  {\hskip\parindent\normalfont\normalsize\bfseries}%
}
\makeatother

\graphicspath{{figures/}}

\crefname{algorithm}{Alg.}{Algs.}
\Crefname{algocf}{Algorithm}{Algorithms}
\crefname{section}{Sec.}{Secs.}
\Crefname{section}{Section}{Sections}
\crefname{table}{Tab.}{Tabs.}
\Crefname{table}{Table}{Tables}
\crefname{figure}{Fig.}{Fig.}
\Crefname{figure}{Figure}{Figure}
\crefname{appendix}{Appx.}{Appx.}
\Crefname{appendix}{Appendix}{Appendices}
\crefname{subfigure}{Fig.}{Figs.}
\Crefname{subfigure}{Figure}{Figures}
\crefformat{subfigure}{#2Fig.~\thefigure\subref{#1}#3}
\Crefformat{subfigure}{#2Figure~\thefigure\subref{#1}#3}

\definecolor{gblue}{HTML}{4285F4}
\definecolor{gred}{HTML}{DB4437}
\definecolor{ggreen}{HTML}{0F9D58}
\definecolor{gray}{gray}{0.9}
\definecolor{tgray}{gray}{0.5}

\usepackage{tikz}
\usetikzlibrary{positioning, arrows.meta, quotes}
\usetikzlibrary{shapes,decorations}
\usetikzlibrary{bayesnet}
\tikzset{>=latex}
\tikzstyle{plate caption} = [
    caption, node distance=0, inner sep=0pt,
    below right=0pt and 0pt of #1.south
]
\tikzstyle{nbase} = [
    circle, 
    draw=black, 
    fill=white, 
    inner sep=0pt, 
    minimum size=0.8cm,
    node distance=16mm
]

\acrodef{bo}[BO]{Bayesian Optimization}
\acrodef{bbo}[BBO]{Black-Box Optimization}
\acrodef{ddpm}[DDPM]{Denoising Diffusion Probabilistic Model}
\acrodef{dlvm}[DLVM]{Deep Latent Variable Model}
\acrodef{ebm}[EBM]{Energy-Based Model}
\acrodef{ga}[GA]{Gradient Ascent}
\acrodef{gan}[GAN]{Generative Adversarial Network}
\acrodef{kde}[KDE]{Kernel Density Estimation}
\acrodef{kld}[KLD]{Kullbeck-Leibler Divergence}
\acrodef{lebm}[LEBM]{Latent-space Energy-Based Model}
\acrodef{leo}[LEO]{\textbf{\underline{L}}atent \textbf{\underline{E}}nergy-based \textbf{\underline{O}}ptimization}
\acrodef{ld}[LD]{Langevin Dynamics}
\acrodef{ldebm}[LDEBM]{Latent Diffusion Energy-Based Model}
\acrodef{mbo}[MBO]{Model-Based Optimization}
\acrodef{mcmc}[MCMC]{Markov Chain Monte Carlo}
\acrodef{mle}[MLE]{Maximum Likelihood Estimation}
\acrodef{nce}[NCE]{Noise Contrastive Estimation}
\acrodef{n2ce}[N$^2$CE]{``\textbf{\underline{N}}oisier'' \textbf{\underline{N}}oise \textbf{\underline{C}}ontrastive \textbf{\underline{E}}stimation}
\acrodef{nwj}[NWJ]{Nguyen-Wainwright-Jordan}
\acrodef{poe}[PoE]{Product-of-Expert}
\acrodef{svgd}[SVGD]{Stein Variational Gradient Descent}
\acrodef{tre}[TRE]{Telescoping Density Ratio Estimation}
\acrodef{vae}[VAE]{Variational Auto-Encoder}

\newcolumntype{C}{>{\columncolor{gray}}c}
\newcolumntype{L}{>{\columncolor{gray}}l}
\newcommand{\tabincell}[2]{\begin{tabular}{@{}#1@{}}#2\end{tabular}}
\newcommand{\iters}[1]{\textit{\,(#1)}}

\title{``Noisier'' Noise Contrastive Estimation is (Almost) Maximum Likelihood}

\author{%
    Peiyu Yu$^1$\thanks{Equal Contribution. Code and data available at \href{https://github.com/yuPeiyu98/Noisier-NCE}{https://github.com/yuPeiyu98/Noisier-NCE}.} \\
    \And Dinghuai Zhang$^{2*}$ \\
    \And Hengzhi He$^{1*}$ \\
    \And Xiaojian Ma$^{4}$ \\
    \And Sirui Xie$^1$ \\
    \And Ruiyao Miao$^1$ \\
    \And Yifan Lu$^1$ \\
    \And Yasi Zhang$^1$ \\
    \And Deqian Kong$^1$ \\
    \And Ruiqi Gao$^{3}$ \\
    \And Jianwen Xie$^{5}$ \\
    \AND
    \begin{tabular}[t]{c}
    \bfseries Guang Cheng$^1$ \qquad Ying Nian Wu$^1$
    \end{tabular}
    \AND
    \normalfont{$^1$University of California, Los Angeles\quad{}}
    \normalfont{$^2$Mila - Quebec AI Institute\quad{}}
    $^3$Google DeepMind \\
    $^4$Beijing Institute for General Artificial Intelligence (BIGAI) \quad{}
    $^5$Akool Research
}

\iclrfinalcopy % Uncomment for camera-ready version, but NOT for submission.
\begin{document}

\maketitle

\begin{abstract}
\acf{nce} has fueled major breakthroughs in representation learning and generative modeling. Yet a long-standing challenge remains: accurately estimating ratios between distributions that differ substantially, which significantly limits the applicability of \ac{nce} on modern high-dimensional and multimodal datasets. We revisit this problem from a less explored perspective: the magnitude of the noise distribution. Specifically, we show that with a virtually scaled (\ie, artificially increased) noise magnitude, the gradient of the \ac{nce} objective can closely align with that of Maximum Likelihood, enabling a trajectory-wise approximation from \ac{nce} to \acs{mle}, and faster convergence both theoretically and empirically. Building on this insight, we introduce ``Noisier'' \ac{nce}, a simple drop-in modification to vanilla \ac{nce} that incurs little to no extra computational cost, while effectively handling density-ratio estimation in challenging regimes where traditional MLE and \ac{nce} struggle. Beyond improving classical density-ratio learning, ``Noisier'' \ac{nce} proves broadly applicable: it achieves strong results across image modeling, anomaly detection, and offline black-box optimization. On CIFAR-10 and ImageNet64×64 datasets, it yields 10-step and even 1-step samplers that match or surpass state-of-the-art methods, while cutting training iterations by up to half.
\end{abstract}

\section{Introduction}
\label{sec:intro}
Generative modeling of data distributions has made impressive progress in recent years, driven by the development of a rich family of deep generative models~\citep{ho2020denoising,song2020score,rombach2022high,karras2022elucidating,lipman2022flow,liu2022flow}. Among existing paradigms, \acf{nce} \citep{gutmann2010noise,gutmann2012noise}, also referred to as density–ratio estimation in some literature~\citep{sugiyama2012density}, is a powerful and fundamental framework that unifies generative modeling with discriminative representation learning~\citep{karras2019style,chen2020simple,radford2021learning,huang2025gan}. \ac{nce} reduces density estimation to a classification task: learning the ratio
\(
r(\x)=q_{*}(\x)/q_{0}(\x)
\)
between a target distribution \(q_{*}\) and a noise distribution \(q_{0}\) from samples of each, thereby enabling density estimation without explicitly modeling the target density itself~\citep{gutmann2010noise,sugiyama2012density}. 

\begin{figure}[t!]
    \centering
    % ---- subfigure (a) ----
    \begin{subfigure}[t]{0.48\linewidth}
        \centering
        \includegraphics[width=\linewidth]{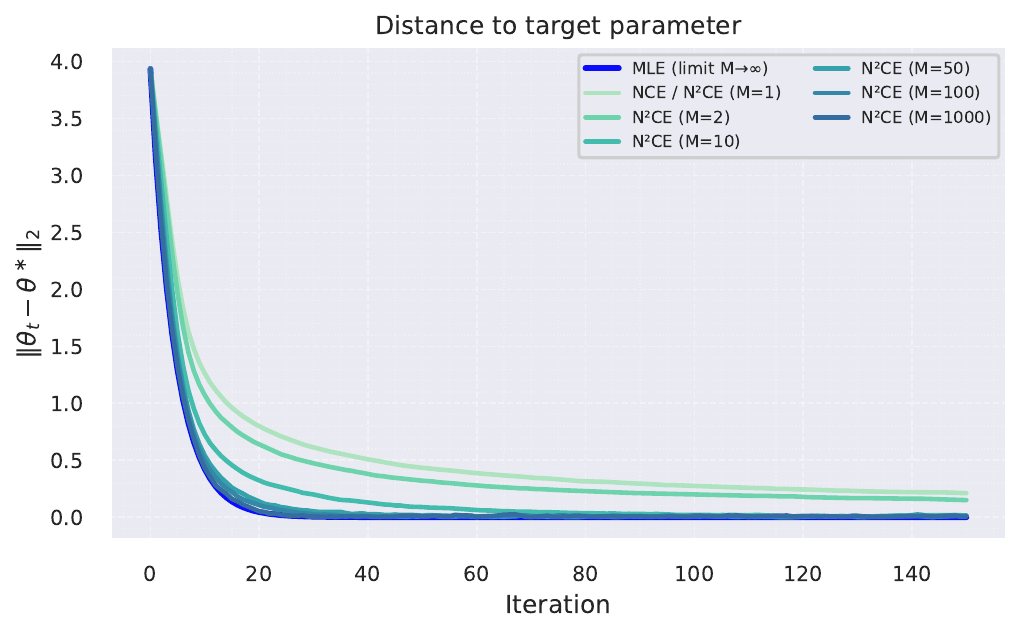}
        \caption{}
        \label{subfig:traj_equiv}
    \end{subfigure}
    \hfill
    % ---- subfigure (b) ----
    \begin{subfigure}[t]{0.48\linewidth}
        \centering
        \includegraphics[width=\linewidth]{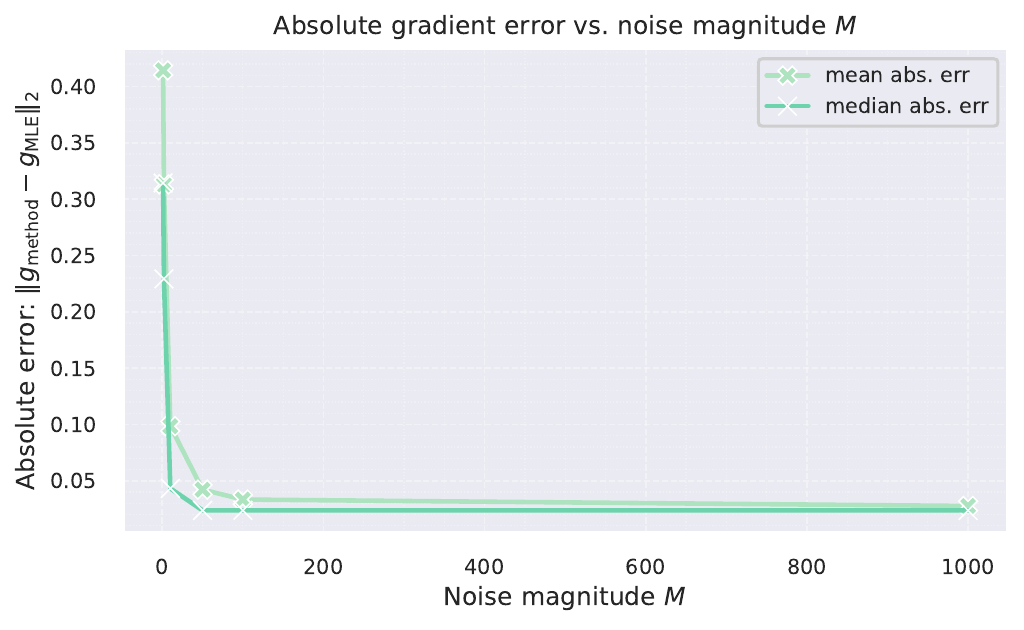}
        \caption{}
        \label{subfig:bias_decay}
    \end{subfigure}
    
    \caption{\textbf{``Noisier'' \acs{nce} gradients approach the \acs{mle} gradients.} 
    As a sanity check, we simulate the results using 2d Gaussian distributions; \textit{true} \acs{mle} gradients can be analytically computed. In \cref{subfig:traj_equiv}, we can see that $M \to \infty$ leads to a trajectory-wise convergence from \ac{nce} to \acs{mle}.
    In \cref{subfig:bias_decay}, as the noise magnitude $M$ increases, ``noisier'' \acs{nce} gradients $\nabla_{\bA}\LL_{{M}}^{\textcolor{cyan!50!green}{{\rm NCE}}}$ approach \acs{mle} gradients $\nabla_{\bA} \cJ^{\rm \color{blue} {MLE}}$; bias decaying in the order of $O(\tfrac{1}{M^2})$, which is consistent with \cref{prop:grad_finite_ver}. Further details can be found in \cref{appx:gauss_simu}.}
    \label{fig:gauss_simu}
\end{figure}

Despite its wide adoption, most existing \ac{nce}-based objectives suffer from a fundamental drawback: when the target and noise distributions differ substantially, the neural classifier can achieve near-perfect discrimination accuracy while still providing a poor estimate of the density ratio. This issue, often referred to as the \textit{density-chasm}~\citep{rhodes2020telescoping}, arises when the gap between the two distributions is large, \eg, when the KL divergence between $q_{*}$ and $q_{0}$ exceeds tens of nats. Such situations are common in modern high-dimensional or highly multimodal datasets. Although NCE estimators are asymptotically consistent, meaning they recover the true density ratio in the infinite-sample limit, the convergence rate is provably slow: even an exponential increase in sample size yields only a linear decrease in estimation error~\citep{poole2019variational,mcallester2020formal}, and the issue persists even with infinite data~\citep{liu2021analyzing}. 

In this paper, we revisit \ac{nce} from a less explored perspective: the magnitude of the noise distribution used for ratio estimation. We show that, under mild conditions, introducing a virtually scaled noise magnitude allows the gradient of the \ac{nce} objective to align with that of \acf{mle}. This observation has two key implications. First, it establishes a principled \emph{gradient-level} connection between \ac{nce} and \ac{mle}, where the \ac{nce} gradient serves as a controlled approximation to the \ac{mle} gradient. This perspective highlights that \ac{nce} can be understood as approximating the \ac{mle} in terms of the optimization trajectory rather than merely matching asymptotic error rates established by \cite{gutmann2012noise}. Second, it offers a new lens on the density-chasm problem, revealing that a ``noisier'' \ac{nce} naturally mitigates convergence issues. Building on this insight, we propose \ac{n2ce}, a simple drop-in modification to vanilla \ac{nce} that requires little to no additional computational cost. As we show in \cref{sec:exps}, this framework proves broadly applicable, demonstrating strong empirical performance in various downstream tasks 
% such as image modeling, anomaly detection, and offline black-box optimization 
(see \cref{fig:gauss_simu} for an illustrative example).

Our main \textbf{contributions} are as follows: i) \textit{Theoretically}, we show that, under mild conditions, a ``noisier'' \ac{nce} objective yields gradients that align with those of maximum likelihood, thus bridging the gap between \ac{nce} and \ac{mle} at the optimization trajectory level. Further, our analysis reveals that a virtually scaled noise magnitude naturally alleviates the convergence issues that arise when the target and noise distributions differ substantially. ii) \textit{Practically}, we propose \ac{n2ce}, a drop-in modification to vanilla \ac{nce} with negligible computational overhead. \ac{n2ce} achieves strong empirical performance across diverse tasks including image modeling, anomaly detection, and offline black-box optimization. It produces 10-step and even 1-step samplers that match or surpass state-of-the-art methods on the CIFAR-10 and ImageNet64×64 datasets, with up to half the training cost.

\section{Background}
\label{sec:background}
\paragraph{\texorpdfstring{\acf{nce}}{} and \texorpdfstring{\acf{nwj}}{}}
\acf{nce}~\citep{gutmann2010noise,gutmann2012noise} offers a discriminative route to training unnormalized statistical models, sidestepping the need for direct likelihood maximization. Let $q_{*}$ denote the target distribution and $q_{0}$ a known noise distribution. Rather than directly modeling $p_{\bA}$ to approximate $q_{*}$, \ac{nce} estimates the ratio
$
r_{\bA}(\x) = \frac{p_{\bA}(\x)}{q_{0}(\x)}
$
by training a classifier to distinguish between samples from $q_{*}$ and $q_{0}$. When the model is correctly specified, the optimal ratio satisfies $r_{\bA^{*}} \approx q_{*}/q_{0}$, which in turn implies $p_{\bA^{*}} \approx q_{*}$. In practice, $r_{\bA}$ is often parameterized as the exponential of a scalar-output neural network $\tilde f_{\bA}$. This parameterization naturally aligns with the energy-based view of $p_{\bA}$~\citep{lecun2006tutorial,xie2016theory,du2019implicit}, where $\tilde f_{\bA}$ incorporates the partition function. A common choice of classification loss is the logistic loss:
\begin{equation}
\label{equ:nce_obj}
\LL(\bA)
= \E_{\x\sim q_*}\!\left[ \log \frac{r_{\bA}(\x)}{1 + r_{\bA}(\x)} \right]
+ \E_{\x\sim q_0}\!\left[ \log \frac{1}{1 + r_{\bA}(\x)} \right].
\end{equation}
Numerous extensions of NCE have been developed, including generalized loss functions~\citep{pihlaja2010family,sugiyama2012density,menon2016linking,poole2019variational,liu2021analyzing} and multi-stage ratio estimation strategies~\citep{rhodes2020telescoping,xiao2022adaptive}, which greatly enrich the paradigm. In this work, we revisit NCE from the perspective of the noise distribution’s magnitude, showing that virtually scaling the noise magnitude offers a simple, effective remedy to long-standing convergence issues with little to no computational overhead.

Being equally fundamental and seminal, \acf{nwj}, addresses the ratio–estimation problem through a variational characterization of $f$-divergences~\citep{nguyen2010estimating}. This variational view reveals that likelihood–ratio estimation can be performed by maximizing a convex risk functional, yielding M-estimators that are statistically consistent and minimax-optimal under suitable regularity assumptions. Algorithmically, \ac{nwj} induces a tractable objective of the form
\begin{equation}
\label{equ:nwj_obj}
\LL^{\rm NWJ}(\bA)
= \E_{\x\sim q_*}\!\left[T_{\bA}(\x) \right]
+ \E_{\x\sim q_0}\!\left[\exp{T_{\bA}(\x)} \right],
\end{equation}
where $T_{\bA}=\log r_{\bA}$. Provably, the optimum critic obtained by maximizing \cref{equ:nwj_obj} retrieves the optimal ratio estimator. 
In contrast to \ac{nce}, which originates from a binary classification surrogate, \ac{nwj} is derived directly from convex duality and makes no explicit reference to a discriminative viewpoint.
Intriguingly, however, as we demonstrate in \cref{sec:info_theory_p}, the noise-scaled \ac{nce} objectives studied in this work establish an implicit connection between these ostensibly distinct paradigms. This perspective positions our approach as a unifying family of likelihood-ratio estimators.

\paragraph{\texorpdfstring{\acf{mle}}{}}
\acf{mle} remains the generative route: it fits probabilistic models by maximizing $\E_{q_{*}}[\log p_{\bA}(\x)]$ over data samples. A common parameterization of $p_{\bA}$~\citep{lecun2006tutorial,xie2016theory,du2019implicit} is 
\begin{equation}
\label{equ:ebm_def}
    p_{\bA}(\x) := \frac{1}{Z_{{\bA}}}\exp
                    \left(
                        f_{\bA}(\x)
                    \right)q_0(\x),
\end{equation}
where $f_{\bA}$ denotes the unnormalized log-density, $Z_{\bA}$ is the partition function, and $q_0(\x)$ is the base noise distribution. The model in \cref{equ:ebm_def} can be interpreted as an energy-based correction or exponential tilting of $q_0$~\citep{xie2016theory}.
The corresponding gradient takes the general form
\begin{equation}
\label{equ:mle_grad}
\nabla_{\bA} \cJ^{\rm MLE}({\bA})
    = \E_{q_*}[\nabla_{\bA} f_{\bA}(\x)]
    - \E_{p_{\bA}}[\nabla_{\bA} f_{\bA}(\x)].
\end{equation}
The first expectation over the data distribution is straightforward to approximate, but the second requires sampling from $p_{\bA}$. This is often intractable when the model involves an unknown partition function. In such cases, \ac{mcmc} methods such as Langevin dynamics~\citep{welling2011bayesian} are commonly employed. While theoretically valid, MCMC sampling can converge slowly in high-dimensional or multimodal settings, making MLE difficult to apply in practice~\citep{nijkamp2020anatomy}. This sampling bottleneck is a key motivation for alternatives such as \ac{nce}, which avoids direct sampling from $p_{\bA}$ by reframing density estimation as a classification problem — an idea we revisit from a less-explored perspective in the next section.

\section{``Noisier'' Noise Contrastive Estimation}
\label{sec:n2ce}
\subsection{Noise Magnitude}
\label{sec:noise_magnitude}
We now make precise what we mean by the \emph{magnitude} of the noise distribution in \ac{n2ce}. Recall the standard objective in \cref{equ:nce_obj}, where samples from the base distribution $q_{0}$ serve as negatives. More generally, we may scale the contribution of $q_{0}$ by a positive factor $M > 1$, which we refer to as the \emph{noise magnitude}. This leads to the following ``noisier'' \ac{nce} objective \citep{gutmann2012noise}:
\begin{equation}
\label{equ:n2ce_obj}
\LL_{M}({\bA}) =
    \E_{q_*(\x)} 
   \left[\log \frac{r_{{\bA}}(\x)}{M + r_{{\bA}}(\x)}\right]
    +
   M\E_{q_{0}(\x)} 
   \left[\log \frac{M}{M + r_{{\bA}}(\x)}\right].
\end{equation}
Here, $M=1$ recovers the standard NCE objective, while larger $M$ corresponds to amplifying the effective weight of the noise distribution. Intuitively, this adjustment can be seen as replacing $q_{0}$ with a \emph{virtually scaled} mixture of $M$ independent copies of $q_{0}$. As we show below, increasing $M$ has a striking effect: under mild conditions, the gradient of $\LL_{M}({\bA})$ approaches that of \ac{mle}, \ie, \cref{equ:mle_grad}. This not only establishes a principled bridge between NCE and MLE, but also provides a simple and effective remedy for the convergence issues of standard \ac{nce}.

\subsection{Scaling Effect of Noise Magnitude} 
\label{sec:scaling_M}
\begin{proposition}[Gradient approximation]
\label{prop:approx_grad}
Under mild regularity conditions,
\[
\lim_{M \to \infty} \nabla_{\bA}\LL_M({\bA})
= \E_{q_*}[\nabla_{\bA} f_{\bA}(\x)]
  - \E_{p_{\bA}}[\nabla_{\bA} f_{\bA}(\x)].
\]
\end{proposition}

\textit{Sketch of Proof.}
The gradient of $\LL_M$ can be expressed as
\[
\nabla_{\bA}\LL_M({\bA})
= \int \frac{M}{M + r_{\bA}(\x)}\,
\big(q_{*}(\x)-p_{\bA}(\x)\big)\,
\nabla_{\bA} f_{\bA}(\x)\, d\x.
\]
This makes clear that, with smooth densities and ratio functions, as $M$ grows, the weight converges to $1$ and the gradient approaches the MLE form. Additional derivation details appear in the \cref{appx:proof_grad}.

\paragraph{Remark.}
Our result (\cref{prop:approx_grad}) establishes a gradient-level 
approximation to MLE. Prior works have not analyzed this: 
\citet{gutmann2012noise} focused only on asymptotic consistency 
without optimization dynamics, while \citet{mnih2012fast} treated 
discrete embeddings without convergence analysis. In contrast, 
our formulation applies to continuous settings and, crucially, 
enables the first explicit convergence guarantees for ``noisier'' NCE-type 
objectives. We demonstrate this in the exponential-family case summarized in \cref{thm:convergence_rate}.

\begin{theorem}[Convergence in exponential families (informal)]
\label{thm:convergence_rate}
Under standard regularity assumptions for exponential families (details in
\cref{appx:proof_of_conv_thm}), letting $\lambda_{\min}, \lambda_{\max}$ denote the extremal eigenvalues of the Fisher information matrix, then for sufficiently large $M$, normalized
gradient \emph{ascent} on \cref{equ:n2ce_obj} finds an iterate
within distance $\delta$ of the true parameter $\bA^{*}$ in at most
\[
T \;\le\; C \left(\frac{\lambda_{\max}}{\lambda_{\min}}\right)^{\!3}
\frac{\|\bA^{0}-\bA^{*}\|_2^{2}}{\delta^{2}}
\]
iterations for a universal constant $C>0$, i.e., there exists $t\le T$ with
$\|\bA^{t}-\bA^{*}\|_2 \le \delta$.
\end{theorem}

\paragraph{Remark.}
\Cref{thm:convergence_rate} gives a \emph{polynomial} iteration complexity in the
condition number $\kappa=\lambda_{\max}/\lambda_{\min}$ for exponential families
when $M$ is sufficiently large. In contrast, for standard \ac{nce} the effective
Hessian can be ill-conditioned unless $q_{*}$ and $q_{0}$ are already close. This can lead to much worse, often effectively exponential, dependence on the gap.
Our result shows that virtually scaling the noise magnitude $M$ acts as a form of landscape regularization: the Hessian condition number of \cref{equ:n2ce_obj} remains uniformly bounded under standard assumptions \citep{liu2021analyzing}, \emph{without} requiring distributional closeness. 
We validate this intuition empirically: in 2D Gaussian simulations (\cref{fig:gauss_simu}) the trajectories indeed converge as predicted, and similar behavior persists even in high-dimensional, multimodal neural settings (\cref{sec:exps}). A full formal statement and proof are given in \cref{appx:proof_of_conv_thm}.

\subsection{Practical Error Analysis, \texorpdfstring{\acs{n2ce}}{} Family and Regularization}
\label{sec:error_n_reg}
The population results above assume exact expectations and arbitrarily large $M$.
In practice, both $M$ and the sample sizes are finite. We therefore ask:
how well does the empirical ``noisier'' \ac{nce} gradient approximate the
\ac{mle} gradient under these practical constraints? \Cref{prop:grad_finite_ver} provides an error decomposition that makes this trade-off explicit:
\begin{proposition}[Finite-$M$, finite-sample error (informal)]
\label{prop:grad_finite_ver}
Let $\widehat{\LL}_M$ be the empirical objective built from $n$ i.i.d. samples
from $q_{*}$ and $q_{0}$, and write $D_M(\x)=\frac{r_{\bA}(\x)}{M+r_{\bA}(\x)}$.
Under standard regularity assumptions (see \cref{appx:proof_finite_error}),
the mean-squared approximation error satisfies
\[
\mathbb{E}\,\big\|\nabla_{\bA}\cJ^{\rm MLE}(\bA)
-\nabla_{\bA}\widehat{\LL}_M(\bA)\big\|_2^2
\;\le\; V_u \;+\; B_u,~\mathrm{\textit{where}}
\]
\[
B_u = O\!\left(\frac{1}{M^{2}}\right),
V_u = \frac{C}{n}\!\left(
\E_{q_*}\!\|\nabla_{\bA}\log r_{\bA}\|_2^2 \;+\;
\min\!\Big\{\,M^{2}\,\E_{q_0}\!\|\nabla_{\bA}\log r_{\bA}\|_2^{2},\;
\E_{q_0}\!\|\nabla_{\bA} r_{\bA}\|_2^{2}\Big\}
\right)
\]
and $C>0$ hides benign constants.
\end{proposition}

\paragraph{Remark.}
$B_u$ is the finite-$M$ \emph{bias};
$V_u$ is the sampling \emph{variance}. The finite-$M$ bias decays as $O(1/M^{2})$, while the variance
can grow as $O(M^{2}/n)$ unless the ratio (or its log) is sufficiently smooth
under $q_{0}$, in which case the variance term saturates at the $\E_{q_0}\|\nabla r_{\bA}\|_2^{2}$ level. 
This highlights a bias–variance trade-off shaped jointly by $M$ and the
behavior of $r_{\bA}$. Importantly, it shows that \ac{n2ce} naturally induces a continuum of empirical objectives parameterized by $M$, and that an optimal finite-$M$ estimator exists in this spectrum. Besides, controlling the \emph{roughness} of the ratio can further help stabilize the variance term and obtain robust
performance. We make the optimal-$M$ characterization precise below, and present two concrete regularizations that restrict $r_{\bA}$ accordingly. 
As shown in \cref{sec:exps}, coupling these practical choices with the \ac{n2ce} objective yields strong performance across a diverse set of tasks.
A formal statement for \cref{prop:grad_finite_ver} with more details appears in \cref{appx:proof_finite_error}.

\paragraph{The \ac{n2ce} family and U-shaped patterns}
Indeed, with a finite fixed sample size $n$, we consistently observe the U-shaped dependence on $M$ predicted by \cref{prop:grad_finite_ver}.
This emerges clearly in controlled 5-dimensional Gaussian experiments across different regimes (\cref{appx:gauss_simu}), and perhaps more impressively, reappears in high-dimensional neural settings (\cref{tab:abl_M_dxmi,tab:abl_M_sida,tab:abl_K_M}).
Furthermore, \cref{prop:grad_finite_ver} predicts that the optimal $M$ should scale no larger than $C\sqrt{n}$, for some $C$ expected to lie within $1-10$ determined by the actual behavior of $r_{\bA}$. This theoretical prediction matches our empirical findings with remarkable fidelity.
Thus, the finite-sample analysis in \cref{prop:grad_finite_ver} both explains and anticipates the observed U-shaped curves, offering a principled and practical guideline for selecting $M$ in real-world applications. We next introduce two regularization strategies that provide practical handles for controlling the behavior of $r_{\bA}$.

\paragraph{Multi-stage ratio estimation} 
One way to stabilize the gradient approximation is to adopt the multi-stage estimation strategy of \citet{rhodes2020telescoping}. This is especially useful when (i) a convenient base noise distribution is available, (ii) we prefer to use the density model directly as a discriminator or for decision making rather than as a smooth reward, and (iii) the problem scale is low- or moderately high-dimensional. The key idea is to decompose the ratio between $q_{*}$ and $q_{0}$ into a telescoping product, $\tfrac{q_{*}}{q_{0}} = \tfrac{q_{*}}{q_{K}} \tfrac{q_{K}}{q_{K-1}} \cdots \tfrac{q_{1}}{q_{0}}$, where $\{q_k\}_{k=1}^K$ are \emph{pre-specified} intermediate distributions. Each ratio now involves a pair $(q_{k+1},q_k)$ with greater overlap by design, yielding smaller values and more controllable variance at each stage. The trade-off is computational: in high-dimensional settings many stages may be needed, increasing overhead. For this reason, we mainly apply this technique in lower-dimensional tasks such as latent-space modeling.

\paragraph{Direct ratio regularization} 
A more general and convenient strategy is to add a penalty on the ratio itself, for example $\E\|\log r_{\bA}\|_2^2$, directly to the ``noisier'' \ac{nce} objective. This approach is broadly applicable to high-dimensional data and does not rely on auxiliary intermediate distributions. The trade-off is that the added penalty may bias the gradients. Nonetheless, as we show in \cref{sec:exps}, this regularization is highly effective in practice, particularly for training reward models or critics in high-dimensional settings such as on ImageNet64$\times$64 datasets.

\subsection{An Information-Theoretic Perspective}
\label{sec:info_theory_p}
Complementing the gradient-based and finite-sample analyses, an information-theoretic interpretation further reveals that the \ac{n2ce} objectives trace a continuous path between variational representations of JS and KL divergences (Sec. 7.13 in \cite{polyanskiy2025information}), thereby making explicit both the $M\!\to\!\infty$ limit of \ac{n2ce} being \ac{nwj} and its consistency with maximum-likelihood estimation.

\paragraph{Variational representations of 
$f$-divergences} Specifically, on one end of the spectrum, the KL divergence satisfies the Nguyen–Wainwright–Jordan (NWJ) representation (see \cref{sec:background}), recovering the NWJ objective in \cref{equ:nwj_obj} when $T=\log r$:
\begin{equation}
\label{equ:kld_nwj}
\KL(q_*\|q_0)
= 1 + \sup_{T}\!\left(\E_{q_*}[T(\x)]-\E_{q_0}[e^{T(\x)}]\right)
  = 1 + \sup_{r}\!\left(\E_{q_*}[\log r(\x)]-\E_{q_0}[r(\x)]\right).
\end{equation}
On the other end, the Jensen–Shannon divergence can be written as
\begin{equation}
\label{equ:jsd_nce}
\JS(q_*\|q_0)
= \log 2
+ \sup_{r}\!\left[
    \E_{q_*}\!\left(\log \frac{r}{1+r}\right)
  + \E_{q_0}\!\left(\log \frac{1}{1+r}\right)
  \right],
\end{equation}
which matches the standard NCE objective in \cref{equ:nce_obj} up to an
additive constant \citep{nowozin2016f}.  
Thus NWJ and NCE arise as variational lower bounds on KL and JS divergences,
respectively, each attaining its optimum at the true ratio $r^{*}=q_*/q_0$.

\paragraph{\ac{n2ce} objectives trace a continuous interpolation path}

The \ac{n2ce} objective in \cref{equ:n2ce_obj} interpolates
between these two extremes.  
Let $\alpha=M/(1+M)\in(0,1)$ and consider the divergence
\begin{equation}
D_{\alpha}(q_*\|q_0)
=(1-\alpha)\KL\!\left(q_*\middle\|\alpha q_0+(1-\alpha)q_*\right)
+\alpha\,\KL\!\left(q_0\middle\|\alpha q_0+(1-\alpha)p_*\right),
\end{equation}
which satisfies $D_{1/2}=\JS$ and $D_{\alpha}\to\KL$ as $\alpha\to1$.
One can now show that
\begin{equation}
\label{equ:d_alpha}
D_{\alpha}(q_*\|q_0)
= h(\alpha)
+ \sup_{r}\left[
    \E_{q_*}\!\left(\log\frac{r}{M+r}\right)
  + M\,\E_{q_0}\!\left(\log\frac{M}{M+r}\right)
  \right],
\end{equation}
which is identical (up to $h(\alpha)$, the binary entropy function) to the \ac{n2ce} objective: (i) at $M=1$ ($\alpha=1/2$), \ac{n2ce} reduces to the usual NCE/JS bound;  
(ii) as $M\to\infty$ ($\alpha\to1$),
$
\log\tfrac{Mr}{M+r}\to\log r,
M\log\tfrac{M}{M+r}\to -r,
$
so the \cref{equ:n2ce_obj} converges to $\E_{q_*}[\log r]-\E_{q_0}[r] + {\rm const}$, the NWJ form of KL; maximizing this variational bound yields the same population optimum as maximum likelihood. We further provide empirical evidence in \cref{appx:gauss_simu} that \ac{n2ce} indeed approaches \ac{nwj} when $M$ is sufficiently large ($M=1e9$). With \cref{sec:scaling_M}, these results jointly explain why letting $M\to\infty$
recovers maximum-likelihood learning, at both the divergence level and
the gradient-dynamics level.

\section{Experiments}
\label{sec:exps}
We organize our empirical study around three main questions:  
(i) How does the proposed objective compare with standard baselines such as pure \ac{mle} and vanilla \ac{nce}?  
(ii) Do the advantages of ``noisier'' \ac{nce} transfer to downstream tasks?  
(iii) How do key hyperparameters, in particular the noise magnitude $M$, affect performance? 

To answer these, we conduct experiments across a broad range of datasets and benchmarks, spanning diverse model families. The results consistently show that our method outperforms existing approaches, while providing new insights into the role of noise scaling. Full experimental settings, implementation details and baseline descriptions are provided in \cref{appx:tr_details_n_arch,appx:add_exp}.

\subsection{Learning \ac{lebm} with \ac{n2ce}}
\label{sec:exp_learn_lebm}

\begin{table*}[!htbp]
\centering
% ----------------- First table -----------------
\begin{minipage}{0.50\linewidth}
\caption{\textbf{FID($\downarrow$) on different datasets}. We highlight 
{\color{tgray} our model}, the $\mathbf{1^{\rm st}}$ and \underline{$2^{\rm nd}$} performances; 
tables henceforth follow this format. Numbers from the first six rows are from \citet{yu2024learning}. 
\texttt{nz} denotes the latent dimension. \texttt{M} and \texttt{K} denote the noise magnitude and num. of stages for ratio estimation, respectively.}
\vspace{8pt}
\label{table:fid}
\resizebox{\linewidth}{!}{
\begin{tabular}{lcccc}
    \toprule
      \multirow{2}{*}{Model}
     & \multicolumn{1}{c}{SVHN} 
     & \multicolumn{1}{c}{CelebA}   
     & \multicolumn{1}{c}{CIFAR10}
     & \multicolumn{1}{c}{CelebAHQ}\\
     {}
     & \multicolumn{1}{c}{\texttt{nz=100}} 
     & \multicolumn{1}{c}{\texttt{nz=100}}   
     & \multicolumn{1}{c}{\texttt{nz=128}}
     & \multicolumn{1}{c}{\texttt{nz=512}}\\
\midrule
ABP
&49.71&51.50&90.30&160.21\\
ABP-LEBM$^*$
&29.44&37.87&\textbf{70.15}&133.07\\
SRI
&44.86&61.03&-&-\\
SRI (L=5)
&35.32&47.95&-&-\\
\midrule
2s-VAE &42.81&44.40&\underline{72.90}
&-\\
RAE&40.02&40.95&74.16
&-\\
\midrule
VAE &34.81 &47.84 &110.37 &154.60\\
\cmidrule(l){1-1}
w/ \ac{mle}-\ac{lebm} &32.74 &40.24 &90.54 &111.11\\
\cmidrule(l){1-1}
w/ \ac{nce}-\ac{lebm} &30.71 &39.61 &92.83 &118.84\\
\cmidrule(l){1-1}
w/ \ac{n2ce}-\ac{lebm} \\
\rowcolor{gray}
\texttt{M=100,K=1} &\underline{26.84} &33.05 &77.35 &\underline{101.71}\\
\rowcolor{gray}
\texttt{M=100,K=3}& \textbf{25.63}& \textbf{31.09}& 77.05& \textbf{95.66}\\
\bottomrule
\end{tabular}}
\end{minipage}\hfill
% ----------------- Second table -----------------
\begin{minipage}{0.47\linewidth}
\caption{\textbf{AUPRC($\uparrow$) scores for unsupervised anomaly detection on MNIST}. 
Baseline numbers are taken from \citet{yoon2023energy,yu2024learning}. 
Full results with variances in found in \cref{appx:anomaly_det}.}
\vspace{8pt}
\label{table:auprc}
\resizebox{\linewidth}{!}{
\begin{tabular}{lccccc}
\toprule
Heldout Digit & 1 & 4 & 5 & 7 & 9 \\
\midrule
AE
&0.062 &0.204 &0.259 &0.125 &0.113\\
VAE
&0.063 &0.337 &0.325 &0.148 &0.104\\
ABP
&$0.095$ & $0.138$
&$0.147$ & $0.138$ & $0.102$ \\
IGEBM
&$0.101$ & $0.106$
&$0.205$ & $0.100$ & $0.079$ \\
MEG
&$0.281$ & $0.401$ 
&$0.402$ & $0.290$ & $0.342$ \\
BiGAN-$\sigma$
&$0.287$ & $0.443$ 
&$0.514$ & $0.347$ & $0.307$ \\
ABP-LEBM 
&$0.336$ & $0.630$
&$0.619$ & $0.463$ & $0.413$ \\
JVAEBM 
&$0.297$ & $0.723$
&$0.676$ & $0.490$ & $0.383$ \\
Adaptive CE
&$0.531$ 
&$0.729$
&$0.742$ 
&$0.620$ 
&$0.499$\\
NAE
&$0.802$ 
&$0.648$
&$0.716$ 
&$0.789$ 
&$0.441$\\
MPDR-S
&$0.764$ 
&$0.823$
&$0.741$ 
&$\textbf{0.857}$ 
&$0.478$\\
MPDR-R
&$0.844$ 
&$0.711$
&$0.757$ 
&$\underline{0.850}$ 
&$0.569$\\
\midrule
DAMC
&$0.684$ 
&$0.911$ 
&$0.939$ 
&$0.801$ 
&$\underline{0.705}$ \\
DAMC-\ac{nce}
&$0.702$ 
&$0.829$ 
&$0.764$ 
&$0.605$ 
&$0.502$ \\
\cmidrule(l){1-1}
DAMC-\ac{n2ce} \\
\rowcolor{gray}
\texttt{M=100,K=1}
&$\underline{0.910}$ 
&$\underline{0.911}$ 
&$\underline{0.935}$ 
&$0.779$ 
&$0.699$ \\
\rowcolor{gray}
\texttt{M=100,K=3}
&$\textbf{0.959}$
&$\textbf{0.935}$
&$\textbf{0.959}$
&$0.845$  
&$\textbf{0.854}$  \\
\bottomrule
\end{tabular}}
\end{minipage}
\end{table*}

As a proof of concept, we first conduct a set of lightweight experiments on latent energy-based models (\acp{lebm}) across image datasets including CIFAR-10 \citep{krizhevsky2009learning}, MNIST \citep{lecun1998mnist}, SVHN \citep{netzer2011reading}, CelebA64 \citep{liu2015deep}, and CelebAMask-HQ \citep{CelebAMask-HQ}. Learning an \ac{lebm} \citep{pang2020learning,yu2022latent} can be viewed as fitting an unnormalized density model in the latent space of a \acs{vae} (\cref{appx:dlvm_alg_n_implem}). Prior approaches largely rely on \ac{mcmc}-based \ac{mle}-style training \citep{pang2020learning,du2020improved}, where convergence and stability are often problematic. This makes \acp{lebm} a natural and tractable testbed for prototyping our method. In these experiments, we apply ratio decomposition as a regularization (\cref{sec:error_n_reg}) and follow the training and evaluation protocols in \citet{pang2020learning} (see \cref{appx:img_exp_arch}).  

\paragraph{Generative image modeling}
We evaluate the quality of the learned \ac{lebm} with generated images measured by FID scores \citep{heusel2017gans} (see \cref{table:fid}). For image generation with \ac{lebm}, we perform 100-step short-run \ac{ld} to draw latent vectors from the learned models, and map the latent vectors to the image space with the decoder network.   

We observe that: first, \acp{lebm} trained with \ac{n2ce} (\ie, rows with \texttt{M=100}) consistently outperform those trained with vanilla \ac{nce}, with multi-stage estimation providing further gains that validate our analysis. Second, \acp{lebm} trained with \ac{n2ce} show substantial improvements over those trained with \ac{mle} and short-run \ac{ld} (denoted \ac{mle}-\ac{lebm}). These results confirm that \ac{n2ce} is a more reliable and effective objective that addresses the issues inherent to vanilla \ac{nce} and \ac{mcmc}-based \ac{mle}. Finally, the improvement becomes more pronounced as the latent dimension \texttt{nz} increases, especially on the CelebAMask-HQ dataset with \texttt{nz=512}. This suggests that our proposal scales more gracefully to high-dimensional, multimodal targets.

\paragraph{Anomaly Detection}
To further highlight the practical advantages of our approach in highly multimodal settings, we evaluate anomaly detection on MNIST following the setup in \citet{zenati2018efficient}. Specifically, we train models with one digit held out as anomalous, focusing on digits 1, 4, 5, 7, and 9—cases known to be especially challenging. We build on DAMC \citep{yu2024learning} for posterior inference, replacing its original prior with an \ac{lebm} learned via \ac{n2ce}. Since empirical posteriors from DAMC are sharp and highly multimodal, this serves as a demanding testbed. \Cref{table:auprc} reports AUPRC scores averaged over 10 trials. Our method yields consistent and often substantial improvements over baselines, achieving a strong overall performance across the hardest digits. Full experimental details appear in \cref{appx:img_exp_arch,appx:anomaly_det}.

\subsection{Reward and Critic Learning for Diffusion Distillation}
\label{sec:exp_diffusion}

Building on the success of latent-space modeling, we next consider
more challenging high-dimensional settings in the ambient image space. In
particular, we study two representative scenarios for distilling diffusion
samplers: (i) learning energy-based rewards for diffusion fine-tuning (DxMI) \citep{yoon2024maximum}, and (ii) learning critics for adversarial
distillation (SiD$^2$A) \citep{zhou2024adversarial}. In both cases, we simply replace the original objectives with our ``noisier'' \ac{nce} combined with direct ratio regularization (\cref{sec:error_n_reg}), while using the same architectures and critical hyperparameters as baseline methods. Full implementation details appear in \cref{appx:reward_n_critic_learn}.

In the DxMI setup \citep{yoon2024maximum}, \ac{n2ce} yields substantial improvements
over vanilla \ac{nce}-like variants, previous methods and even the teacher models with full NFE (\cref{tab:dxmi_n2ce}),
underscoring its effectiveness in training rewards for diffusion
distillation. In the SiD$^2$A setting \citep{zhou2024adversarial}, our method greatly outperforms vanilla NCE variants, and not only matches but can surpass strong baselines (\cref{tab:cifar10-merged-iters,tab:imagenet64-merged-iters}), notably with up to \emph{half the training iterations}. Together, these results confirm that \ac{n2ce} scales to high-dimensional tasks while improving training efficiency across distinct distillation regimes. We provide additional ablation studies on $M$ on CIFAR-10 in \cref{appx:reward_n_critic_abl}.

\begin{table*}[t]
\centering
\caption{\textbf{CIFAR-10 (DDPM backbone) and ImageNet64$\times$64 (EDM backbone) results} shown side-by-side. First six rows are from \cite{yoon2024maximum}. $^\dagger$ highlights the starting point of DxMI fine-tuning.}
\vspace{8pt}
\label{tab:dxmi_n2ce}
\scriptsize
\setlength{\tabcolsep}{5pt}
\renewcommand{\arraystretch}{1.12}

\begin{tabular}{lccc @{\hspace{14pt}} lcccc}
\toprule
\multicolumn{4}{l}{\textbf{-- Backbone: DDPM}} & \multicolumn{5}{l}{\textbf{-- Backbone: EDM}} \\
\cmidrule(lr){1-4}\cmidrule(lr){5-9}
Method & NFE & FID ($\downarrow$) & Rec. ($\uparrow$) &
Method & NFE & FID ($\downarrow$) & Prec. ($\uparrow$) & Rec. ($\uparrow$) \\
\midrule
DDPM            & 1000 & 3.21 & 0.57  &
EDM (Heun)            & 79 & \underline{2.44} & 0.71 & 0.67 \\
FastDPM$^\dagger$    &   10 & 35.85 & 0.29 &
EDM (Ancestral)$^\dagger$ & 10 & 50.27 & 0.37 & 0.35 \\
\cmidrule(lr){1-4}\cmidrule(lr){5-9}
DDIM             &   10 & 13.36 & --   &
Consistency Model   &  2 & 4.70 & 0.69 & 0.64 \\
SFT-PG           &   10 & 4.82  & 0.606 &
Consistency Model   &  1 & 6.20 & 0.68 & 0.63 \\
DxMI             &   10 & 3.19  & 0.625 &
DxMI              & 10 & 2.68 & \underline{0.777} & 0.574 \\
DxMI + Value Guidance &   10 & \underline{3.17}  & \underline{0.623} &
DxMI + Value Guidance & 10 & 2.67 & \textbf{0.780} & 0.574 \\
\cmidrule(lr){1-4}\cmidrule(lr){5-9}
DxMI + \ac{nce} (\texttt{M=1})       &   10 & 3.93 & 0.623 &
DxMI + \ac{nce} (\texttt{M=1})       & 10 & 2.69 & 0.756 & \underline{0.585} \\
\rowcolor{gray}
DxMI + \ac{n2ce} (\texttt{M=100})      &   10 & \textbf{2.99} & \textbf{0.638} &
DxMI + \ac{n2ce} (\texttt{M=100})      & 10 & \textbf{2.23} & 0.757 & \textbf{0.599} \\
\bottomrule
\end{tabular}
\end{table*}

\begin{table*}[t]
\centering
\setlength{\tabcolsep}{3.6pt}
\renewcommand{\arraystretch}{1.08}

% ---------- Left: CIFAR-10 ----------
\begin{minipage}[t]{0.49\textwidth}
\captionsetup{type=table,width=\linewidth}
\caption{\textbf{CIFAR-10 results}. ``FID-U/C (iters)'' shows uncond./cond. FID and, when provided, corresponding training iterations in parentheses. Baseline numbers from \citep{zhou2024adversarial,zheng2025revisiting}.}
\label{tab:cifar10-merged-iters}
\scriptsize
\begin{tabularx}{\linewidth}{>{\raggedright\arraybackslash}X c c c c}
\toprule
Method & NFE & FID-U $\downarrow$ & IS $\uparrow$ & FID-C $\downarrow$ \\
\midrule
\multicolumn{5}{l}{\textbf{Training From Scratch}} \\
\midrule
DDPM            & 1000 & 3.17              & --    & --    \\
DDIM            &  100 & 4.16              & --    & --    \\
Score SDE       & 2000 & --                & --    & 2.20  \\
DPM\mbox{-}Solve\mbox{-}3 & 48 & --        & --    & 2.65  \\
EDM             &   35 & 1.98              & --    & 1.79  \\
BigGAN          &    1 & --                & -- & 14.73    \\
StyleGAN2\mbox{-}ADA & 1 & 2.92            & 9.82  & 2.42  \\
SAN             &    1 & \textbf{1.36}     & --    & --    \\
iCT    &    1 & 2.83              & 9.54  & --    \\
iCT   &    2 & 2.46              & 9.80  & --    \\
iCT-deep    &    1 & 2.51              & 9.76  & --    \\
iCT-deep    &    2 & 2.24              & 9.89  & --    \\
\midrule
\multicolumn{5}{l}{\textbf{Post-training}} \\
\midrule
PD              & 1 & 9.12                 & --    & --    \\
DFNO            & 1 & 3.78                 & --    & --    \\
CD     & 1 & 3.55                 & --    & --    \\
CD    & 2 & 2.93                 & --    & --    \\
CTM    & 1 & 1.98                 & --    & 1.73  \\
CTM   & 2 & 1.87                 & --    & 1.63  \\
DMD             & 1 & 2.62                 & --    & --    \\
D2O\mbox{-}F    & 1 & 1.54        & 10.10 & 1.44 \\
\cmidrule(l){1-5}
SiD                 & 1 & 1.92                 & 9.98 & 1.71 \\
SiD$^2$A      & 1 & \underline{1.50\iters{30K}} & 10.19    & 1.40\iters{50K}   \\
SiD + \ac{nce} (\texttt{M=1})     & 1 & 1.53\iters{30K}          & \underline{10.20}    & 1.46\iters{30K}   \\
\rowcolor{gray}
SiD + \ac{n2ce} (\texttt{M=50})     & 1 & \textbf{1.45\iters{20K}}          & \textbf{10.23}    & \textbf{1.39 (20K)}   \\
\bottomrule
\end{tabularx}
\end{minipage}
\hfill
% ---------- Right: ImageNet 64×64 ----------
\begin{minipage}[t]{0.49\textwidth}
\captionsetup{type=table,width=\linewidth}
\caption{\textbf{Conditional ImageNet 64$\times$64 results.}}
\label{tab:imagenet64-merged-iters}
\scriptsize
\begin{tabularx}{\linewidth}{>{\raggedright\arraybackslash}X c c c c}
\toprule
Method & NFE & FID $\downarrow$ (iters) & Prec.\ $\uparrow$ & Rec.\ $\uparrow$ \\
\midrule
\multicolumn{5}{l}{\textbf{Training From Scratch}} \\
\midrule
RIN          & 1000 & 1.23               & --    & --    \\
DDPM         &  250 & 11.00              & 0.67  & 0.58  \\
ADM          &  250 &  2.07              & 0.74  & 0.63  \\
EDM          &   79 &  2.64              & --  & --  \\
iCT  &    1 &  4.02              & 0.70  & 0.63  \\
iCT  &    2 &  3.20              & 0.73  & 0.63  \\
iCT-deep &    1 &  3.25              & 0.72  & 0.63  \\
iCT-deep &    2 &  2.77              & 0.74  & 0.62  \\
BigGAN-deep  &    1 &  4.06              & \textbf{0.79}  & 0.48  \\
StyleGAN2-XL &    1 &  1.51              & --    & --    \\
\midrule
\multicolumn{5}{l}{\textbf{Post-training}} \\
\midrule
PD           & 1 & 15.39                 & --    & --    \\
BOOT         & 1 & 16.3                 & 0.68  & 0.36  \\
DFNO         & 1 &  7.83                 & --    & 0.61    \\
CD           & 1 &  6.20                 & 0.68  & 0.63  \\
CD           & 2 &  4.79                 & 0.69  & \textbf{0.64} \\
CTM      & 1 &  1.92                 & 0.70 & 0.57 \\
CTM      & 2 &  1.73                 & 0.64  & 0.57  \\
DMD          & 1 &  2.62                 & --    & --    \\
sCD-S    & 1 &  2.97                 & --  & --  \\
sCD-S  & 2 &  2.07                 & --  & --  \\
DMD2         & 1 &  1.51                 & --    & --    \\
MSD (DM)     & 1 &  2.37                 & --    & --    \\
MSD (ADM)    & 1 &  1.20                 & --    & --    \\
D2O\mbox{-}F & 1 &  \underline{1.16} & \underline{0.75} & 0.60 \\
\cmidrule(l){1-5}
SiD          & 1 & 1.52                 & 0.74 & \underline{0.63} \\
SiD$^2$A     & 1 & \textbf{1.11\iters{20K}}                 & \underline{0.75} & 0.62 \\
SiD + \ac{nce} (\texttt{M=1}) & 1 & 1.28\iters{15K}          & \underline{0.75}   & 0.62   \\
\rowcolor{gray}
SiD + \ac{n2ce} (\texttt{M=50}) & 1 & \textbf{1.11\iters{10K}}  & \underline{0.75}   & \underline{0.63}   \\
\bottomrule
\end{tabularx}
\end{minipage}

\end{table*}

\subsection{Offline Black-Box Optimization}
\label{sec:exp_bbo}

\begin{figure*}[!ht]
\centering
\begin{minipage}{.45\textwidth}
    \includegraphics[width=\linewidth]{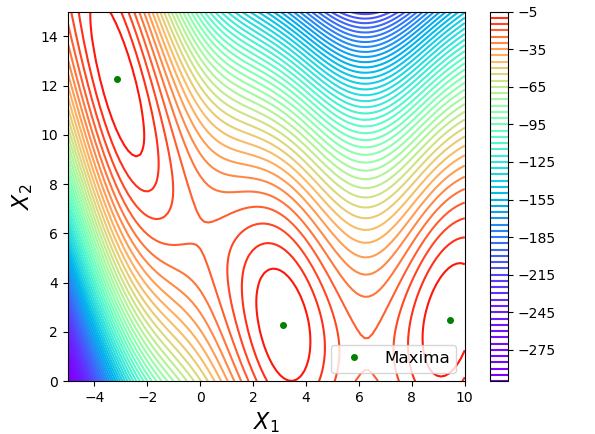}
    \vspace{1pt}
    \captionof{figure}{\textbf{Branin function level sets.}}
    \label{fig:branin}
\end{minipage}\hfill
\begin{minipage}{.52\textwidth}
    % make this panel become the *next* figure before subfigs:
    \refstepcounter{figure}
    \setcounter{subfigure}{0}% start at (a)
    \centering
    % ---- row 1 ----
    \begin{subfigure}{0.22\linewidth}
        \includegraphics[width=\linewidth]{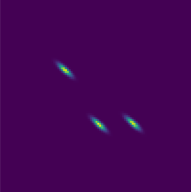}
        \caption{GT}
        \label{subfig:gt}
    \end{subfigure}
    \begin{subfigure}{0.22\linewidth}
        \includegraphics[width=\linewidth]{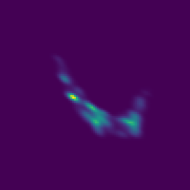}
        \caption{$\LL_{M=1}$}
        \label{subfig:m_1}
    \end{subfigure}
    \begin{subfigure}{0.22\linewidth}
        \includegraphics[width=\linewidth]{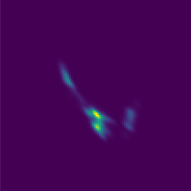}
        \caption{$\LL_{M=10}$}
        \label{subfig:m_10}
    \end{subfigure}
    \begin{subfigure}{0.22\linewidth}
        \includegraphics[width=\linewidth]{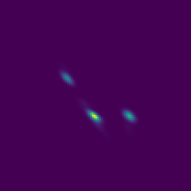}
        \caption{$\LL_{M=100}$}
        \label{subfig:m_100}
    \end{subfigure}

    % ---- row 2 ----
    \begin{subfigure}{0.22\linewidth}
        \includegraphics[width=\linewidth]{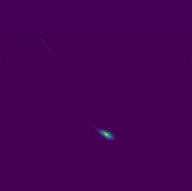}
        \caption{DDOM}
        \label{subfig:cdm}
    \end{subfigure}
    \begin{subfigure}{0.22\linewidth}
        \includegraphics[width=\linewidth]{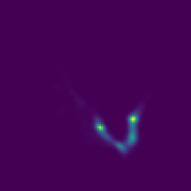}
        \caption{G-SV}
        \label{subfig:g_sv}
    \end{subfigure}
    \begin{subfigure}{0.22\linewidth}
        \includegraphics[width=\linewidth]{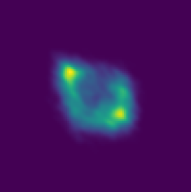}
        \caption{\acs{mle}-LD}
        \label{subfig:mle_ld}
    \end{subfigure}
    \begin{subfigure}{0.22\linewidth}
        \includegraphics[width=\linewidth]{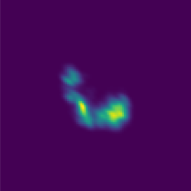}
        \caption{\acs{mle}-SV}
        \label{subfig:mle_svgd}
    \end{subfigure}
    \vspace{8pt}
    % avoid double-stepping when printing the real caption:
    \addtocounter{figure}{-1}
    \captionof{figure}{\textbf{Viz. of Branin optimal samples.} (b--d) are results of our method. 
    G-SV denotes the Gaussian prior model sampled with \acs{svgd}. 
    MLE-LD and MLE-SV denote the model trained by \ac{mle} sampled with 
    \ac{ld} and \ac{svgd}, respectively.}
    \label{fig:M_viz}
\end{minipage}
\end{figure*}

Beyond image modeling, we also explore the broader impact of our proposed technique through the lens of offline \ac{bbo}. This task evaluates not only how well a model captures the training
distribution, but also its ability to internalize structural regularities and
generalize to unseen queries. In the \textit{offline} \ac{bbo} setting \citep{trabucco2022design}, one is given a
pre-collected dataset $\cD=\{(\x_i, y_i)\}_{i=1}^n$ of inputs and their
black-box function values $h(\x_i)=y_i$. At test time, the optimizer may access
a limited budget $Q$ of evaluations of the unknown function $h$ and must return
candidates with high observed values. This setup is particularly well-suited
for assessing generative approaches to optimization, where the function value
$y$ acts as a conditioning signal
\citep{brookes2019conditioning,kumar2020model,krishnamoorthy2023diffusion}.

We instantiate this view by parameterizing a conditional model
$p_{\bT}(\x \mid y)$ via latent variables $\z$, using
\[
p_{\bT}(\x \mid y) \;\propto\;
\E_{p(\z \mid y)}\!\left[p_{\bB,\x}(\x \mid \z)\right],\quad
p(\z \mid y) \propto p_{\bB,y}(y\mid \z)\,p_{\bA}(\z),
\]
and employing stochastic samplers such as \acf{ld} or \acf{svgd} \citep{liu2016stein} for conditional
sampling. To train this model, we use a VAE with an \ac{lebm} prior that jointly
models $(\x,y)$. Full problem statement, implementation details, and extended
results are deferred to \cref{appx:bbo_general}.

\subsubsection{2D Branin Function}
\label{sec:exp_branin}

% wrap on the right (use {l} to wrap on the left)
\begin{wrapfigure}{r}{0.52\textwidth}   
\vspace{-12pt}
\centering

% ---------- Table on top ----------
\begin{minipage}{\linewidth}
  \captionsetup{font=small}
  \captionof{table}{\textbf{Results on the top-10\%-tile-removed Branin task (avg. over 5 runs).}
  \textsc{Opt} denotes the global optimum.}
  \vspace{8pt}
  \label{tab:10p_removed_branin}
  \resizebox{\linewidth}{!}{
  \begin{tabular}{ccccC}
    \toprule
    $\cD_{\rm max}$/Opt. & GA & BONET & DDOM & Ours \\
    \midrule
    -6.1/-0.4
    & $-4.0_{\pm 4.3}$ 
    & $-1.8_{\pm 0.8}$ 
    & $\underline{-1.6_{\pm 0.1}}$ 
    & $\mathbf{-0.4}_{\pm 0.1}$ \\
    \bottomrule
  \end{tabular}}
\end{minipage}

\vspace{0.6em} % gap between table and plots

% ---------- Plots under it ----------
\begin{minipage}{\linewidth}
  \captionsetup{font=small}
  \begin{subfigure}[t]{0.48\linewidth}
    \includegraphics[width=\linewidth]{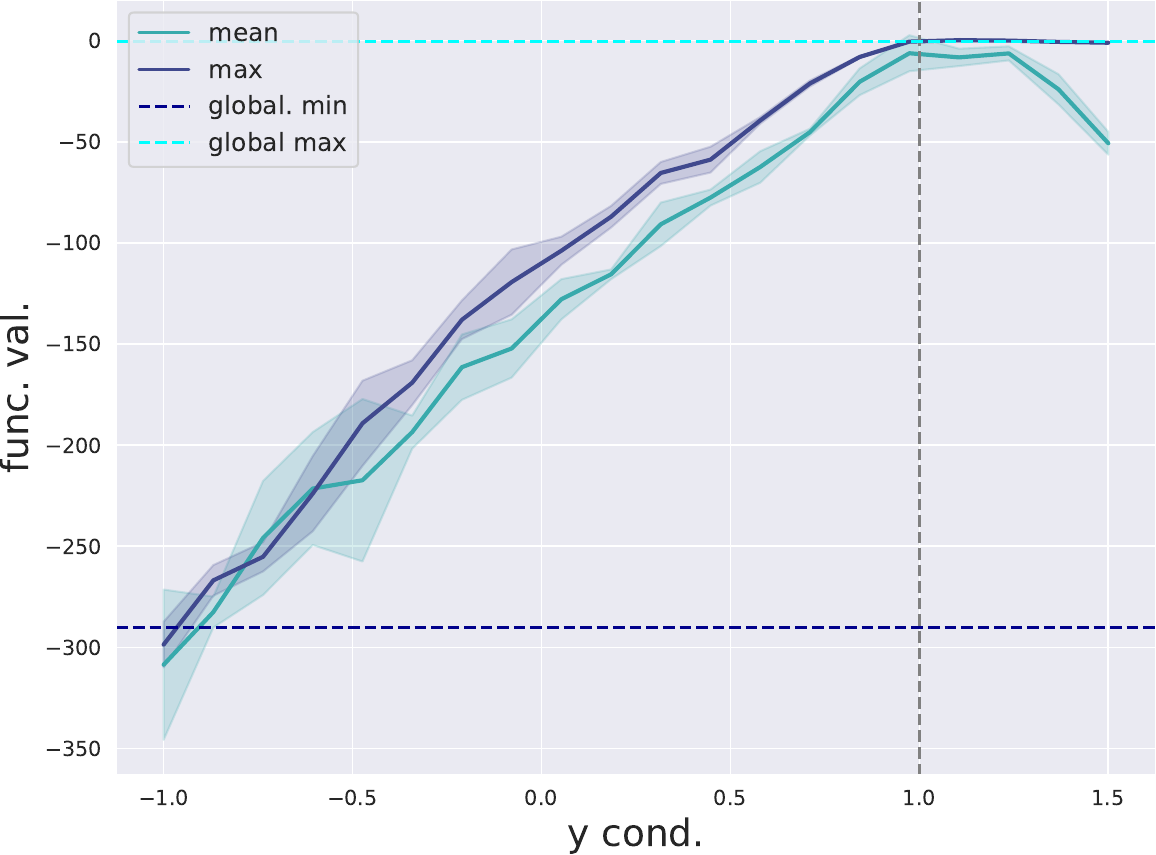}
    \caption{w/ top-10\% points}
    \label{subfig:y_all}
  \end{subfigure}\hfill
  \begin{subfigure}[t]{0.48\linewidth}
    \includegraphics[width=\linewidth]{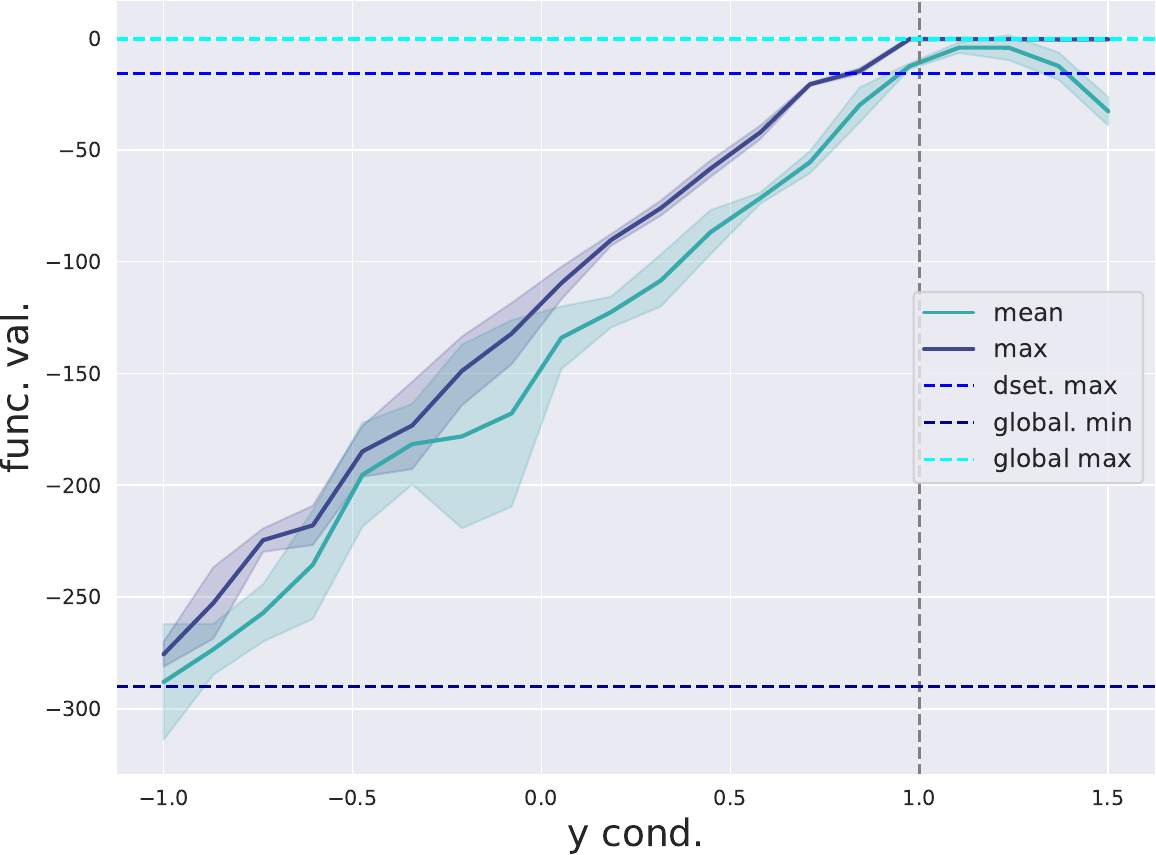}
    \caption{w/o top-10\% points}
    \label{subfig:y_remove10p}
  \end{subfigure}
  \vspace{8pt}
  \captionof{figure}{\textbf{Results on uniformly sampled Branin w/ and w/o top-10\% points.} Zoom in for details.}
  \label{fig:branin_y}
\end{minipage}

\vspace{-12pt}
\end{wrapfigure}

We begin with the 2D Branin function (\cref{fig:branin}) as a proof of concept.
This task serves as a sanity check to validate that our method can both
faithfully capture high-value modes and generalize beyond the best points
observed in the offline dataset (see \cref{appx:branin} for details). Following
DDOM \citep{krishnamoorthy2023diffusion}, we perform a held-out offline
optimization experiment. Specifically, we uniformly sample $N=5000$ points from
the domain $[-5,10]\times[0,15]$ to construct the offline dataset, then remove
the top $10\%$ by function value. This ensures that high-value contours are
retained, but the exact global optima are excluded during training. For evaluation, we adopt a query budget $Q=128$ and use \ac{svgd} for sampling.
\Cref{tab:10p_removed_branin} shows that our method generalizes well beyond the
maximum value in the dataset $\cD_{\rm max}$, approaching the true optimum and
outperforming gradient-based GA as well as modern generative baselines BONET
\citep{mashkaria2023generative} and DDOM \citep{krishnamoorthy2023diffusion}.

We further visualize conditional samples corresponding to high function values
(\cref{fig:M_viz}), and examine the correlation between target conditions (desired function values) and
achieved function values (\cref{fig:branin_y}). For visualization, we draw $2560$ samples from each model. In \cref{subfig:gt}, the distribution of true
optima is emulated with a 3-component GMM centered at the global minima. As
shown in the first row of \cref{fig:M_viz}, increasing $M$ in
\cref{equ:n2ce_obj} produces samples that better approximate the ground-truth
distribution, while models trained with MCMC-based \ac{mle} or vanilla \ac{nce}
(\texttt{M=1}) lag behind. This again supports our analysis that a sufficiently
large $M$ is critical for accurate gradient approximation. 

\begin{table*}[!htb]
\centering
\caption{\textbf{Normalized results on design-bench with $Q=256$ (avg. over 5 runs).}}
\vskip 0.1in
\resizebox{\textwidth}{!}{
\begin{tabular}{llllllllc}
\toprule
\multicolumn{1}{c}{\textbf{BASELINE}}  &\multicolumn{1}{c}{\textbf{TFBIND8}} 
&\multicolumn{1}{c}{\textbf{TFBIND10}}
&\multicolumn{1}{c}{\textbf{CHEMBL}} 
&\multicolumn{1}{c}{\textbf{SUPERCON.}}
&\multicolumn{1}{c}{\textbf{ANT}}  
&\multicolumn{1}{c}{\textbf{D'KITTY}}  
&\multicolumn{1}{c}{\textbf{MEAN SCORE}$^\uparrow$} 
&\multicolumn{1}{c}{\textbf{MEAN RANK}$^\downarrow$}\\
\midrule
$\cD$ (best) & $0.439$ & $0.467$ & $0.605$ & $0.399$ & $0.565$ & $0.884$ & {-} & {-}\\
\midrule
GP-qEI & $0.824 \pm 0.086$ & $0.635 \pm 0.011$ & $0.633 \pm 0.000$ & $0.501 \pm 0.021$ & $0.887 \pm 0.000$ & $0.896 \pm 0.000$ & $0.729 \pm 0.019$ & $7.8$ \\ 
CMA-ES & $0.933 \pm 0.035$ & $0.679 \pm 0.034$ & $0.636 \pm 0.004$ & $0.491 \pm 0.004$ & $\mathbf{1.436 \pm 0.928}$ & $0.725 \pm 0.002$ & $\underline{0.816 \pm 0.168}$ & $5.8$ \\ 
REINFORCE & $0.959 \pm 0.013$  & $0.640 \pm 0.028$ & $0.636 \pm 0.023$ & $0.481 \pm 0.017$ & $0.261 \pm 0.042$ & $0.474 \pm 0.202$ & $0.575 \pm 0.054$ & $8.3$ \\ 
\midrule
Gradient Ascent & $\underline{0.981 \pm 0.015}$ & $0.659 \pm 0.039$ & $0.647 \pm 0.020$ & $0.504 \pm 0.005$ & $0.340 \pm 0.034$ & $0.906 \pm 0.017$ & $0.672 \pm 0.021$ & $5.0$ \\ 
COMs & $0.964 \pm 0.020$ & $0.654 \pm 0.020$ & $0.648 \pm 0.005$ & $0.423 \pm 0.033$ & $0.949 \pm 0.021$ & $0.948 \pm 0.006$ & $0.764\pm 0.018$ & $5.8$ \\
BONET & $0.975 \pm 0.004$ & $0.681 \pm 0.035$ & $\underline{0.654 \pm 0.019}$ & $0.437 \pm 0.022$ & $0.976 \pm 0.012$ & $\underline{0.954 \pm 0.012}$ & $0.780 \pm 0.022$ & $\underline{3.7}$ \\
\midrule
CbAS & $0.958 \pm 0.018$ & $0.657 \pm 0.017$ & $0.640 \pm 0.005$ & $0.450 \pm 0.083$ & $0.876 \pm 0.015$ & $0.896 \pm 0.016$ & $0.746 \pm 0.003$ & $7.3$ \\ 
MINs & $0.938 \pm 0.047$ & $0.659 \pm 0.044$ & $0.653 \pm 0.002$ & $0.484 \pm 0.017$ & $0.942 \pm 0.018$ & $0.944 \pm 0.009$ & $0.770\pm 0.023$ & $5.5$ \\
DDOM & $0.971 \pm 0.005$ & $\underline{0.688 \pm 0.092}$ & $0.633 \pm 0.007$ & $\underline{0.560 \pm 0.044}$ & $0.957 \pm 0.012$ & $0.926 \pm 0.009$ & $0.787 \pm 0.034$ & $4.5$ \\
\midrule
\rowcolor{gray}
Ours & $\mathbf{0.990 \pm 0.003}$ & $\mathbf{0.803 \pm 0.085}$ & $\mathbf{0.661 \pm 0.025}$ & $\mathbf{0.567 \pm 0.017}$ & $\underline{0.982 \pm 0.012}$ & $\mathbf{0.961 \pm 0.006}$ & $\mathbf{0.827 \pm 0.021}$ & $\mathbf{1.2}$ \\
\bottomrule
\end{tabular}}    
\label{tab:main_table_256}
\end{table*}

\begin{table*}[!htb]
\centering
\caption{\textbf{Ablation studies on design-bench (avg. over 5 runs)}, results with a budget $Q = 256$.}
\vskip 0.1in
\resizebox{\textwidth}{!}{
\begin{tabular}{llllllll}
\toprule
\multicolumn{1}{c}{\textbf{BASELINE}}  &\multicolumn{1}{c}{\textbf{TFBIND8}} 
&\multicolumn{1}{c}{\textbf{TFBIND10}}
&\multicolumn{1}{c}{\textbf{CHEMBL}} 
&\multicolumn{1}{c}{\textbf{SUPERCON.}}
&\multicolumn{1}{c}{\textbf{ANT}}  
&\multicolumn{1}{c}{\textbf{D'KITTY}}  
&\multicolumn{1}{c}{\textbf{MEAN SCORE}$^\uparrow$} \\
\midrule
$\cD$ (best) & $0.439$ & $0.467$ & $0.399$ & $0.565$ & $0.884$ & $0.605$ & {-} \\
\midrule
\ac{ld} sampler \\
\cmidrule(l){1-1}
\acs{mle}-\ac{lebm} & $0.524 \pm 0.228$ & $0.667 \pm 0.000$ & $0.633 \pm 0.000$ & $0.362 \pm 0.006$ & $0.512 \pm 0.000$ & $0.672 \pm 0.049$ & $0.562 \pm 0.033$ \\ 
\rowcolor{gray}
\ac{n2ce}-\ac{lebm} (\texttt{M=100,K=6}) & $0.834 \pm 0.073$ & \underline{$0.739 \pm 0.095$} & $0.639 \pm 0.013$ & $0.363 \pm 0.006$ & \underline{$0.957 \pm 0.023$} & $0.955 \pm 0.004$ & $0.748 \pm 0.025$ \\ 
\midrule
\ac{svgd} sampler \\
\cmidrule(l){1-1}
\acs{mle}-\ac{lebm} & $0.941 \pm 0.012$ & $0.681 \pm 0.079$ & $0.634 \pm 0.002$ & $0.296 \pm 0.004$ & $0.926 \pm 0.016$ & $0.915 \pm 0.007$ & $0.732 \pm 0.014$ \\ 
\ac{nce}-\ac{lebm} (\texttt{M=1,K=6}) & \underline{$0.961 \pm 0.024$} & $0.655 \pm 0.034$ & $0.640 \pm 0.011$ & \underline{$0.497 \pm 0.044$} & $0.907 \pm 0.007$ & \underline{$0.959 \pm 0.002$} & $0.770 \pm 0.007$ \\ 
\cmidrule(l){1-1}
\rowcolor{gray}
\ac{n2ce}-\ac{lebm} (\texttt{M=100,K=1}) & $0.945 \pm 0.025$ & $0.722 \pm 0.043$ & $0.648 \pm 0.003$ & $0.421 \pm 0.020$ & $0.893 \pm 0.008$ & $0.953 \pm 0.008$ & $0.764 \pm 0.005$ \\ 
\rowcolor{gray}
\ac{n2ce}-\ac{lebm} (\texttt{M=100,K=6})  & $\mathbf{0.990 \pm 0.003}$ & $\mathbf{0.803 \pm 0.085}$ & $\mathbf{0.661 \pm 0.025}$ & $\mathbf{0.567 \pm 0.017}$ & $\mathbf{0.982 \pm 0.012}$ & $\mathbf{0.961 \pm 0.006}$ & $\mathbf{0.827 \pm 0.021}$ \\
\bottomrule
\end{tabular}}    
\label{tab:abl_table}
\end{table*}

\subsubsection{Offline \texorpdfstring{\acs{bbo}}{} on Design-Bench}
\label{sec:exp_design_bench}
\paragraph{Task setup and evaluation.} 
We next evaluate on higher-dimensional real-world tasks\footnote{NAS and Hopper tasks are excluded following \cite{mashkaria2023generative}; see \cref{appx:no_hopper}.} from \citet{trabucco2022design}, grouped into
\textbf{discrete optimization} (TF-Bind-8, TF-Bind-10, ChEMBL) and
\textbf{continuous optimization} (D’Kitty, Ant Morphology, Superconductor).  
We compare against three categories of baselines:
\textbf{(i)} generative inverse models with different parameterizations, including CbAS \citep{brookes2019conditioning}, Auto.CbAS \citep{fannjiang2020autofocused}, MIN \citep{kumar2020model}, and DDOM \citep{krishnamoorthy2023diffusion};
\textbf{(ii)} gradient(-like) updating from existing designs, such as gradient-ascent-based methods \citep{fu2020offline,yu2021roma,trabucco2021conservative,chen2022bidirectional,qi2022data,yuan2024importance,chen2024parallel} and BONET \citep{mashkaria2023generative};
\textbf{(iii)} additional baselines including REINFORCE and evolutionary algorithm reported in \citet{trabucco2022design}.
Further details are given in \cref{appx:db_overview,appx:db_baselines}. 

Following standard practice \citep{trabucco2022design}, we report normalized ground-truth function values
$
y_n = \tfrac{y - y^*_{\rm min}}{y^*_{\rm max} - y^*_{\rm min}},
$
where $y^*_{\rm min}$ and $y^*_{\rm max}$ are the minimum and maximum values in the full dataset (unseen during training). We summarize performance using both mean scores and mean normalized ranks (MNR) across tasks. We provide two sets of results:  
(1) $Q=256$ queries following \citet{krishnamoorthy2023diffusion};  
(2) $Q=128$ queries following \citet{trabucco2021conservative,chen2024parallel}, excluding ChEMBL.  
In the main text we report $100$-th percentile results with $Q=256$, and defer $Q=128$ results and additional percentiles (e.g., $50$-th) to \cref{appx:db_results}.

\paragraph{Main results.}
\Cref{tab:main_table_256,tab:main_table_128} show that our method attains the best
average ranks of $1.2$ ($Q=256$) and $1.8$ ($Q=128$), compared to runner-ups
BONET \citep{mashkaria2023generative} and Tri-mentoring
\citep{chen2024parallel}, which achieve $3.7$ and $2.8$, respectively. Our
approach consistently yields the highest normalized mean values across tasks,
outperforming all baselines—particularly generative inverse models—and confirming the
effectiveness of our design. With $Q=256$, we obtain the best performance on
five of six tasks, with the sole exception being \textbf{ANT}. Notably, the
best-performing baseline on this task, CMA-ES, exhibits a much higher variance
(std.\ 0.928) compared to ours (0.012), underscoring its instability to
initialization. With $Q=128$, we achieve the best results on three out of five
tasks. While ExPT \citep{nguyen2024expt} achieves competitive results on the
remaining tasks, it heavily relies on pretraining from larger datasets and is
not directly comparable to other baselines. Nevertheless, our method delivers top-tier performance across all tasks.

In \Cref{tab:abl_table}, we further compare our method with
\ac{mle}-\ac{lebm} and \ac{nce} (with ratio decomposition) under both \ac{ld}
and \ac{svgd} sampling. We find that the vanilla \ac{mle} paradigm, and even
\ac{nce}, perform poorly in the offline \ac{bbo} setting—likely due to the
complex, highly multimodal joint latent space of $(\x,y)$. By contrast, our
\ac{n2ce}-based model produces strong results even with basic \ac{ld} sampling,
in line with our analysis on the importance of noise magnitude. Extended
results, including ablations over \texttt{M,K} (\cref{tab:abl_K_M}) and
robustness to query budget $Q$ (\cref{tab:abl_Q}), are deferred to
\cref{appx:db_results}.

\section{Conclusion}
In this work, we revisited \ac{nce} from the perspective of noise magnitude, an often-overlooked factor in ratio estimation. Building on this view, we introduced the \ac{n2ce} framework, interpreting a ``noisier'' \ac{nce} objective as a controlled approximation to \ac{mle} and providing both theoretical and practical insights into its behavior. Our analysis shows that, under mild conditions, the \ac{n2ce} gradient aligns with the \ac{mle} gradient when the noise magnitude is sufficiently large, and we further examined the finite-sample regime to motivate simple yet effective regularizations.

Empirically, \ac{n2ce} delivers consistent improvements across image modeling, anomaly detection, and offline black-box optimization, outperforming strong baselines and demonstrating its versatility. More broadly, our results suggest that \ac{n2ce} provides a practical and theoretically grounded framework for learning rewards and critics. Looking ahead, we see particular promise in extending this framework to discrete domains such as language modeling and to multimodal tasks like text-to-image and text-to-video generation, where traditional \ac{nce} remains limited.

\section*{Acknowledgement}
Y.~W. is partially supported by NSF DMS-2415226, DARPA W912CG25CA007 and research gift funds from Amazon and Qualcomm. We gratefully acknowledge Lambda, Inc. for providing the computational resources used in this project. We would also like to express our gratitude to the anonymous reviewer PhJz, for drawing our attention to the connection with NWJ and for contributing directly to~\cref{sec:info_theory_p}. We also sincerely thank the anonymous AC FyXy for the careful handling of our submission and for the time and attention devoted to the review process.

\clearpage
\newpage
\section*{Reproducibility Statement}
We provide complete statements and proofs of all theoretical results in \cref{appx:statement_n_proof}. Implementation details, including key hyperparameters, network architectures, and step-by-step training instructions, are described in \cref{appx:tr_details_n_arch}. Comprehensive descriptions of our experimental setups, datasets, and evaluation metrics are available in \cref{appx:add_exp}. We will release code, data, and model checkpoints upon acceptance of this manuscript to ensure full reproducibility.

\newpage
\bibliography{iclr2026_conference}
\bibliographystyle{iclr2026_conference}

\newpage
\appendix
\crefalias{section}{appendix}
\crefalias{subsection}{appendix}
\crefalias{subsubsection}{appendix}

\onecolumn
\section{Related Work}
\label{appx:related_work}
\paragraph{Noise contrastive estimation}
\acf{nce} \citep{hastie2009elements,sugiyama2012density,gutmann2012noise} provides a handy interface bridging the gap between discriminative and generative learning. 
An interesting line of research has been devoted to estimating unnormalized density, \ie, \acp{ebm} or \acp{lebm}, with \ac{nce}-like objectives \citep{gutmann2012noise,tu2007learning,lazarow2017introspective,ceylan2018conditional,grover2018boosted,gao2020flow,aneja2021contrastive}. Particularly, \citet{gao2020flow} recruits a normalizing flow \citep{rezende2014stochastic,kingma2018glow} as the base distribution for density estimation. \citet{aneja2021contrastive} refine the prior distribution of a pre-trained powerful VAE with noise contrastive learning. However, vanilla \ac{nce}, even with multi-stage estimation \citep{rhodes2020telescoping}, may
demonstrate less desirable results due to large gaps between target and noise distributions, dubbed as \textit{density-chasm}. In this paper, we propose to study this fundamental problem from an often-neglected perspective: the magnitude of the noise distribution used for ratio estimation. Specifically, with an increased magnitude of the noise distribution, the gradient of the \ac{nce} objective can closely match the \ac{mle} gradient. We provide theoretical and empirical analysis showing the effectiveness of our proposed framework.

\paragraph{Energy-based prior model}
As an interesting branch of \acp{ebm} \citep{xie2016theory,nijkamp2019learning,nijkamp2020anatomy,du2019implicit,du2020improved}, \citet{pang2020learning} show that an \ac{ebm} can serve as an informative prior model in the latent space, \ie, \ac{lebm}; such a prior greatly improves the model expressivity over those with non-informative gaussian priors and brings strong performance on downstream tasks \citep{yu2021unsupervised,yu2022latent,pang2020molecule,pang2021trajectory}. However, learning \ac{lebm} requires \acs{mcmc} sampling to estimate the learning gradients, which needs careful tuning and numerous iterations to converge when the target distributions are high-dimensional or highly multimodal. Commonly-used short-run MCMC \citep{nijkamp2019learning} in practice can lead to malformed energy landscapes and instability in training \citep{yu2022latent,yu2024learning, gao2020learning,nijkamp2020anatomy,du2019implicit,du2020improved}. In this work, we consider inducing the informative latent space with \ac{n2ce} integrated into variational learning; the proposed method requires no \ac{mcmc} for estimating the prior model, and shows reliable sampling quality in practice.

\paragraph{Diffusion distillation} 
Diffusion distillation has recently emerged as a key direction for accelerating sampling while retaining the strong generative performance of diffusion models. Early efforts such as consistency models \citep{song2023consistency} and progressive distillation \citep{salimans2022progressive} demonstrated that multi-step denoising trajectories could be compressed into a few steps, but these methods often faced trade-offs between speed and fidelity. More recent approaches refine this idea by explicitly distilling score information from pretrained diffusion models. For instance, Score Distillation Sampling (SDS) and its variants \citep{poole2022dreamfusion, wang2023prolificdreamer, luo2023diff} leverage gradients of the diffused KL divergence to align student and teacher scores at different noise levels. Extensions such as Diffusion-GAN \citep{wang2022diffusion} and Distribution Matching Distillation (DMD) \citep{yin2024one} incorporate adversarial or distribution-matching objectives to better handle mismatched supports in high-dimensional spaces. More recently, \citet{zhou2024score} introduced Score identity Distillation (SiD), a data-free, Fisher-divergence-based method that enables one-step generation, later improved by SiDA \citep{zhou2024adversarial}, which integrates adversarial loss to surpass the teacher’s performance. Complementary to these, reinforcement learning–based perspectives have been proposed: \citet{yoon2024maximum} formulated diffusion distillation as a maximum entropy inverse reinforcement learning problem, jointly training diffusion models and energy-based models, and further developed dynamic programming techniques to enable efficient updates. Collectively, these works highlight a growing trend toward principled frameworks that unify score matching, adversarial training, and control-theoretic formulations to accelerate and strengthen diffusion generation. Our work adds to this line by revisiting noise-contrastive estimation from the perspective of noise magnitude. In particular, we introduce \ac{n2ce}, a framework that connects ``noisier'' objectives to maximum likelihood, offering both theoretical insights and practical improvements across domains.

\paragraph{Online \ac{bbo}} An important set-up for \ac{bbo} problems is the online \ac{bbo}, also referred to as active \ac{bbo} in many literature \citep{kong2023moleculeLPT,kong2023molecule,krishnamoorthy2023diffusion,mashkaria2023generative}. For online \ac{bbo}, the proposed method can query the black-box function with limited times during training. A probably most well-known family is \ac{bo}, with a large body of prior work in the area 
\citep{shahriari2015taking,garnett2014active,moriconi2020high,moriconi2020highg,kandasamy2015high,rolland2018high,de2013bayesian,kusner2017grammar,lu2018structured,gomez2018automatic,snoek2012practical,srinivas2010gaussian,nguyen2020knowing}. Typically, \ac{bo} employs surrogates such as Gaussian Processes to model the underlying function, where the surrogates are sequentially updated by querying the black-box function with newly proposed points from an uncertainty-aware acquisition function. Other branches include derivative free methods such as the cross-entropy method \citep{rubinstein2004cross}, methods derived from the
REINFORCE trick \citep{williams1992simple,rubinstein1997optimization} and reward-weighted regression \citep{peters2007reinforcement}, \etc. Several practical approaches have combined these methodologies with Bayesian neural networks \citep{snoek2012practical,snoek2015scalable}, neural processes \citep{kim2018attentive,garnelo2018conditional,garnelo2018neural}, and ensembles of learned score models \citep{angermueller2019model,angermueller2020population,mirhoseini2020chip} with extensive efforts in developing advanced querying mechanisms or better approximation of the surrogates to address the online \ac{bbo} problem. 

\paragraph{Offline \ac{bbo}} 
Recent works have shown significant progress in the offline \ac{bbo} set-up 
\citep{
brookes2019conditioning,fannjiang2020autofocused,kumar2020model,krishnamoorthy2023diffusion,kim2024bootstrapped,yu2021roma,trabucco2021conservative,chen2022bidirectional,qi2022data,yuan2024importance,mashkaria2023generative,trabucco2022design,chen2024parallel,fu2020offline}. Among them, \cite{kumar2020model} trains a stochastic inverse mapping $p_{\bT}(\x | y) \propto \E_{p(\z | y)}[p(\x |\z, y)]$ with a conditional GAN-like model to instantiate $p(\x |\z, y)$ \citep{goodfellow2020generative,mirza2014conditional}. The method optimizes over $\z$ given the offline dataset maximum $y_{\rm max}$ and map $\z$ back to $\x$ for \ac{bbo} solutions. The optimization process over $\z$ draw approximate samples from $p(\z | y)$.
DDOM \citep{krishnamoorthy2023diffusion} consider directly parameterizing the inverse mapping $p_{\bT}(\x | y)$ with a conditional diffusion model in the input design space utilizing the expressiveness of DDPMs \citep{ho2020denoising}. 
We show in this paper that our framework demonstrates significant improvements over previous parameterizations of inverse model.
Forward methods \citep{fu2020offline,trabucco2021conservative,chen2022bidirectional,qi2022data,yuan2024importance,chen2024parallel} employ gradient ascent to optimize the learned surrogates to propose candidate solutions for \ac{bbo}. One critical issue, however, exists for the surrogate: the forward model can assign the out-of-training-distribution input designs $\x$ with erroneously large values or underestimate the true support of the distribution \citep{kumar2020model, mashkaria2023generative,chen2024parallel}, especially when the set of valid inputs lies in a low-dimensional manifold in the high-dimensional input design space. Several inspiring works have focus extensively on addressing this issue \citep{yu2021roma,trabucco2021conservative,chen2022bidirectional,qi2022data,yuan2024importance,chen2024parallel}. For example, \cite{yu2021roma,trabucco2021conservative} assign lower scores to identified outliers so that the approximated surrogate is expected to be robust. \cite{mashkaria2023generative} instead mimic the optimization trajectories of black-box optimizers, and rolls out evaluation trajectories from the trained sequence model for optimization. Although the method circumvents the need of the surrogate, it assumes knowledge of the approximate value of the true optima during optimization, and can struggle relative to other approaches on certain domains. Our method learns an informative latent space suitable for \ac{bbo}, and delivers strong results in comparison to prior works.

\newpage
\section{Detailed Derivations and Proofs}
\label{appx:statement_n_proof}

\subsection{Proof of \cref{prop:approx_grad}}
\label{appx:proof_grad}
\begin{assumption}[Well-defined ratio]
The reference distribution $q_0$ does not depend on $\bA$ and its support covers that of $p_{\bA}$. Hence
$r_{\bA}(\x)=p_{\bA}(\x)/q_0(\x)$ is well-defined for all $\x$.
\label{ass:ratio}
\end{assumption}

\begin{assumption}[Differentiability and integrability]
For each $\x$, $p_{\bA}(\x)$ is differentiable in $\bA$, and the score function $\nabla_{\bA}\log p_{\bA}(\x)$ is integrable with respect to both $q_*$ and $p_{\bA}$.
\label{ass:regularity}
\end{assumption}

\begin{proof}
Recall the objective
\[
\LL_M(\bA) =
\E_{q_*}\!\left[\log\frac{r_{\bA}}{M+r_{\bA}}\right] +
M\,\E_{q_0}\!\left[\log\frac{M}{M+r_{\bA}}\right].
\]

By Assumption~\ref{ass:regularity}, we may interchange $\nabla_{\bA}$ and $\E[\cdot]$. For the first term,
\[
\nabla_{\bA}\log\frac{r_{\bA}}{M+r_{\bA}}
=
-\nabla_{\bA}\log\!\left(1+\frac{M}{r_{\bA}}\right)
=
\frac{M}{M+r_{\bA}}\;\nabla_{\bA}\log r_{\bA}.
\]
Since $q_0$ does not depend on $\bA$, $\nabla_{\bA}\log r_{\bA}=\nabla_{\bA}\log p_{\bA}$. Hence
\[
\nabla_{\bA}\,\E_{q_*}\!\left[\log\frac{r_{\bA}}{M+r_{\bA}}\right]
=
\E_{q_*}\!\left[\frac{M}{M+r_{\bA}}\,\nabla_{\bA}\log p_{\bA}\right].
\]

For the second term,
\[
\nabla_{\bA}\log\frac{M}{M+r_{\bA}}
=
-\frac{1}{M+r_{\bA}}\nabla_{\bA} r_{\bA}
=
-\frac{r_{\bA}}{M+r_{\bA}}\;\nabla_{\bA}\log p_{\bA},
\]
so
\[
M\,\nabla_{\bA}\,\E_{q_0}\!\left[\log\frac{M}{M+r_{\bA}}\right]
=
-\,M\,\E_{q_0}\!\left[\frac{r_{\bA}}{M+r_{\bA}}\,\nabla_{\bA}\log p_{\bA}\right].
\]
Using the change of measure $\E_{q_0}[r_{\bA} g]=\E_{p_{\bA}}[g]$ from Assumption~\ref{ass:ratio}, we get
\[
M\,\E_{q_0}\!\left[\frac{r_{\bA}}{M+r_{\bA}}\,g\right]
=
\E_{p_{\bA}}\!\left[\frac{M}{M+r_{\bA}}\,g\right].
\]
Applying this with $g=\nabla_{\bA}\log p_{\bA}$ yields the stated gradient identity.

For the limit, note that $\frac{M}{M+r_{\bA}}\to 1$ pointwise as $M\to\infty$ and $0\le \frac{M}{M+r_{\bA}}\le 1$. By Assumption~\ref{ass:regularity}, dominated convergence applies, giving
\[
\lim_{M\to\infty}
\E_{q_*}\!\left[\frac{M}{M+r_{\bA}}\,\nabla_{\bA}\log p_{\bA}\right]
=
\E_{q_*}[\nabla_{\bA}\log p_{\bA}],
\]
and similarly for the expectation under $p_{\bA}$. Combining finishes the proof.
\end{proof}

\newpage
\subsection{Proof of \cref{prop:grad_finite_ver}}
\label{appx:proof_finite_error}

\begin{proof}
Since $\nabla_{\bA}\widehat{\LL}_M(\bA)$ is an unbiased estimator of $\nabla_{\bA}\LL_M(\bA)$, we have the standard bias--variance decomposition:
\begin{align}
\begin{split}
& \E\Big\|\nabla_{\bA}\widehat{\LL}_M(\bA) - \big(\E_{q_*}[\nabla_{\bA} f_{\bA}] - \E_{p_{\bA}}[\nabla_{\bA} f_{\bA}]\big)\Big\|^2 \\
&\qquad= \mathbf{Var}\!\left(\nabla_{\bA}\widehat{\LL}_M(\bA)\right)
+ \Big\|\nabla_{\bA}\LL_M(\bA) - \big(\E_{q_*}[\nabla_{\bA} f_{\bA}] - \E_{p_{\bA}}[\nabla_{\bA} f_{\bA}]\big)\Big\|^2.
\end{split}
\end{align}

\paragraph{Variance term.}
Since $\{(\x_0^i,\x_*^i)\}_{i=1}^n \sim (q_0, q_{*})$ are i.i.d., 
\begin{align}
\label{eq:var_def}
\begin{split}
\mathbf{Var}\!\left(\nabla_{\bA}\widehat{\LL}_M(\bA)\right)
&= \frac{1}{n}\mathbf{Var}\!\Bigg(
\nabla_{\bA}\log\frac{r_{\bA}(\x_*^i)}{M+r_{\bA}(\x_*^i)}
+ M\nabla_{\bA}\log\frac{M}{M+r_{\bA}(\x_0^i)}\Bigg) \\
&= \frac{1}{n}\mathbf{Var}\!\left(
\frac{M}{M+r_{\bA}(\x_*^i)}\,\nabla_{\bA}\log r_{\bA}(\x_*^i)
- \frac{M}{M+r_{\bA}(\x_0^i)}\,\nabla_{\bA} r_{\bA}(\x_0^i)
\right).
\end{split}
\end{align}
Using $(a-b)^2 \le 2a^2+2b^2$, we obtain
\begin{align}
\label{eq:var_bound1}
\mathbf{Var}\!\left(\nabla_{\bA}\widehat{\LL}_M(\bA)\right)
\le \frac{2}{n}\E\!\left[\frac{M}{M+r_{\bA}(\x_*^i)}\,\nabla_{\bA}\log r_{\bA}(\x_*^i)\right]^2 + 
\frac{2}{n}\E\!\left[\frac{M}{M+r_{\bA}(\x_0^i)}\,\nabla_{\bA} r_{\bA}(\x_0^i)\right]^2.
\end{align}
Since $\frac{M}{M+r} \le 1$ and $\frac{M}{M+r} \le \frac{M}{r}$ for any $M,r>0$, we further bound
\begin{align}
\label{eq:var_bound2}
\mathbf{Var}\!\left(\nabla_{\bA}\widehat{\LL}_M(\bA)\right)
&\le \frac{2}{n}\,\E\!\left[\big\|\nabla_{\bA}\log r_{\bA}(\x_*^i)\big\|^2\right] \\
&+ \frac{2}{n}\,\min\!\left\{\,M^2\,\E\!\left[\big\|\nabla_{\bA}\log r_{\bA}(\x_0^i)\big\|^2\right],\;
\E\!\left[\big\|\nabla_{\bA} r_{\bA}(\x_0^i)\big\|^2\right]\right\}. \nonumber
\end{align}
Thus the variance is controlled by the second moments of the score (or equivalently $\nabla_{\bA} r_{\bA}$), and in the typical regime where the first branch dominates, it decays as $O(M^2/n)$.

\paragraph{Bias term.}
From the proof of Proposition~\ref{prop:approx_grad}, 
\begin{align}
\label{eq:bias_bound}
\begin{split}
\Big\|\nabla_{\bA}\LL_M(\bA) - \big(\E_{q_*}[\nabla_{\bA} f_{\bA}] - \E_{p_{\bA}}[\nabla_{\bA} f_{\bA}]\big)\Big\|^2
&= \Bigg\|\int_{\x}\frac{r_{\bA}(\x)}{M+r_{\bA}(\x)}\,
\big(q_*(\x)-p_{\bA}(\x)\big)\,\nabla_{\bA} f_{\bA}(\x)\,d\x\Bigg\|^2 \\
&\le \frac{1}{M^2}\Bigg\|\int_{\x} r_{\bA}(\x)\,
\big(q_*(\x)-p_{\bA}(\x)\big)\,\nabla_{\bA} f_{\bA}(\x)\,d\x\Bigg\|^2.
\end{split}
\end{align}
Hence the squared bias vanishes at rate $O(1/M^2)$.

\paragraph{Conclusion.}
Combining \eqref{eq:var_bound2} and \eqref{eq:bias_bound}, the mean squared error of $\nabla_{\bA}\widehat{\LL}_M(\bA)$ satisfies
\[
\E\Big\|\nabla_{\bA}\widehat{\LL}_M(\bA) - \big(\E_{q_*}[\nabla_{\bA} f_{\bA}] - \E_{p_{\bA}}[\nabla_{\bA} f_{\bA}]\big)\Big\|^2
= O\!\left(\tfrac{M^2}{n}\right) + O\!\left(\tfrac{1}{M^2}\right),
\]
which completes the proof.
\end{proof}

\newpage
\subsection{Proof of \cref{thm:convergence_rate}}
\label{appx:proof_of_conv_thm}
As in \citet{liu2021analyzing}, we work with the exponential family
\[
p_{\bA}(\x) = \exp\!\left(\bA^\top \Tilde{T}(\x) - \log Z(\bA)\right),
\]
where $\bA\in\mathbb{R}^m$ is the natural parameter, $\Tilde{T}(\x)\in\mathbb{R}^m$ are sufficient statistics, and $Z(\bA)$ is the partition function.
For convenience in derivatives, we also use the extended parameterization
\[
\bTau := [\bA,\,\alpha_Z],\quad \alpha_Z := \log Z(\bA),
\qquad
T(\x) := [\,\Tilde{T}(\x),\,-1\,],
\]
so that $p_{\bTau}(\x)=\exp(\bTau^\top T(\x))$. Expectations $\E_{\bTau}[\cdot]$ are taken w.r.t.\ $p_{\bTau}$.

\begin{assumption}[Exponential-family regularity]
\label{ass:exp_reg}
The map $\bA\mapsto \log Z(\bA)$ is twice continuously differentiable on the parameter set $\mathcal{A}$.
Equivalently, $p_{\bA}(\x)$ is $C^2$ in $\bA$ and differentiation under the integral is valid for the quantities considered.
\end{assumption}

\begin{assumption}[Reference coverage and Fisher regularity]
\label{ass:ref_fisher}
The reference $q_0$ covers the support of $p_{\bA}$, so $r_{\bA}(\x)=p_{\bA}(\x)/q_0(\x)$ is well-defined.
Moreover, there exist constants $0<\lambda_{\min}\le\lambda_{\max}<\infty$ such that, for all extended parameters $\bTau$,
\[
\lambda_{\min} I \;\preceq\; \E_{\bTau}\!\big[T(\x)T(\x)^\top\big] \;\preceq\; \lambda_{\max} I.
\]
\end{assumption}

\begin{assumption}[Bounded parameter set]
\label{ass:bounded_param}
$\|\bA\|_2 \le \omega$ for all $\bA\in\mathcal{A}$.
\end{assumption}

\begin{lemma}[Convexity]
\label{lem:n2ce_convex}
Let $r_{\bA}=p_{\bA}/q_0$ with $p_{\bA}$ in the exponential family. Under Assumptions~\ref{ass:exp_reg} and \ref{ass:ref_fisher}, the negative \ac{n2ce} objective in \cref{equ:n2ce_obj} is convex in $\bA$.
\end{lemma}

\begin{proof}
Use the extended parameterization $p_{\bTau}(\x)=\exp(\bTau^\top T(\x))$. 
By Assumption~\ref{ass:exp_reg} and finite second moments implied by Assumption~\ref{ass:ref_fisher}, we may compute the derivatives inside the integral.

The gradient can be written as
\begin{equation}
\label{eq:grad_n2ce}
\nabla_{\bTau}\LL_M(\bTau) 
= \int \frac{M\,q_0(\x)}{p_{\bTau}(\x)+M q_0(\x)}\,\big(q_*(\x)-p_{\bTau}(\x)\big)\,T(\x)\,d\x.
\end{equation}
Differentiating once more gives the Hessian
\begin{equation}
\label{eq:hess_n2ce}
\nabla_{\bTau}^2 \LL_M(\bTau)
= - \int \frac{M\,p_{\bTau}(\x)\,q_0(\x)}{\big(p_{\bTau}(\x)+M q_0(\x)\big)^2}\,
\big(q_*(\x)+M q_0(\x)\big)\,T(\x)T(\x)^\top\,d\x.
\end{equation}
For each $\x$, the scalar weight 
\[
\frac{M\,p_{\bTau}(\x)\,q_0(\x)}{\big(p_{\bTau}(\x)+M q_0(\x)\big)^2}\,\big(q_*(\x)+M q_0(\x)\big)\;\ge\;0,
\]
and $T(\x)T(\x)^\top\succeq 0$. Hence the integrand in \eqref{eq:hess_n2ce} is positive semidefinite; with the leading minus sign, the Hessian is negative semidefinite for all $\bTau$. Therefore $-\LL_M$ is convex in $\bA$.
\end{proof}

\begin{lemma}[Condition number]
\label{lem:n2ce_cond_hess}
Under Assumptions~\ref{ass:exp_reg} and \ref{ass:ref_fisher},
\begin{equation}
\label{eq:limit_hessian}
\lim_{M\to\infty}\big(-\nabla_{\bTau}^2 \LL_M(\bTau^*)\big)
= \E_{\bTau^*}\!\big[T(\x)T(\x)^\top\big],
\end{equation}
and hence
\begin{equation}
\label{eq:kappa_bound}
\kappa^* \;=\; \kappa\!\left(-\lim_{M\to\infty}\nabla_{\bTau}^2 \LL_M(\bTau^*)\right)
\;\le\; \frac{\lambda_{\max}}{\lambda_{\min}}.
\end{equation}
Moreover, for large but finite $M$,
\begin{equation}
\label{eq:kappa_finiteM}
\kappa\!\left(-\nabla_{\bTau}^2 \LL_M(\bTau^*)\right)
\;\le\; \frac{\lambda_{\max}}{\lambda_{\min}}\,(1+O(1/M)).
\end{equation}
\end{lemma}

\begin{proof}
From \eqref{eq:hess_n2ce}, for any $\bTau$ we can write
\[
\nabla_{\bTau}^2 \LL_M(\bTau)
= -\,\E_{p_{\bTau}}\!\big[\phi_M(\x;\bTau)\,T(\x)T(\x)^\top\big],
\qquad
\phi_M(\x;\bTau)=\frac{1+\tfrac{q_*(\x)}{M q_0(\x)}}{\big(1+\tfrac{p_{\bTau}(\x)}{M q_0(\x)}\big)^2}.
\]
Pointwise, $\phi_M(\x;\bTau)\to 1$ as $M\to\infty$, and $0<\phi_M(\x;\bTau)\le 1+O(1/M)$ uniformly. By Assumption~\ref{ass:ref_fisher}, $\E_{p_{\bTau}}\!\big[\|T(\x)\|^2\big] < \infty$, so dominated convergence applies. Assumption~\ref{ass:exp_reg} ensures differentiation and expectations are well-defined. Hence
\[
\lim_{M\to\infty}\nabla_{\bTau}^2 \LL_M(\bTau) 
= -\,\E_{p_{\bTau}}[T(\x)T(\x)^\top].
\]
Evaluating at the optimum $\bTau^*$ yields \eqref{eq:limit_hessian}. By Assumption~\ref{ass:ref_fisher}, the eigenvalues of $\E_{\bTau^*}[T(\x)T(\x)^\top]$ lie in $[\lambda_{\min},\lambda_{\max}]$, which gives the condition-number bound \eqref{eq:kappa_bound}.

For finite $M$, write
\[
-\nabla_{\bTau}^2 \LL_M(\bTau^*)
= \E_{\bTau^*}[T(\x)T(\x)^\top] \;+\; R_M,\qquad
R_M \;=\; \E_{p_{\bTau^*}}\!\big[(1-\phi_M(\x;\bTau^*))\,T(\x)T(\x)^\top\big].
\]
Since $1-\phi_M(\x;\bTau^*)=O(1/M)$ and $\E_{\bTau^*}[\|T(\x)\|^2]<\infty$, we have $\|R_M\|=O(1/M)$. Standard eigenvalue perturbation bounds then give \eqref{eq:kappa_finiteM}.
\end{proof}

\begin{theorem}[Convergence rate]
\label{thm:full_proof}
Let $\LL_M$ denote the \ac{n2ce} objective. 
Under Assumptions~\ref{ass:exp_reg} and \ref{ass:ref_fisher}, 
the Hessian of $-\LL_M$ satisfies the local smoothness and curvature bounds required in Theorem~5.1 of \citet{liu2021analyzing}. 
Consequently, normalized gradient descent on $\LL_M$ with $M\to\infty$ converges to the optimum $\bTau^*$ at a rate polynomial in the condition number. 
In particular, for any tolerance $\delta>0$,
\[
T \;\le\; \Big(\tfrac{\lambda_{\max}}{\lambda_{\min}}\Big)^{\!3}
\frac{\|\bTau_0-\bTau^*\|_2^2}{\delta^2}
\]
iterations suffice to obtain some iterate $\bTau_t$ with $\|\bTau_t-\bTau^*\|_2 \le \delta$.
\end{theorem}

\begin{proof}
This follows from Theorem~5.1 of \citet{liu2021analyzing}, 
together with convexity (Lemma~\ref{lem:n2ce_convex}) 
and the condition number bound (Lemma~\ref{lem:n2ce_cond_hess}), 
which verify the smoothness and curvature assumptions in that theorem.
\end{proof}

\paragraph{Remark.}
The \ac{n2ce} loss inherits favorable convexity and well-conditioned curvature from the exponential family, 
ensuring that normalized gradient methods converge efficiently to the optimum.

\clearpage
\newpage
\section{Network Architecture and Implementation Details}
\label{appx:tr_details_n_arch}
\paragraph{Code} We promise to release code, data and model checkpoints upon acceptance of this manuscript.

\subsection{How to Learn a Latent Space Model with \texorpdfstring{\acs{n2ce}}{}}
\label{appx:dlvm_alg_n_implem}

\paragraph{ELBO}
We show how the proposed form can be integrated into the variational learning scheme to learn an accurate and explicit density $p_{\bA}$. We can first see that the target distribution for $p_{\bA}$ is the aggregated posterior $q_{K+1}(\z) = \int_{\x} q_{\bP}(\z|\x) p(\x) d\x$ \citep{tomczak2018vae}, where $q_{\bP}$ denotes the posterior distribution parameterized by the encoder network. When considering ratio decomposition, we can form a corresponding KL divergence which leads to the following ELBO: 
\begin{equation}
    \begin{aligned}
    \ELBO_{{\bT}, \bP} 
         = \E_{q_{\bP}} [
            \log p_{\bB}(\x|\z)] 
          - \KL(q_{\bP} \| 
                r_{{\bA}_K} q_K)
         - \sum_{k=0}^{K-1} 
            \KL(q_{k+1} \| r_{{\bA}_k} q_k),
    \end{aligned}
    \label{equ:elbo_ratio}
\end{equation}
which is a valid ELBO by the non-negativity of the KL Divergence. ${\bB}$ denotes the decoder network.

\paragraph{Training algorithms}
We directly learn the log-ratio estimator $\Tilde{f}_{{\bA}_k} = \log r_{{\bA}_k}$ for implementation. Let $\sigma(\cdot)$ denote the sigmoid function, the gradient for learning each ratio estimator $\{p_{{\bA}_k}\}_{k=0}^K$ can be derived from \cref{equ:n2ce_obj} as:
\begin{equation}
\begin{aligned}
\displaystyle
\nabla_{{\bA}_k} \ELBO &\approx 
\nabla_{{\bA}_k}
\E_{q_{k+1}} 
\left[\log \sigma\left(\Tilde{f}_{{\bA}_k} - \log M \right)\right] +
\nabla_{{\bA}_k} M \E_{q_{k}} 
\left[\log \left(1 - \sigma\left(\Tilde{f}_{{\bA}_k} - \log M \right)\right) \right].
\end{aligned}
\label{equ:prior_grad_sigmoid}
\end{equation}

Let $\bPs = \{\bB, \bP\}$ collect the parameters of the decoder and encoder networks. For $q_{\bP}(\z|\x)$ and $p_{\bB}(\x|\z)$, we can calculate the learning gradient as follows:
\begin{equation}
\begin{aligned}
\displaystyle
\nabla_{\bPs} \ELBO 
 = \nabla_{\bPs} \E_{q_{\bP}} 
    [\log p_{\bB} ]
- \nabla_{\bP} 
  \KL( q_{\bP} \| p_0 )
 - \nabla_{\bP} \E_{q_{\bP}}
    \left[ 
    \sum_{k=0}^{K} \Tilde{f}_{{\bA}_k}
    \right].
\end{aligned}
\label{equ:poste_grad}
\end{equation}
\cref{alg:example} summarizes the training procedure.

\paragraph{Implementation}
When optimizing the $\ELBO$ in \cref{equ:elbo_ratio}, we randomly choose one intermediate stage to optimize at each training iteration as in \citet{ho2020denoising}. 
To construct these intermediate distributions, we follow \citet{yu2022latent} to spherically interpolates between pairs of samples $\z_{K+1} \sim q_{K+1}$ and $\z_0 \sim q_0$; samples from the $k$-th intermediate distribution $q_k$ are: $\z_k = \sqrt{1 - \sigma_k^2} \z_0 + \sigma_k \z_{K + 1}$, where the $\{\sigma_k\}_{k=1}^K$ form an increasing sequence from $0$ to $1$ controlling the distance between these implicitly defined distributions $\{q_k\}_{k=1}^K$.
We use a linear schedule to specify $\sigma_k^2$ and construct the intermediate distributions $\{q_k\}_{k=0}^K$.

\begin{algorithm}[!htbp]
   \caption{Variational Learning with \acs{n2ce}}
   \label{alg:example}
\begin{algorithmic}
   \STATE {\bfseries Input:} 
    dataset
    $\cD:\{\x^{(i)}\}_{i=1}^N$,
    params. 
    $\left(\{{\bA}_k\}_{k=0}^K, {\bB}, {\bP}\right)$,
   \STATE {\bfseries Output:} 
   updated params. $\left(\{{\bA}^*_k\}_{k=0}^K, {\bB}^*, {\bP}^*\right)$
   \vspace{0.15cm}
   \REPEAT
    \STATE {\bfseries posterior sampling:} 
        For each pair of $\{\x^{(i)}\}$, sample $\z^{(i)}_{K+1} \sim q_\phi(\z|\x^{(i)})$ using inference network.
    \STATE {\bfseries forming $\{q_k\}_{k=1}^K$:} 
        For each sample $\z^{(i)}_{K+1}$, sample $k \sim \text{Unif}(\{0, ..., K\})$ and obtain 
        $
        \left(\z^{(i)}_{k}, \z^{(i)}_{k+1}\right)
        \sim
        \left(q_k, q_{k+1}\right)
        $ using
        $\z_k = \sqrt{1 - \sigma_k^2} \z_0 + \sigma_k \z_{K + 1}$.
    \STATE {\bfseries learning $\{{\bA}_k\}_{k=0}^K$:} Update $\{{\bA}_k\}_{k=0}^K$ with \cref{equ:prior_grad_sigmoid}. 
    \STATE {\bfseries learning ${\bPs} = ({\bB},{\bP})$:}
    Update ${\bPs}$ with \cref{equ:poste_grad}.
    \UNTIL{converged.}
\end{algorithmic}
\end{algorithm}

\subsection{Image Modeling and Anomaly Detection}
\label{appx:img_exp_arch}
\paragraph{Architecture}
We provide detailed network architecture for the log-ratio estimator for each stage $\Tilde{f}_{{{\bA}}_k}(\z, k)$ in \cref{tab:appx_logr}. For the VAE model,
the decoder network has a simple deconvolution structure similar to DCGAN \citep{radford2015unsupervised} shown in \cref{tab:gen}.
The encoder network to embed the observed images has a fully convolutional structure, as shown in \cref{tab:enc}.

\begin{table}[!htb]
  \caption{\textbf{Network structures of the generator networks} used for the SVHN, CelebA, CIFAR-10, CelebA-HQ and MNIST (from top to bottom) datasets. ConvT($n$) indicates a transposed convolutional operation with $n$ output channels. We use \texttt{ngf=64} for the SVHN dataset and \texttt{ngf=128} for the rest. LReLU indicates the Leaky-ReLU activation function. The slope in Leaky ReLU is set to be 0.2.}
  \vspace{0.1in}
  \label{tab:gen}
  \centering
     \begin{tabular}{ccc}
     \toprule
     {\bf Layers} & {\bf Out Size} & {\bf Stride}\\ \hline
     Input: $\z$ & 1x1x100 &- \\
	 4x4 ConvT(ngf x $8$), LReLU & 4x4x(ngf x 8) & 1 \\
	 4x4 ConvT(ngf x $4$), LReLU & 8x8x(ngf x 4) & 2\\
	 4x4 ConvT(ngf x $2$), LReLU & 16x16x(ngf x 2) & 2\\
	 4x4 ConvT(3), Tanh & 32x32x3 & 2 \\ 
     \hline
     {\bf Layers} & {\bf Out Size} & {\bf Stride}\\ \hline
     Input: $\z$ & 1x1x100 &- \\
		4x4 ConvT(ngf x $8$), LReLU & 4x4x(ngf x 8) & 1 \\
		4x4 ConvT(ngf x $4$), LReLU & 8x8x(ngf x 4) & 2\\
		4x4 ConvT(ngf x $2$), LReLU & 16x16x(ngf x 2) & 2\\
		4x4 ConvT(ngf x $1$), LReLU & 32x32x(ngf x 1) & 2\\
		4x4 ConvT(3), Tanh & 64x64x3 & 2 \\ 
     \hline
     {\bf Layers} & {\bf Out Size} & {\bf Stride}\\ \hline
     Input: $\z$ & 1x1x128 & -\\
		8x8 ConvT(ngf x $8$), LReLU & 8x8x(ngf x 8) & 1 \\
		4x4 ConvT(ngf x $4$), LReLU & 16x16x(ngf x 4) & 2\\
		4x4 ConvT(ngf x $2$), LReLU & 32x32x(ngf x 2) & 2\\
		3x3 ConvT(3), Tanh & 32x32x3 & 1 \\ 
     \hline
     {\bf Layers} & {\bf Out Size} & {\bf Stride}\\ \hline
     Input: $\z$ & 1x1x128 &- \\
        4x4 ConvT(ngf x $16$), LReLU & 4x4x(ngf x 16) & 1 \\
	4x4 ConvT(ngf x $8$), LReLU & 8x8x(ngf x 8) 
        & 2 \\
	4x4 ConvT(ngf x $4$), LReLU & 16x16x(ngf x 4) 
        & 2\\
        4x4 ConvT(ngf x $4$), LReLU & 32x32x(ngf x 4) 
        & 2\\
	4x4 ConvT(ngf x $2$), LReLU & 64x64x(ngf x 2)
        & 2\\
	4x4 ConvT(ngf x $1$), LReLU & 128x128x(ngfx1)     & 2\\
	4x4 ConvT(3), Tanh & 256x256x3 & 2 \\ 
     \hline
     {\bf Layers} & {\bf Out Size} & {\bf Stride}\\ \hline
     Input: $\z$ & 1x1x8 & -\\
		7x7 ConvT(ngf x $8$), LReLU & 7x7x(ngf x 8) & 1 \\
		4x4 ConvT(ngf x $4$), LReLU & 14x14x(ngf x 4) & 2\\
		4x4 ConvT(ngf x $2$), LReLU & 28x28x(ngf x 2) & 2\\
		3x3 ConvT(1), Tanh & 28x28x1 & 1 \\ 
     \bottomrule
     \end{tabular}
 \end{table}
 
\begin{table}[!htb]
  \caption{\textbf{Network structures of the encoder networks} used for the SVHN, CelebA, CIFAR-10, CelebA-HQ and MNIST (from top to bottom) datasets. Conv($n$)Norm indicates a convolutional operation with $n$ output channels followed by the Instance Normalization \citep{ulyanov2016instance}. We use \texttt{nif=64} for all the datasets. LReLU indicates the Leaky-ReLU activation function. The slope in Leaky ReLU is set to be 0.2.}
  \label{tab:enc}
  \vspace{0.1in}
  \centering
     \begin{tabular}{ccc}
     \hline
     {\bf Layers} & {\bf Out Size} & {\bf Stride}\\ \hline
     Input: $\x$ & 32x32x3 &- \\
	 3x3 Conv(nif x $1$)Norm, LReLU & 32x32x(nif x 1) & 1 \\
	 4x4 Conv(nif x $2$)Norm, LReLU & 16x16x(nif x 2) & 2\\
	 4x4 Conv(nif x $4$)Norm, LReLU &  8x8x(nif x 4) & 2\\
      4x4 Conv(nif x $8$)Norm, LReLU &  4x4x(nif x
     8) & 2\\
	 4x4 Conv(nemb)Norm, LReLU & 1x1x(nz) & 1 \\ 
     \hline
      {\bf Layers} & {\bf Out Size} & {\bf Stride}\\ \hline
     Input: $\x$ & 64x64x3 &- \\
	 3x3 Conv(nif x $1$)Norm, LReLU & 64x64x(nif x 1) & 1 \\
	 4x4 Conv(nif x $2$)Norm, LReLU & 32x32x(nif x 2) & 2\\
	 4x4 Conv(nif x $4$)Norm, LReLU & 16x16x(nif x 4) & 2\\
      4x4 Conv(nif x $8$)Norm, LReLU &  8x8x(nif x
     8) & 2\\
      4x4 Conv(nif x $8$)Norm, LReLU &  4x4x(nif x
     8) & 2\\
	 4x4 Conv(nemb)Norm, LReLU & 1x1x(nz) & 1 \\ 
     \hline
      {\bf Layers} & {\bf Out Size} & {\bf Stride}\\ \hline
     Input: $\x$ & 32x32x3 &- \\
	 3x3 Conv(nif x $1$)Norm, LReLU & 32x32x(nif x 1) & 1 \\
	 4x4 Conv(nif x $2$)Norm, LReLU & 16x16x(nif x 2) & 2\\
	 4x4 Conv(nif x $4$)Norm, LReLU &  8x8x(nif x 4) & 2\\
      4x4 Conv(nif x $8$)Norm, LReLU &  4x4x(nif x
     8) & 2\\
	 4x4 Conv(nz)Norm, LReLU & 1x1x(nz) & 1 \\
     \hline
     {\bf Layers} & {\bf Out Size} & {\bf Stride}\\ \hline
     Input: $\x$ & 256x256x3 &- \\
	 3x3 Conv(nif x $1$)Norm, LReLU & 256x256x(nif x 1) & 1 \\
	 4x4 Conv(nif x $2$)Norm, LReLU & 128x128x(nif x 2) & 2\\
	 4x4 Conv(nif x $4$)Norm, LReLU &  64x64x(nif x 4) & 2\\
      4x4 Conv(nif x $4$)Norm, LReLU &  32x32x(nif x 4) & 2\\
      4x4 Conv(nif x $8$)Norm, LReLU &  16x16x(nif x 8) & 2\\
      4x4 Conv(nif x $8$)Norm, LReLU &  8x8x(nif x
      8) & 2\\
      4x4 Conv(nif x $8$)Norm, LReLU &  4x4x(nif x
      8) & 2\\
	 4x4 Conv(nz)Norm, LReLU & 1x1x(nz) & 1 \\
     \hline
     {\bf Layers} & {\bf Out Size} & {\bf Stride}\\ \hline
     Input: $\x$ & 28x28x3 &- \\
	 3x3 Conv(nif x $1$)Norm, LReLU & 28x28x(nif x 1) & 1 \\
	 4x4 Conv(nif x $2$)Norm, LReLU & 14x14x(nif x 2) & 2\\
	 4x4 Conv(nif x $4$)Norm, LReLU &  7x7x(nif x 4) & 2\\
      4x4 Conv(nif x $8$)Norm, LReLU &  3x3x(nif x
     8) & 2\\
	 3x3 Conv(nz)Norm, LReLU & 1x1x(nz) & 1 \\
     \bottomrule
     \end{tabular}
 \end{table}

\paragraph{Hyperparameters and training details}
In the image modeling experiments, for the \ac{mle}-\ac{lebm} baseline, we run \ac{ld} for $T=60$ iterations for prior updates during training with a step size of $s=0.4$. For test time sampling from \acp{lebm} trained with different objectives, we consistently \ac{ld} for $T=100$ with the step size of $s=0.4$ to draw samples from $p_{\bA}$. In the anomaly detection experiments, we follow the same experiment configuration as used in \citet{yu2024learning}.

We use the linear schedule as in and \citet{yu2022latent} to specify $\sigma_k^2$ and construct the intermediate distributions $\{q_k\}_{k=0}^K$. We set the number of stages to 3. Specifically, the sequence of $\{\sigma_k^2\}_{k=0}^3$ is: 
$
\{\sigma_k^2\}_{k=0}^3 = [0.01, 0.69175489, 0.92238785, 0.99974058].
$
To estimate the \cref{equ:n2ce_obj}, we can use Monte-Carlo average: 
$
\LL_{M}({\bA}_k) \approx \frac{1}{n_1} \sum_{i=1}^{n_1} 
\left[\log \frac{r_{{\bA}_k}(\z_i)}{M + r_{{\bA}_k}(\z_i)}\right]
+ 
\frac{M}{n_2} \sum_{j=1}^{n_2} \left[\log \frac{M}{M + r_{{\bA}_k}(\z_j)}\right]
$, where $\z_i \sim q_{k+1}, \z_j \sim q_k$. We follow \citet{mnih2012fast} and set $n_2 = M n_1$ in practice, where $M=100$ for all experiments; $n_1$ is the batch size. When pre-training the VAE model, we set the weight for reconstruction to ${[1, 1, 100, 1]}$ on the four datasets and KL weights to 1; when training jointly with the \ac{lebm}, we set the weight for reconstructing $\x$ to ${[5, 5, 2, 1]}$ for the term $\E_{q_{\bP}} [\log p_{\bB}(\x|\z)]$ on SVHN, CelebA64, CIFAR10 and CelebAHQ datasets, respectively. We set the weight for controlling the KL term between $q_{\bP}$ and $p_{\bA}$ to $1$, \ie, $
\KL(q_{\bP} \|r_{{\bA}_K} q_K)
- \sum_{k=0}^{K-1} \KL(q_{k+1} \| r_{{\bA}_k} q_k)
$. Other hyperparameters are shared across all the experiments.

The parameters of all the networks are initialized with the default pytorch methods \citep{paszke2019pytorch}. We use the Adam optimizer \citep{kingma2014adam} with $\beta_1=0.5$ and $\beta_2=0.999$ to train the decoder and encoder networks; the log-ratio estimator network is trained with default $(\beta_1, \beta_2)$. The initial learning rates of the generator, encoder and log-ratio networks are \texttt{2e-4}. We further perform gradient clipping by setting the maximal gradient norm as 100 for these networks. We run the experiments on a A6000 GPU with the batch size of $128$. Training converges within 100K iterations on all the datasets.

\newpage
\subsection{Reward and Critic Learning}
\label{appx:reward_n_critic_learn}
We adapt the objectives in prior work as follows:
\begin{itemize}
    \item[(i)] For DxMI \citep{yoon2024maximum}, we replace the contrastive-divergence-style EBM update (Sec.~3.2 of \cite{yoon2024maximum}) with our objective in \cref{equ:n2ce_obj}, treating the true data distribution as the target and diffusion samples as noise.
    \item[(ii)] For SiD$^2$A \citep{zhou2024adversarial}, we substitute the discriminator loss in Eq.~(10) in \cite{zhou2024adversarial} with our objective in \cref{equ:n2ce_obj}, again using true data as the target and diffusion samples as noise.
\end{itemize}

All experiments retain the original network architectures and hyperparameters unless noted:
\begin{itemize}
    \item[(i)] DxMI: direct ratio regularization weight is $1.5$ on CIFAR-10 and $1.55$ on ImageNet64$\times$64; on ImageNet64$\times$64, the value function learning rate is \texttt{3e-5}.
    \item[(ii)] SiD$^2$A: direct ratio regularization weight is $0.5$ for CIFAR-10 (conditional and unconditional) and $0.4$ for ImageNet64$\times$64; we set $\alpha=1.2$ throughout.
    \item[(iii)] All remaining hyperparameters (e.g., sampler learning rates, batch sizes) follow the defaults, making training iterations directly comparable as a measure of computational cost, particularly for SiD$^2$A.
\end{itemize}

\subsection{Offline Black-Box Optimization}
\label{appx:bbo_exp_arch}
\paragraph{Architecture}
We provide detailed network architecture for the log-ratio estimator $\Tilde{f}_{{{\bA}}_k}(\z, k)$ in \cref{tab:appx_logr}; we adopt the same architecture as in \cref{appx:img_exp_arch} throughout the experiments.
The generator networks for $\x$ and $y$ is a 5-layer MLP structure shown in \cref{tab:gen_n_enc}. We employ a single network that directly outputs the $(\x, y)$ pair to implement $g_{{\bB},\x}$, $g_{{\bB},y}$. For discrete tasks, the $g_{{\bB},\x}$ head outputs logits for reconstruction. The encoder network to embed the input is a 3-layer MLP, as shown in \cref{tab:gen_n_enc}. For discrete tasks, we add an additional embedding layer to map discrete inputs to continuous vectors. The network architectures are shared across different task domains; we only adjust the input size for inputs from different tasks. For function values $y$ and continuous $\x$, we normalize these values for training using the offline dataset maxium ${\x}_{\rm max}, y_{\rm max}$ and minimum ${\x}_{\rm min}, y_{\rm min}$: ${\x}_{\rm train} = \frac{\x - {\x}_{\rm min}}{{\x}_{\rm max} - {\x}_{\rm min}}$. $y_{\rm train} = \frac{y - y_{\rm min}}{y_{\rm max} - y_{\rm min}}$.

\begin{table}[ht!]
\begin{minipage}{.47\textwidth}
    \caption{\textbf{Network architecture for the log-ratio estimator $\Tilde{f}_{{{\bA}}_k}(\z, k)$.} $N$ is set to $3$ for image modeling and anomaly detection experiments, and $16$ for the offline \ac{bbo} experiments. ReLU denotes the Leaky ReLU activation function. The slope in Leaky ReLU is set to 0.2. For SVHN and CelebA datasets, we use \texttt{nz=100}. For the CIFAR-10 and CelebA-HQ datasets, we use \texttt{nz=128} and \texttt{nz=512}, respectively. We use \texttt{nz=8} for anomaly detection on the MNIST dataset; \texttt{nz=20} is the latent space dimension used for offline \ac{bbo} experiments.}
    \label{tab:appx_logr} 
    \vskip 0.1in
    \centering
    \small
    \begin{tabular}{ccc}
        \toprule
        {\bf Layers} & {\bf Output size} & {\bf Note} \\
        \cmidrule(r){1-1} \cmidrule(r){2-2} \cmidrule(r){3-3}
        
        \multicolumn{3}{c}{\bf Indx. Emb.} \\
        \hline
        Input: $k$ & $1$ & \tabincell{c}{stage indx. \\ of \ac{n2ce}} \\
        Sin. emb. & $128$ & {} \\
        Linear, LReLU & $128$ & \tabincell{c}{ neg\_slope \\ $0.2$} \\
        Linear & $128$ & {} \\
        
        \hline
        \multicolumn{3}{c}{\bf Input Emb.} \\
        \hline
        Input: $\z$ & $nz = 20$ & {} \\
        Linear, LReLU & $128$ & \tabincell{c}{neg\_slope \\ $0.2$} \\
        Linear & $128$ & {} \\
        
        \hline
        \multicolumn{3}{c}{$\mathbf{\Tilde{f}_{{{\bA}}_k}(\z, k)}$} \\
        \hline
        Input: \tabincell{c}{$\z, t$} & \tabincell{c}{$1$, $nz = 20$} & {} \\
        \tabincell{c}{Input \& Indx. \\ Emb.} & $128 \times 2$ & \tabincell{c}{Emb. of \\ each input} \\
        Concat. & \tabincell{c}{$256$} & {} \\
        LReLU, Linear & $128$ & \tabincell{c}{neg\_slope \\ $0.2$} \\
        N ResBlocks & $128$ & \tabincell{c}{LReLU, Linear \\ $+$ Input Skip} \\
        LReLU, Linear & $1$ & energy score \\
        \bottomrule
    \end{tabular}
\end{minipage}\quad\quad\quad
\begin{minipage}{.47\textwidth}
\caption{\textbf{Generator and encoder network architectures.} LReLU indicates the Leaky-ReLU activation function. $d_{\x}$ is the dimension of input $\x$ mentioned in \cref{tab:appx_db_stats}. For discrete tasks, $d_{\x}$ specifies the length of input, and $L$ denotes number of possible discrete tokens.}
  \label{tab:gen_n_enc}
  \vskip 0.1in
  \centering
  \small
     \begin{tabular}{ccc}
     \toprule
     \multicolumn{3}{c}{\bf Generator} \\
     \hline
     {\bf Layers} & {\bf Out Size} & {\bf Note} \\ 
     \hline
     Input: $\z$ & $nz = 20$ {} \\
     \hline
	 Linear, LReLU & $128$ & 
      \multirow{4.5}{*}{
      \tabincell{c}{neg\_slope \\ $0.2$}
      } \\
      Linear, LReLU & $64$ & {}\\
      Linear, LReLU & $32$ & {}\\
      Linear, LReLU & $32$ & {}\\
      \hline
      \multicolumn{3}{c}{\bf Continuous Output} \\
	 Linear & $d_{\x} + 1$ & output $(\x, y)$\\ 
      \multicolumn{3}{c}{\bf Discrete Output} \\
	 Linear & $d_{\x} \times L + 1$ & {num. tokens. $L$}\\ 
     \bottomrule
     % \end{tabular} \\
     % %%%%%%%%%%%%%%%%%%%%%%
     % \begin{tabular}{ccc}
     \toprule
     \multicolumn{3}{c}{\bf Encoder} \\
     \hline
     {\bf Layers} & {\bf Out Size} & {\bf Note}\\ 
     \hline
     \multicolumn{3}{c}{\tabincell{c}{{\bf Embedding} \\ (for discrete $\x$ only)}} \\
     Input: $\x$ & $32 \times d_{\x}$ & emb. $\x$ \\
     \hline
     concat $(\x, y)$ & $(32\times) d_{\x} + 1$ & input $(\x, y)$\\
     \hline
      Linear, LReLU & $128$ & 
      \multirow{2.2}{*}{
      \tabincell{c}{neg\_slope \\ $0.2$}
      } \\
      Linear, LReLU & $128$ & {}\\
      \hline
      \multicolumn{3}{c}{\bf Mean head} \\
      Linear & $20$ & {}\\
      \hline
      \multicolumn{3}{c}{\bf Var. head} \\
      Linear & $20$ & {}\\
     \bottomrule
     \end{tabular}
\end{minipage}
\end{table}

\paragraph{Hyperparameters and implementation details}
We use the linear schedule as in \citet{ho2020denoising} and \citet{yu2022latent} to specify $\sigma_k^2$ and construct the intermediate distributions $\{q_k\}_{k=0}^K$. We set the number of stages to 6. Specifically, the sequence of $\{\sigma_k^2\}_{k=0}^5$ is: 
$
\{\sigma_k^2\}_{k=0}^5 = [0.01, 0.3237, 0.5165, 0.6322, 0.7132, 0.7734, 0.9997].
$
We re-weight the objective for each ${\bA}_k$ with $w_k = \sqrt{\sigma_K / \prod_{i=k}^K \sigma_i}$ to further emphasize stages closer to the final target distribution $q_{K+1}$, as in \citet{yu2022latent}. To estimate the \cref{equ:n2ce_obj}, we can use Monte-Carlo average: 
$
\LL_{M}({\bA}_k) \approx \frac{1}{n_1} \sum_{i=1}^{n_1} 
\left[\log \frac{r_{{\bA}_k}(\z_i)}{M + r_{{\bA}_k}(\z_i)}\right]
+ 
\frac{M}{n_2} \sum_{j=1}^{n_2} \left[\log \frac{M}{M + r_{{\bA}_k}(\z_j)}\right]
$, where $\z_i \sim q_{k+1}, \z_j \sim q_k$. We follow \citet{mnih2012fast} and set $n_2 = M n_1$ in practice, where $M=100$ for all experiments; $n_1$ is the batch size.
When training with \ac{lebm}, we set the weight for reconstructing $\x$ and predicting $y$ to 50 for the term $\E_{q_{\bP}} [\log p_{\bB}(\x, y|\z)]$. We set the weight for KL term between $q_{\bP}$ and $p_{\bA}$ to $0.75$. 
We apply orthogonal regularization \citep{brock2016neural} to the parameters of decoder networks, with a weight of $0.001$ to facilitate training. These hyperparameters are shared across all the experiments.

The parameters of all the networks are initialized with the default pytorch methods \citep{paszke2019pytorch}. We use the Adam optimizer \citep{kingma2014adam} with $\beta_1=0.5$ and $\beta_2=0.999$ to train the generator and encoder networks; the log-ratio estimator network is trained with default $(\beta_1, \beta_2)$. The initial learning rates of the generator and encoder networks are \texttt{1e-4}, and \texttt{5e-5} for the log-ratio network. We further perform gradient clipping by setting the maximal gradient norm as 100 for these networks. We run the experiments on a A6000 GPU with the batch size of $128$. Training typically converges within 30K iterations on all the datasets.

When sampling with \ac{svgd} for optimization, the number of \ac{svgd} iterations $T$ is set to 500. For all experiments, we use RBF kernel 
$
k(\z, \z_0) = \exp\left(
-\frac{1}{2 h^2} \|\z - \z_0 \|_2^2
\right)
$ for \ac{svgd}, and set the bandwidth to be 
$
h^2 = \frac{{\rm med}^2}{2 \log(n + 1)}
$.
${\rm med}$ is the median of the pairwise distance between the current samples $\{\z_i^t\}_{i=1}^Q$; this follows the same heuristic as in \citet{liu2016stein} that 
$\sum_j k(\z_i, \z_j) \approx 
n \exp(-\frac{1}{2h^2} {\rm med}^2) = 1$.
We can balance the contribution from its own gradient and the influence from the other points
for each sample $\z_i$. Of note, the bandwidth changes adaptively across iterations with this heuristic. For stability, we use Adam \citep{kingma2014adam} with default hyperparemters instead of AdaGrad \citep{duchi2011adaptive} used in \citet{liu2016stein} to allow for adaptive step size $\epsilon_t$. The initial step size is within the range of $[0.3, 0.5]$ for different tasks.

\newpage
\clearpage
\section{Additional Experiment Details \& Supplementary Results}
\label{appx:add_exp}

\subsection{Gaussian Simulation}
\label{appx:gauss_simu}

We validate our analysis in \cref{prop:grad_finite_ver} and \cref{thm:convergence_rate} on a Gaussian location family. 
Specifically, we consider $p_{\bA}(\x) = \mathcal N(\x;\bA,I_d)$, where $\bA \in \mathbb R^d$ denotes the mean parameter and $I_d$ is the identity covariance. The target distribution is set to $q_* = \mathcal N(\bA^*,I_d)$, and the noise distribution to $q_0 = \mathcal N(0,I_d)$. Unless otherwise noted, we use $d=2$ for visualization. 

The log-ratio and its gradient admit closed forms:
\[
f_{\bA}(\x) = \bA^\top \x - \tfrac12\|\bA\|^2,
\qquad 
\nabla_{\bA} f_{\bA}(\x) = \x - \bA.
\]
Hence, the noisier NCE gradient is
\[
\nabla_{\bA} \LL_M(\bA)
= \E_{q_*}\!\left[\tfrac{M}{M+r_{\bA}(\x)}\,\left(\x - \bA\right)\right]
- \E_{p_{\bA}}\!\left[\tfrac{M}{M+r_{\bA}(\x)}\,\left(\x - \bA\right)\right],
\]
where $r_{\bA}(\x) = \tfrac{p_{\bA}(\x)}{q_0(\x)}$. As $M\to\infty$, this reduces to the exact MLE gradient 
\[
\nabla_{\bA} \LL_{\mathrm{MLE}}(\bA) = \E_{q_*}[\x]-\E_{p_{\bA}}[\x] = \bA^*-\bA.
\]

\paragraph{Setup.}
We set $\bA^* = (1.5,-0.8)$ and initialize at $\bA_0 = (-2.0,1.0)$. Gradients are approximated by Monte Carlo with $n=4000$ samples per iteration. Optimization is performed with step size $\eta=0.2$ for $T=150$ iterations. We track several metrics during training, including the distance to the optimum $\|\bA_t-\bA^*\|_2$ and mean and median of absolute gradient errors. We report results across different noise magnitudes $M \in \{1,2,10,50,100,1000\}$, with $M=1$ corresponding to vanilla NCE and large $M$ approximating MLE.

\paragraph{Further comparison with Nguyen-Wainwright-Jordan (NWJ) and simple reweighting.} We also compare our method with the NWJ objective \citep{nguyen2010estimating}:
\[
\LL_{\rm NWJ}(\bA) = \E_{q_*}[\log r_{\bA}]-\E_{q_{0}}[r_{\bA}] \implies 
\nabla_{\bA} \LL_{\rm NWJ}(\bA) = \E_{q_*}[\x - \bA]-\E_{q_{0}}[r_{\bA}(\x)\left(\x - \bA \right)],
\]
as well as a variant of our \ac{n2ce} objective, where we only use \ac{n2ce}-style weighting for the negative part of the objective and use the positive part of the vanilla \ac{n2ce} objective:
\[
\LL_{\rm neg-reweight}(\bA) =
    \E_{q_*(\x)} 
   \left[\log \frac{r_{{\bA}}(\x)}{1 + r_{{\bA}}(\x)}\right]
    +
   M\E_{q_{0}(\x)} 
   \left[\log \frac{M}{M + r_{{\bA}}(\x)}\right].
\]
Similar to our previous setting, we consider in this case 5-d gaussians. We set $\bA^* = ([-1.5 , -0.75,  0.  ,  0.75,  1.5 ])$ and initialize at $\bA_0 = -\bA^*$, \ie, on the opposite side. We use different number of Monte Carlo samples to approximate different regimes and for better visual comparison: (i) Gradients are approximated by Monte Carlo with $n=500$ samples per iteration, corresponding with low gradient estimation variance in \cref{prop:grad_finite_ver}, and (ii) with $n=2$ samples per iteration, corresponding with high estimation variance. Optimization is performed with step size $\eta=0.2$ for $T=150$ iterations. We plot again the distance to the optimum $\|\bA_t-\bA^*\|_2$ to illustrate the training trajectories. We plot a single-run visualization of the optimization trajectory shown in \cref{fig:gauss_simu_more}. Further, for more systematic and scientific comparisons, we use the trajectory-wise average squared norm $\tfrac{1}{T}\sum_{t=0}^{T-1} \| \theta_t - \theta^*\|_2^2$ as a summary of one run. We repeat 100 independent experiments and report the means and stds in \cref{tab:appx_gaussian_simu_100r_n2,tab:appx_gaussian_simu_100r_n500}. We can see that (i) with $n=2/500$, the optimal-$M$ \acp{n2ce} are indeed much better than those of NWJs, demonstrated by both lower mean values and lower stds. (ii) When $M$ is sufficiently large, \ac{n2ce} empirically approaches the \ac{nwj} objective with $T=\log r$.

\begin{figure}[t!]
    \centering
    % ---- subfigure (a) ----
    \begin{subfigure}[t]{0.48\linewidth}
        \centering
        \includegraphics[width=\linewidth]{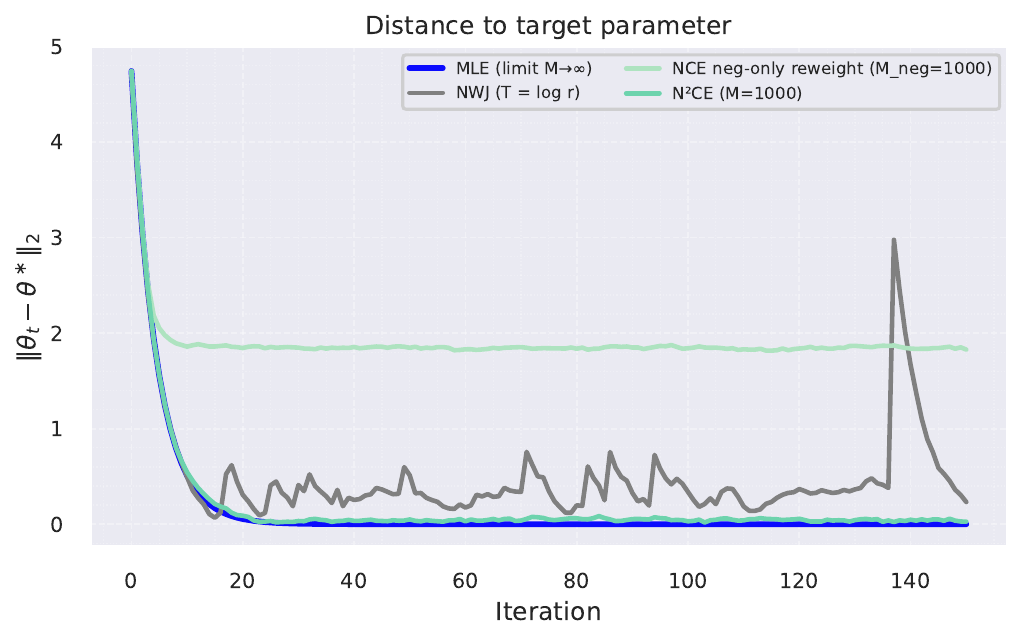}
        \caption{}
        \label{subfig:toy_low}
    \end{subfigure}
    \hfill
    % ---- subfigure (b) ----
    \begin{subfigure}[t]{0.48\linewidth}
        \centering
        \includegraphics[width=\linewidth]{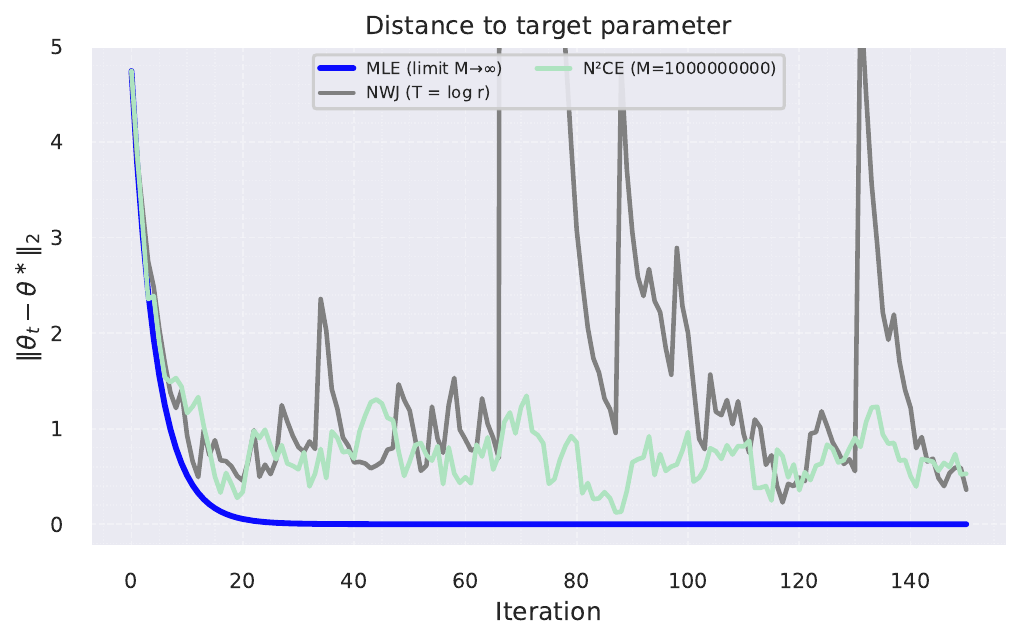}
        \caption{}
        \label{subfig:toy_high}
    \end{subfigure}   
    \caption{\textbf{\ac{n2ce} gradients can approach \acs{mle} gradients with appropriate $M$s, while NWJ and simple reweighting cannot.} We plot trajectories represented by $L_2$ norms between model and target parameters.}
    \label{fig:gauss_simu_more}
\end{figure}

\begin{table}[!hbtp]
\centering
\footnotesize
\setlength{\tabcolsep}{3pt}
\renewcommand{\arraystretch}{1.05}
\caption{\textbf{MSE sweep for $n=2$.} Optimal $M$ predicted by Prop.~3.3: $C\sqrt{2} \approx 1.414C$.}
\label{tab:appx_gaussian_simu_100r_n2}
\begin{tabular}{@{}lccccc@{}}
\toprule
$M$ & NWJ ($T=\log r$) & 1 & 1.5 & 2 & 5 \\
\midrule
avg.\ runs MSE 
& {\color{tgray} $54.548 \pm 237.528$}
& $0.904 \pm 0.084$
& $\mathbf{0.884 \pm 0.083}$
& $0.888 \pm 0.084$
& $0.983 \pm 0.107$ \\
\midrule
$M$ & 10 & 100 & 1000 & $10^9$ & \\
\midrule
avg.\ runs MSE 
& $1.139 \pm 0.168$
& $3.158 \pm 2.683$
& $17.564 \pm 42.751$
& {\color{tgray} $61.291 \pm 277.336$}
& \\
\bottomrule
\end{tabular}
\end{table}

\begin{table}[!htbp]
\centering
\footnotesize
\setlength{\tabcolsep}{3pt}
\renewcommand{\arraystretch}{1.05}
\caption{\textbf{MSE sweep for $n=500$}. Optimal $M$ predicted by Prop.~3.3: $C\sqrt{500} \approx 22.36C$.}
\label{tab:appx_gaussian_simu_100r_n500}
\begin{tabular}{@{}lccccc@{}}
\toprule
$M$ & NWJ ($T=\log r$) & 1 & 10 & 50 & 100 \\
\midrule
avg.\ runs MSE 
& {\color{tgray} $1.359 \pm 5.454$}
& $0.678 \pm 0.004$
& $0.489 \pm 0.004$
& $0.456 \pm 0.006$
& $\mathbf{0.453 \pm 0.007}$ \\
\midrule
$M$ & 1000 & $10^4$ & $2\!\times\!10^4$ & $10^9$ & \\
\midrule
avg.\ runs MSE 
& $0.489 \pm 0.020$
& $0.641 \pm 0.324$
& $0.750 \pm 0.889$
& {\color{tgray} $1.909 \pm 11.201$}
& \\
\bottomrule
\end{tabular}
\end{table}

\subsection{Image Modeling}
\label{appx:lebm_img_model} 
For image modeling in \cref{sec:exp_learn_lebm}, we include the following datasets to study our method: SVHN (32 × 32 × 3), CIFAR-10 (32 × 32 × 3), CelebA (64 × 64 × 3) and CeleAMask-HQ (256 x 256 x 3). Following \citet{pang2020learning}, we use the full training set of SVHN (73,257) and CIFAR-10 (50,000), and take 40,000 samples of CelebA as the training data. We take 29,500 samples from the CelebAMask-HQ dataset as the training data. The images are scaled to $[-1, 1]$ and are randomly horizontally flipped with a prob. of .5 for training.

\begin{table}[!h]
\caption{\textbf{AUPRC($\uparrow$) scores for unsupervised anomaly detection on MNIST}. Baseline numbers are taken from \citet{yoon2023energy,yu2024learning}. Results of our model are averaged over 10 trials.}
\vspace{0.1in}
\label{table:auprc_var}
\centering
\resizebox{0.90\linewidth}{!}{
\begin{tabular}{lccccc}
\toprule
Heldout Digit & 1 & 4 & 5 & 7 & 9 \\
\midrule
AE
&$0.062 \pm 0.00$ 
&$0.204 \pm 0.00$ 
&$0.259 \pm 0.01$ 
&$0.125 \pm 0.00$ 
&$0.113 \pm 0.00$\\
VAE
&0.063 
&0.337 
&0.325 
&0.148 
&0.104\\
ABP
&$0.095 \pm 0.03$ 
&$0.138 \pm 0.04$
&$0.147 \pm 0.03$ 
&$0.138 \pm 0.02$ 
&$0.102 \pm 0.03$ \\
IGEBM
&$0.101 \pm 0.02$ 
&$0.106 \pm 0.02$
&$0.205 \pm 0.11$ 
&$0.100 \pm 0.04$ 
&$0.079 \pm 0.02$ \\
MEG
&$0.281 \pm 0.04$ 
&$0.401 \pm 0.06$ 
&$0.402 \pm 0.06$ 
&$0.290 \pm 0.04$ 
&$0.342 \pm 0.03$ \\
BiGAN-$\sigma$
&$0.287 \pm 0.02$ 
&$0.443 \pm 0.03$ 
&$0.514 \pm 0.03$ 
&$0.347 \pm 0.02$ 
&$0.307 \pm 0.03$ \\
ABP-LEBM 
&$0.336 \pm 0.01$ 
&$0.630 \pm 0.02$
&$0.619 \pm 0.01$ 
&$0.463 \pm 0.01$ 
&$0.413 \pm 0.01$ \\
JVAEBM 
&$0.297 \pm 0.03$ 
&$0.723 \pm 0.04$
&$0.676 \pm 0.04$ 
&$0.490 \pm 0.04$ 
&$0.383 \pm 0.03$ \\
Adaptive CE
&$0.531 \pm 0.02$ 
&$0.729 \pm 0.02$
&$0.742 \pm 0.01$ 
&$0.620 \pm 0.02$ 
&$0.499 \pm 0.01$\\
NAE
&$0.802 \pm 0.08$ 
&$0.648 \pm 0.05$
&$0.716 \pm 0.03$ 
&$0.789 \pm 0.04$ 
&$0.441 \pm 0.07$\\
MPDR-S
&$0.764 \pm 0.05$ 
&$0.823 \pm 0.02$
&$0.741 \pm 0.04$ 
&$\textbf{0.857} \pm 0.02$ 
&$0.478 \pm 0.05$\\
MPDR-R
&$0.844 \pm 0.03$ 
&$0.711 \pm 0.03$
&$0.757 \pm 0.02$ 
&$\underline{0.850} \pm 0.01$ 
&$0.569 \pm 0.04$\\
\midrule
DAMC
&$0.684 \pm 0.02$ 
&$0.911 \pm 0.01$ 
&$0.939 \pm 0.02$ 
&$0.801 \pm 0.01$ 
&$\underline{0.705} \pm 0.01$ \\
DAMC-\ac{nce}
&$0.702 \pm 0.01$ 
&$0.829 \pm 0.02$ 
&$0.764 \pm 0.01$ 
&$0.605 \pm 0.01$ 
&$0.502 \pm 0.02$ \\
\cmidrule(l){1-1}
DAMC-\ac{n2ce} \\
\rowcolor{gray}
\texttt{M=100,K=1}
&$\underline{0.910} \pm 0.02$ 
&$\underline{0.911} \pm 0.01$ 
&$\underline{0.935} \pm 0.01$ 
&$0.779 \pm 0.01$ 
&$0.699 \pm 0.02$ \\
\rowcolor{gray}
\texttt{M=100,K=3}
&$\textbf{0.959} \pm 0.01$
&$\textbf{0.935} \pm 0.01$
&$\textbf{0.959} \pm 0.02$
&$0.845 \pm 0.01$  
&$\textbf{0.854} \pm 0.01$  \\
\bottomrule
\end{tabular}
}
\end{table}

\subsection{Anomaly Detection}
\label{appx:anomaly_det}
\paragraph{Experiment settings}
For anomaly detection in \cref{sec:exp_learn_lebm}, we consider the MNIST (28 x 28 x 1) dataset; we follow the experimental settings in \cite{zenati2018efficient,yu2024learning} and use 80\% of the in-domain data to train the model. We base our method upon the DAMC \citep{yu2024learning} model for posterior inference, where we use the \ac{lebm} learned by \ac{n2ce} as a plug-in replacement of the
original prior model. With properly learned models, the posterior $p_{{\bT},{\bP}}(\z|\x)$ could form a discriminative latent space that has separated probability densities for normal and anomalous data. Given the testing sample $\x$, we use log of unnormalized joint density $p_{{\bT},{\bP}}(\z | \x) \propto p_{{\bT},{\bP}}(\x, \z) \approx p_{\bB}(\x | \z) p_{\bA}(\z)|_{\z \sim q_{\bP}(\z|\x)}$ as our decision function; we draw samples from $q_{\bP}(\z|\x)$ and compare the corresponding reconstruction errors and energy scores. 
A higher value of log joint density indicates a higher probability of the test sample being a normal sample.

\paragraph{Baselines}
We benchmark our method against three groups of strong counterparts: i) data-space energy-based models, including IGEBM \citep{du2019implicit}, MEG \citep{kumar2019maximum}, JVAEBM \citep{han2020joint} and MPDR-S/R \citep{yoon2023energy}; ii) latent variable models, including AE, VAE \citep{kingma2013auto}, ABP \citep{han2017alternating}, ABP-LEBM \citep{pang2020learning}, Adaptive CE \citep{xiao2022adaptive}, NAE \citep{yoon2021autoencoding} and DAMC \citep{yu2024learning}, and iii) GAN-Based method like BiGAN-$\sigma$ \citep{zenati2018efficient}.

\subsection{Reward and Critic Learning for Diffusion Distillation}
\label{appx:reward_n_critic_abl}
\begin{table}[!htbp]
\caption{\textbf{Ablative FID scores of $M$ on CIFAR-10 under \cite{yoon2024maximum} setting.}}
\label{tab:abl_M_dxmi}
\begin{center}
\begin{tabular}{ccCc}
\toprule
\rowcolor{white}
$M=1$ & $M=50$ & {\color{tgray}$M=100$} & $M=200$\\
\midrule
3.93
&3.05
&\textbf{2.99}
&3.04 \\
\bottomrule
\end{tabular}
\end{center}
\end{table}

\begin{table}[!htbp]
\caption{\textbf{Ablative uncond. FID scores of $M$ on CIFAR-10 under \cite{zhou2024adversarial} setting.}}
\label{tab:abl_M_sida}
\begin{center}
\begin{tabular}{ccCc}
\toprule
\rowcolor{white}
$M=1$ & $M=10$ & {\color{tgray}$M=50$} & $M=100$\\
\midrule
1.53
&1.51
&\textbf{1.45}
&1.50 \\
\bottomrule
\end{tabular}
\end{center}
\end{table}

We further provide ablation of $M$ on CIFAR-10 in \cref{tab:abl_M_dxmi,tab:abl_M_sida}. We can see that greater $M$, \eg, $M=50$ or $M=100$ is significantly better than smaller $M$, especially $M=1$. However, larger $M$ may also incur larger approximation variance based on our analysis in \cref{prop:grad_finite_ver}. We therefore choose $M=100$ and $M=50$ respectively after the sweep for these experiments. These results further confirm our assumption about the noise magnitude.

\begin{figure}[t!]
    \centering
    % ---- subfigure (a) ----
    \begin{subfigure}[t]{0.48\linewidth}
        \centering
        \includegraphics[width=\linewidth]{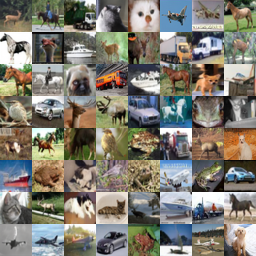}
        \caption{}
        \label{subfig:cifar}
    \end{subfigure}
    \hfill
    % ---- subfigure (b) ----
    \begin{subfigure}[t]{0.48\linewidth}
        \centering
        \includegraphics[width=\linewidth]{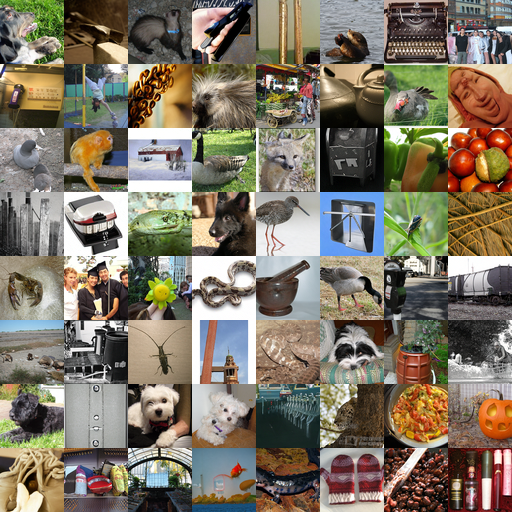}
        \caption{}
        \label{subfig:imgnet}
    \end{subfigure}
    
    \caption{\textbf{Uncurated 1-step generation results from our distilled EDM models.} 
    We present 1-step generation results from our distilled models on CIFAR-10 (left) and ImageNet64x64 (right) datasets.}
    \label{fig:1_step_gen}
\end{figure}

\newpage
\subsection{Offline Black-Box Optimization}
\label{appx:bbo_general}
\subsubsection{Problem Statement}
Let $h: \cX \rightarrow \R$ be a scalar function; the domain $\cX$ is an arbitrary subset of $\R^d$. In \acf{bbo}, $h$ is an unknown black-box function; we are interested in finding $\x^*$ that maximizes $h$:
\begin{equation}
\label{equ:bbo_def}
    \x^* \in \argmax_{\x \in \cX} h(\x).
\end{equation}
Typically, $h$ is expensive to evaluate. In offline \ac{bbo}, we assume no direct access to $h$ during training, and are thus \textbf{not allowed} to actively query the black-box function during optimization, unlike in online \ac{bbo} where most approaches would employ iterative online solving schemes \citep{kong2023molecule,kong2023moleculeLPT,shahriari2015taking,snoek2012practical}. Specifically, in the \textit{offline} \ac{bbo} setting \citep{trabucco2022design}, one has only access to a pre-collected dataset of observations, $\cD  = \{(\x_i, y_i)\}_{i=1}^n$, where $h(\x_i) = y_i$. During evaluation, one may query the function $h$ for a small budget of $Q$ queries to output candidates with the best function value obtained.

\subsubsection{Learning Latent Variable Models for Offline Black-Box Optimization}

\begin{figure}[h!]
    \centering
    \begin{tikzpicture}
        %%%%% nodes
        \node [nbase, fill=gray] (x) at (-3.2,1.0) {$\x$};
        \node [nbase, fill=gray] (y) at (-3.2,-1.0) {$y$};
        \node [nbase] (zK1) at (-0.75,0) {$\z_{K+1}$};
        \node [nbase] (zK) at (0.5,0) 
        {$\z_{K}$};
        \node [nbase] (z1) at (2,0) 
        {$\z_{1}$};
        \node [nbase] (z0) at (3.25,0) 
        {$\z_{0}$};
    
        %%%%% text
        \node (encode) at (-2.5, 0.0)
            {$q_{\bP}(\z|\x, y)$};
        \node (decode_x) at (-2.2,  1.5)
            {\color{red}$p_{{\bB}, \x}(\x|\z)$};
        \node (decode_y) at (-2.2, -1.5)
            {\color{red}$p_{{\bB}, y}(y|\z)$};

        \node[font=\fontsize{12}{6}\selectfont] (dots_l) at (1.25, 0)
        {\color{blue}$\mathbf{\cdots}$};
        \node[font=\fontsize{12}{6}\selectfont] (dots_u) at (1.25, 0.75)
        {\color{blue}$\mathbf{\cdots}$};
        
        \node (qK1qK) at (-0.1, 0.75)
        {\color{blue}$r_{{\bA}_K} \approx \frac{q_{K+1}}{q_K}$};
        \node (q1q0) at (2.55, 0.75)
        {\color{blue}$r_{{\bA}_0} \approx \frac{q_{1}}{q_0}$};
        
        %%%%% arrow
        \path [draw, ->, dashed] (x) edge (zK1);
        \path [draw, ->, color=red] 
            (zK1) edge [bend right] node [right] {} (x);
        \path [draw, ->, dashed] (y) edge (zK1);
        \path [draw, ->, color=red] 
            (zK1) edge [bend left] node [right] {} (y);
        \path [draw, <->, color=blue] (zK1) edge (zK);
        \path [draw, <->, color=blue] (z1) edge (z0);

        %%%%% bbox    
        \tikzset{plate caption/.append style={below right=0pt and -30pt of #1.south}}
        \plate [color=pink] {part1} {(zK1)(zK)(z1)(z0)(qK1qK)(q1q0)}
        {{\color{pink}$\LL_{M\to\infty}({\color{blue}{\bA}_k}),~k=0,...,K$}};
        \tikzset{plate caption/.append style={below right=2pt and -30pt of #1.south}}
        \plate [color=gray] {part2} {(part1)}
        {${\color{red}p_{\bA}(\z)} = q_0(\z)\prod_{k=0}^K 
        {\color{blue}r_{{\bA}_{k}}(\z)}$};
    \end{tikzpicture}
    \caption{\textbf{Graphical illustration of our \ac{bbo} framework.} We construct an energy-based latent space model {\color{red}$p_{\bA}$} for offline \ac{bbo} via learning a series of ratio estimators {\color{blue}$\{r_{{\bA}_k}\}_{k=0}^K$} with the \acs{n2ce} objective {\color{pink}$\LL_{M\to\infty}$} to optimize the ELBO \textit{without} \ac{mcmc}. After training, we employ stochastic samplers like \ac{ld} or \acs{svgd} to perform \ac{bbo} by sampling from the implicit inverse model {\color{red} $p_{\bT}(\x | y) \propto \E_{p_{\bT}(\z | y)} [p_{{\bB}, \x}(\x | \z)]$}, where {\color{red} $p_{\bT}(\z | y) \propto p_{{\bB}, y}(y | \z) p_{\bA}(\z)$} given $y$. Best viewed in color.}
    \label{fig:pgm_opt}
\end{figure}

We consider again variational learning of the latent variable model. Specifically, the ELBO in this setting would be
\begin{equation}
    \begin{aligned}
    \ELBO_{{\bT}, \bP} 
         &= \log p_{\bT}(\x, y) - \KL( q_{\bP}(\z|\x, y) \| p_{\bT}(\z|\x, y) ) \\
         &= \E_{q_{\bP}(\z|\x, y)} [
            \log p_{\bB}(\x, y|\z)] 
          - \KL( q_{\bP}(\z|\x, y) \| p_{\bA}(\z) ),
    \end{aligned}
    \label{equ:elbo_opt}
\end{equation}
where in $p_{\bT}(\x, y)$, $p_{{\bB}, \x}(\x|\z) = \N(g_{{\bB}, \x}(\z), \sigma_x^2 \I)$ for continuous $\x$ and multinomial for discrete $\x$, and $p_{{\bB}, y}(y|\z) = \N(g_{{\bB}, y}(\z), \sigma_y^2 \I)$. $g_{{\bB}, \x}$ and $g_{{\bB}, y}$ is the generator network for $\x$ and prediction network for $y$, respectively. $\sigma_x^2, \sigma_y^2$ take pre-specified values. $q_{\bP}(\z|\x, y) = \N({\bM}_{\bP}, {\bS}_{\bP}\I)$ is the amortized posterior. See \cref{fig:pgm_opt} for the graphical illustration.

\paragraph{Optimization as gradient-based sampling}
Once the latent variable model with $p_{\bA}$ is properly trained with \ac{n2ce}, we can sample from the parameterized inverse model $p_{\bT}(\x | y) \propto \E_{p_{{\bT}}(\z | y)} [p_{{\bB}, \x}(\x | \z)]$ to solve the \ac{bbo} problem. Specifically, we \textbf{i)} sample $\z$ from $p_{\bT}(\z | y) \propto p_{{\bB}, y}(y | \z) p_{\bA}(\z)$ with \ac{ld} or \ac{svgd}, where $y$ is set to the offline dataset maximum $y_{\rm max}$ as in \citep{krishnamoorthy2023diffusion}; and \textbf{ii)} feed them into the generator network $g_{{\bB}, \x}$ to obtain the proposed input designs $\x$.  
Intuitively, $p_{\bT}(\z | y=y_{\rm max})$ exploits the offline high value input design modes, while sampling from $p_{\bT}$ and mapping from $z$ to $x$ with $p_{{\bB}, \x}(\x | \z)$ achieves expanded latent exploration guided by the latent space model $p_{\bA}$. We provide an interpretation of latent exploration in \cref{appx:interp_opt}. 

\paragraph{\texorpdfstring{\acf{svgd}}{}}
A well-established counterpart of \ac{mcmc} for approximating the target distribution is \ac{svgd}, first proposed in the seminal work of \citet{liu2016stein}. Specifically, it tackles the sampling problem with a set of particles for approximation, on which a form of (functional) gradient descent is performed to minimize the \ac{kld} and drive the particles to fit the true target distribution. Given an unnormalized log-density function $l(\z)$, and an initial set of samples $\{\z_i^0 \}_{i=1}^n$, we can iteratively update these samples as follows to approximate the distribution specified by $l$:
\begin{equation}
\label{equ:svgd}
\begin{aligned}
\z_i^{t+1} &= \z_i^t + \epsilon_t \hat{\phi}^*(\z_i^t),~t=0,1,...,T-1~{\rm where} \\
\hat{\phi}^*(\z) &= \frac{1}{n}\sum_{j=1}^n 
\left[
k(\z_j^t, \z) \nabla_{\z_j^t} l(\z_j^t) + 
\nabla_{\z_j^t} k(\z_j^t, \z)
\right],
\end{aligned}
\end{equation}
where $k(\cdot, \cdot)$ is a positive definite kernel (\eg, RBF kernel) and $\epsilon_t$ the step size at the $t$-th iteration. We formulate the optimization process as sampling from an unnormalized log-density and utilize the mode-exploration ability of \ac{svgd} to propose high-quality candidates for the \ac{bbo} problem. 

\subsubsection{Interpretation of Optimization as Gradient-Based Sampling}
\label{appx:interp_opt}
In order to extrapolate as far as possible, while still staying on the actual manifold of high-value observations, we need to measure the validity of the generated samples $\x$ as in \citet{trabucco2021conservative}. We identify the value of $y$ from $p_{{\bB},y}(y | \z)$ (forward model) serves as one good indicator as in \cite{kumar2020model}. 
The generated samples $\x$, associated with function values similar to existing $y$ values in the offline dataset, are likely in-distribution; those where the $p_{{\bB},y}(y | \z)$ predicts a very different score (often erroneously large, see \cref{fig:branin_y} when $y > 1$) can be too far outside the training distribution \citep{kumar2020model,mashkaria2023generative,krishnamoorthy2023diffusion}. 
In addition, ideally, after training we have a well-learned and informative prior model $p_{\bA}$ that i) captures the multiple possible modes of the one-to-many inverse mapping $p_{\bT}(\x | y)$, and ii) faithfully assigns the function values to observations through the joint latent space (\cref{fig:M_viz,fig:branin_y}). Taking together, we can optimize over $\z$ in the latent space to find the best, most trustworthy $\x$ subject to the following constraints: i) $\z$ has a high likelihood under the prior and ii) the predicted function value associated with $\z$ is not too different from the value $y_{\rm max}$. This can be formulated as the following optimization problem:
\begin{equation}
\label{equ:copt_form}
\argmax_{y(\z)} y = g_{\bB}(\z)~s.t.~p_{\bB}(y_{\rm max} | \z) > \epsilon_1,~p_{\bA}(\z) > \epsilon_2.
\end{equation}
The above constrained optimization problem is conceptually similar to the formulation mentioned in \citet{kumar2020model} Sec. 3.2. 
The $p_{\bB}(y_{\rm max} | \z) > \epsilon_1$ constraint enforces that $\z$ stays on the manifold of high-value observations as we discussed; the $p_{\bA}(\z) > \epsilon_2$ constraint encourages extrapolating beyond the best score in the offline dataset $\cD$ by providing meaningful gradient during the optimization. 

In our set-up, the unnormalized log-density of $p_{{\bB}, y}(y | \z) p_{\bA}(\z)$ given the best offline function value $y_{\rm max}$ can be written as
$
- \lambda_1 \|y_{\rm max} - g_{{\bB}, y}(\z)\|_2^2 + 
\lambda_2 \left(
\sum_{k=0}^K \Tilde{f}_{{\bA}_k}(\z) - \frac{1}{2}\|\z\|_2^2
\right),
$
where $\lambda_1$ is the variance of $p_{{\bB}, y}(y|\z) = \N(g_{{\bB}, y}(\z), \sigma_y^2 \I)$ and $\lambda_2$ re-weights the prior term.
This is equivalent to optimizing the Lagrangian of \cref{equ:copt_form}, with the density constraints modified to be log-densities. In practice, we employ 
\acs{ld} or \acs{svgd} to fully utilize the informative latent space during optimization. Ideally, the dual variables, $\lambda_1$ and $\lambda_2$ are supposed to be optimized together with $y$ and $\z$, however, we find it convenient to choose them to be fixed as a constant throughout, at $\lambda_1 = 20.0$ and $\lambda_2 = 1.0$, analogous to penalty methods in constrained optimization. The same values of $\lambda_1$ and $\lambda_2$ are used for all domains in our experiments.

\subsection{2D Branin Function}
\label{appx:branin}
Branin is a well-known function for benchmarking optimization methods. We consider the negative of the standard 2D Branin function in the range $x_1 \in [-5, 10], x_2 \in [0, 15]$: 
\begin{equation}
    f_{br}(x_1, x_2) = -a(x_2 - bx_1^2 + cx_1 - r)^2 - s(1 - t)\cos{x_1} - s,
\end{equation}
where $a = 1$, $b = \frac{5.1}{4\pi^2}$, $c = \frac{5}{\pi}$, $r = 6$, $s = 10$, and $t = \frac{1}{8\pi}$. 
In this square region, $f_{br}$ has three global maximas, $(-\pi, 12.275)$, $(\pi, 2.275)$, and $(9.425, 2.475)$; with the maximum value of $-0.398$ (\cref{fig:branin}). 
The \ac{ga} baseline uses the offline dataset to train a forward prediction model parameterized by a 2-layer MLP mapping $\x$ to $y$ and then performs gradient ascent on $\x$ to infer its optima. DDOM and BONET follow their default implementation. For the typical forward method, \ac{ga} in the data space, the problem of assigning an input $\x$ with erroneously large value is clear even in this 2D toy example; the proposed points often violates the square domain constraints. This can be partly seen in the largest std. in \cref{tab:10p_removed_branin}.

\subsection{Design-Bench}
\label{appx:design_bench}
\subsubsection{Task Overview}
\label{appx:db_overview}
The goal of \textbf{TF-Bind-8} and \textbf{TF-Bind-10} is to optimize for a DNA sequence that has a maximum affinity to bind with a particular transcription factor. The sequences are of length 8 for \textbf{TF-Bind-8} and 10 for \textbf{TF-Bind-10}, where each element in the sequence is one of 4 bases.  \textbf{ChEMBL} optimizes drugs for specific chemical properties. In \textbf{D’Kitty} and \textbf{Ant Morphology}, one need to optimize for the morphology of robots. In \textbf{Superconductor}, the goal is to find a material with a high critical temperature. These tasks contain neither personally identifiable nor offensive information. We present detailed information on the tasks we evaluate on in Design-Bench \citep{trabucco2022design} as below in \cref{tab:appx_db_stats}. \textbf{SIZE} denotes the training dataset size, and {\bf NUM. TOKS} denotes the number of tokens for discrete tasks.
\begin{table}[!ht]
\centering
\caption{\textbf{Design-Bench dataset statistics.}}
\label{tab:appx_db_stats}
\vskip 0.1in
\begin{tabular}{ccccc}
\toprule
{\bf TASK} & {\bf SIZE} & {\bf INPUT DIM.} & {\bf NUM. TOKS} & {\bf TASK MAX} \\
\hline
TF-Bind-8 & 32898 & 8 & 4 & 1.0 \\ 
TF-Bind-10 & 10000 & 10 & 4 & 2.128 \\ 
ChEMBL & 441 & 31 & 185 & 443000.0 \\
\hline
D'Kitty & 10004 & 56 & {-} & 340.0 \\ 
Ant & 10004 & 60 & {-} & 590.0 \\ 
Superconductor & 17014 & 86 & {-} & 185.0  \\ 
\bottomrule
\end{tabular}
\end{table}

\subsubsection{HopperController \& NAS}
\label{appx:no_hopper}
As in \citet{krishnamoorthy2023diffusion} and \citet{mashkaria2023generative}, we exclude the HopperController task in our results due to significant inconsistencies between the offline dataset values and the values obtained when running the oracle (see \cref{fig:appx_hopper1,fig:appx_hopper2}). This is a known bug in Design-bench (see details in the \href{https://github.com/brandontrabucco/design-bench/issues/8}{link}). Therefore, we decided not to include Hopper in our evaluation and analysis. Following \citet{krishnamoorthy2023diffusion}, we exclude NAS due to excessive computation cost required beyond our budget for evaluating across multiple seeds.

\begin{figure}[!ht]
\centering
\begin{minipage}[t]{.49\textwidth}
    \centering
    \includegraphics[scale=0.45]{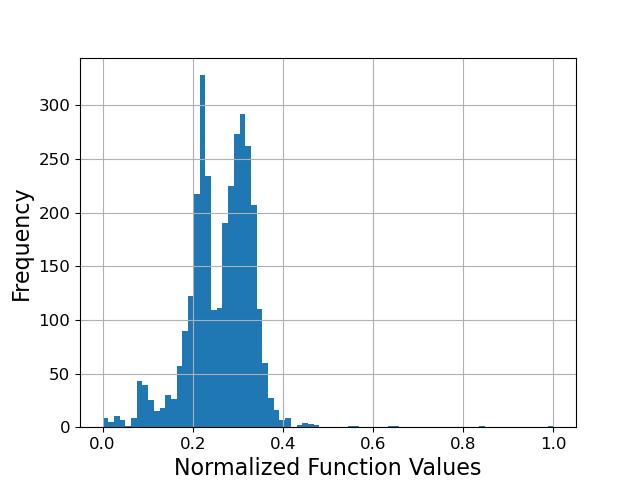}
    \captionof{figure}{Histogram of normalized function values in the Hopper dataset. The distribution is highly skewed towards low function values (plot from \citet{mashkaria2023generative}).}
    \label{fig:appx_hopper1}
\end{minipage}%
\hfill
\begin{minipage}[t]{.49\textwidth}
    \centering
    \includegraphics[scale=0.42]{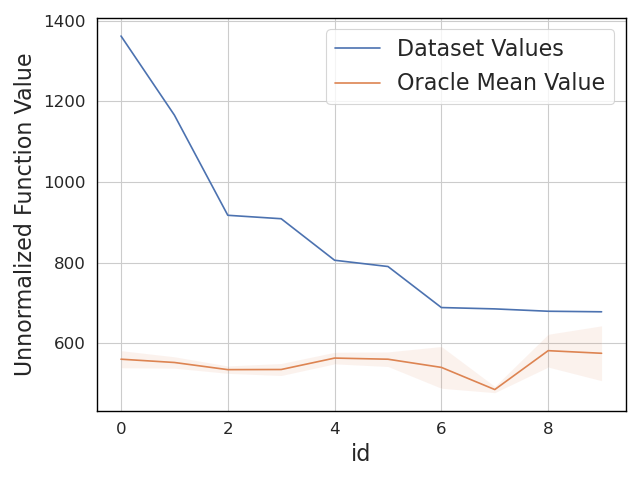}
    \captionof{figure}{Dataset values vs Oracle values for top $10$ points. We can see from mean and standard deviation over $20$ runs that oracle is noisy (plot from \citet{mashkaria2023generative}).}
    \label{fig:appx_hopper2}
\end{minipage}
\end{figure}

\subsubsection{Baselines}
\label{appx:db_baselines}
As mentioned in \cref{sec:exp_design_bench}, we compare our method with three groups of baselines: \textbf{i)} sampling via generative inverse models with different parameterizations, including CbAS \citep{brookes2019conditioning}, Auto.CbAS \citep{fannjiang2020autofocused}, MIN \citep{kumar2020model} and DDOM to validate our model design; \textbf{ii)} gradient(-like) updating from existing designs, \eg, \ac{ga}-based \citep{fu2020offline,trabucco2021conservative,chen2022bidirectional,qi2022data,yuan2024importance,chen2024parallel} and BONET and \textbf{iii)} baselines mentioned in \citep{trabucco2022design}. 
The generative inverse models learn $p_{\bT}(\x | y)$ with different parameterizations and then perform conditional sampling of the high-value designs for \ac{bbo}. CbAS \citep{brookes2019conditioning} directly models $p_{\bT}(\x | y)$, while introducing the joint distribution of $\x$ and $\z$ to facilitate importance sampling. It trains a VAE model to parameterize the distribution using identified high-quality inputs $\x$ (by thresholding); it gradually adapts the distribution to the high-value part by refining the threshold.
MIN \citep{kumar2020model} learns $p_{\bT}(\x | y) \propto \E_{p(\z | y)}[p(\x |\z, y)]$ with a conditional GAN-like model to instantiate $p(\x |\z, y)$ \citep{goodfellow2020generative,mirza2014conditional}. The method optimizes over $\z$ given the offline dataset maximum $y_{\rm max}$ and map $\z$ back to $\x$ for \ac{bbo} solutions. The optimization process over $\z$ draw approximate samples from $p(\z | y)$.
DDOM \citep{krishnamoorthy2023diffusion} consider directly parameterizing the inverse mapping $p_{\bT}(\x | y)$ with a conditional diffusion model in the input design space utilizing the expressiveness of DDPMs \citep{ho2020denoising}. 
The second group includes the gradient(-like) updating methods, and has the flavor of forward methods in general. 
The baseline methods include: (i) Gradient Ascent: operates directly on the learned surrogates to propose new input designs via simple gradient ascent; (ii) COMs \citep{trabucco2021conservative}:
regularizes the NN-parameterized surrogate by assigning lower function values to designs obtained during the gradient ascent process; (iii) ROMA \citep{yu2021roma}: incorporates prior knowledge about function smoothness into the surrogate and optimizes the design against the proxy; (iv) NEMO \citep{fu2020offline}: leverages normalized maximum likelihood to constrain the distance between the surrogate and the ground-truth;
(v) BDI \citep{chen2022bidirectional}: proposes to distill the information from the static dataset into the high-scoring design;
(vi) IOM \citep{qi2022data}: enforces the invariance between the representations of the static dataset and generated
designs to achieve a natural trade-off; (vii) ICT \citep{yuan2024importance} explores using a pseudo-labeler to generate valuable data for fine-tuning the surrogate. BONET \citep{mashkaria2023generative} trains a transformer-based model to mimic the optimization trajectories of black-box optimizers, and rolls out evaluation trajectories from the trained model for optimization.

In addition to these recent works, we also compare with several baseline methods described in \citet{trabucco2022design}:
1) GP-qEI \citep{wilson2017reparameterization}: instantiates \ac{bo} using a Gaussian Process as the uncertainty quantifier and the quasi-Expected Improvement (q-EI) acquisition function, 2)
CMA-ES \citep{hansen2006cma}: an evolutionary algorithm that maintains a belief distribution over the optimal candidates; it gradually refines the distribution by adapting the covariance
matrix towards optimal candidates using feedback from the learned surrogate, and 3) REINFORCE \citep{williams1992simple}: first learns a proxy function and then optimizes the input space distribution by leveraging the proxy and the policy-gradient estimator. Since we are in the offline setting, for active methods like \ac{bo}, we follow the procedure of \cite{trabucco2022design} and optimize on a surrogate model trained on the offline dataset.

For $Q=128$ results in \cref{tab:main_table_128}, we additionally include ExPT \citep{nguyen2024expt} and BOOTGEN \citep{kim2024bootstrapped} for rough reference. Of note, these two methods have slightly different set-ups, and are therefore not directly comparable to other baselines and our method. ExPT \citep{nguyen2024expt} focuses extensively on few-shot learning scenarios with models pre-trained on larger datasets. BOOTGEN \citep{kim2024bootstrapped} focuses specifcially on optimizing biological sequences. 

\subsubsection{Additional Results \& Analysis}
\label{appx:db_results}
\paragraph{Proof-of-concept results for data efficiency in \ac{bbo}} First, we uniformly sample $N = 5000$ and $N=50$ points from the branin function domain for training. We observe that our method i) performs consistently better than BONET \citep{mashkaria2023generative} and gradient ascent, and ii) demonstrates less performance drop compared with BONET when dataset size is significantly reduced. Further, we use the D'Kitty dataset and withhold an $x\%$ size subsection of data, \ie, $0\%$ represents the full dataset. We have similar observation on these datasets summarized in \cref{tab:appx_reduced_branin,tab:appx_reduced_dkitty}, where our method shows much less performance drop compared with the transformer-based method BONET.

\begin{table}[!htbp]
\begin{minipage}{0.47\textwidth}
\centering
    \caption{\textbf{Results on (reduced size) Branin function dataset.}}
    \vskip 0.1in
    \resizebox{0.99\textwidth}{!}{
    \begin{tabular}{clllL}
    \toprule
    \rowcolor{white}
     \multicolumn{1}{c}{\textbf{N}}  
    &\multicolumn{1}{c}{$\cD$ (best)} 
    &\multicolumn{1}{c}{Grad. Ascent}
    &\multicolumn{1}{c}{BONET}
    &\multicolumn{1}{c}{Ours} \\
    \midrule
5000&$-6.199$&$-3.95 \pm 4.26$&$-1.79 \pm 0.84$&$\mathbf{-0.41 \pm 0.13}$\\
50&$-6.231$&$-4.64 \pm 3.17$&$-2.13 \pm 0.15$&$\mathbf{-0.43 \pm 0.12}$\\
    \bottomrule
    \end{tabular}}
\label{tab:appx_reduced_branin}
\end{minipage}\quad\quad
\begin{minipage}{0.47\textwidth}
    \centering
\caption{\textbf{Results on (reduced size) D'Kitty dataset.} $x\%$ indicates the proportion of training data withheld, \ie, $0\%$ represents the full dataset.}
    \vskip 0.1in
    \resizebox{0.9\textwidth}{!}{
    \begin{tabular}{cllL}
    \toprule
    \rowcolor{white}
     \multicolumn{1}{c}{}  
    &\multicolumn{1}{c}{$0\%$} 
    &\multicolumn{1}{c}{$90\%$}
    &\multicolumn{1}{c}{$99\%$}\\
    \midrule
BONET&$285.11 \pm 15.13$&$274.11 \pm 7.57$&$\mathbf{241.17 \pm 18.07}$\\
Ours&$293.33 \pm 7.45$&$289.67 \pm 7.86$&$\mathbf{283.65 \pm 9.23}$\\
    \bottomrule
\end{tabular}}
\label{tab:appx_reduced_dkitty}
\end{minipage}
\end{table}

\subsubsection{Additional results for $Q=256$}

\paragraph{Unnormalized results} For reference, we present the unnormalized results of \cref{tab:main_table_256} in \cref{tab:appx_un_table}.
\begin{table*}[!htbp]
\centering
\caption{\textbf{Results on design-bench.} \textbf{Unnormalized} results with a budget $Q = 256$.}
\vskip 0.1in
\small
\resizebox{0.99\textwidth}{!}{
\begin{tabular}{lllllll}
\toprule
\multicolumn{1}{c}{\textbf{BASELINE}}  &\multicolumn{1}{c}{\textbf{TFBIND8}} 
&\multicolumn{1}{c}{\textbf{TFBIND10}}
&\multicolumn{1}{c}{\textbf{CHEMBL}} 
&\multicolumn{1}{c}{\textbf{SUPERCON.}}
&\multicolumn{1}{c}{\textbf{ANT}}  
&\multicolumn{1}{c}{\textbf{D'KITTY}}\\
\midrule
$\cD$ (best) & $0.439$ & $0.00532$ & $383700.000$ & $74.0$ & $165.326$ & $199.363$\\
\midrule
GP-qEI & $0.824 \pm 0.086$ & $0.675 \pm 0.043$ & $ 387950.000 \pm 0.000$ & $92.686 \pm 3.944$ & $480.049 \pm 0.000$ & $213.816 \pm 0.000$ \\ 
CMA-ES & $0.933 \pm 0.035$ & $0.848 \pm 0.136$ & $ 388400.000 \pm 400.000$ & $90.821 \pm 0.661$ & $\mathbf{1016.409 \pm 906.407}$ & $4.700 \pm 2.230$\\ 
REINFORCE & $0.959 \pm 0.013$  & $0.692 \pm 0.113$ & $388400.000 \pm 2100.000$ & $89.027 \pm 3.093$ & $-131.907 \pm 41.003$ & $-301.866 \pm 246.284$\\ 
\midrule
Gradient Ascent & $\underline{0.981 \pm 0.015}$ & $0.770 \pm 0.154$ & $390050.000 \pm 2000.000$ & $93.252 \pm 0.886$ & $-54.955 \pm 33.482$ & $226.491 \pm 21.120$\\ 
COMs & $0.964 \pm 0.020$ & $0.750 \pm 0.078$ & $390200.000 \pm 500.000$ & $78.178 \pm 6.179$ & $540.603 \pm 20.205$ & $277.888 \pm 7.799$\\
BONET & $0.975 \pm 0.004$ & $0.855 \pm 0.139$ & $\underline{391000.000 \pm 1900.000}$ & $80.845 \pm 4.087$ & $567.042 \pm 11.653$ & $\underline{285.110 \pm 15.130}$\\
\midrule
CbAS & $0.958 \pm 0.018$ & $0.761 \pm 0.067$ & $ 389000.000 \pm 500.000$ & $83.178 \pm 15.372$ & $ 468.711 \pm 14.593$ & $213.917 \pm 19.863$\\ 
MINs & $0.938 \pm 0.047$ & $0.770 \pm 0.177$ & $390950.000 \pm 200.000$ & $89.469 \pm 3.227$ & $ 533.636 \pm 17.938$ & $272.675 \pm 11.069$\\
DDOM & $0.971 \pm 0.005$ & $\underline{0.885 \pm 0.367}$ & $387950.000 \pm 1050.000$ & $\underline{ 103.600 \pm 8.139}$ & $548.227 \pm 11.725$ & $250.529 \pm 10.992$\\
\rowcolor{gray}
Ours & $\mathbf{0.990 \pm 0.003}$ & $\mathbf{1.344 \pm 0.340}$ & $\mathbf{392149.225
\pm 3821.853959
}$ & $\mathbf{104.848 \pm 3.113}$ & $\underline{572.490 \pm 7.277}$ & $\mathbf{293.327 \pm 7.452}$\\
\bottomrule
\end{tabular}}
\label{tab:appx_un_table}
\end{table*}

\paragraph{$50$-th percentile results} We present the normalized results with $Q=256$ at $50$-th percentile to further demonstrate the effectiveness of our approach in \cref{tab:appx_q256_50_table}.
\begin{table*}[!htbp]
\centering
\caption{\textbf{Normalized results on design-bench.} $50$-th percentile results with $Q = 256$. Baseline numbers from \cite{mashkaria2023generative}. DDOM \citep{krishnamoorthy2023diffusion} does not report $50$-th percentile results; we are unable to find public model weights corresponding with the results in \cref{tab:main_table_256} for DDOM and therefore omit DDOM in this table.}
\vskip 0.1in
\small
\resizebox{\textwidth}{!}{
\begin{tabular}{llllllllc}
\toprule
\multicolumn{1}{c}{\textbf{BASELINE}}  &\multicolumn{1}{c}{\textbf{TFBIND8}} 
&\multicolumn{1}{c}{\textbf{TFBIND10}}
&\multicolumn{1}{c}{\textbf{CHEMBL}} 
&\multicolumn{1}{c}{\textbf{SUPERCON.}}
&\multicolumn{1}{c}{\textbf{ANT}}  
&\multicolumn{1}{c}{\textbf{D'KITTY}}
&\multicolumn{1}{c}{\textbf{MEAN SCORE}$^\uparrow$} 
&\multicolumn{1}{c}{\textbf{MNR}$^\downarrow$}\\
\midrule
GP-qEI &$ 0.443 \pm 0.004 $&$ 0.494 \pm 0.002 $&$ 0.299 \pm 0.002 $&$ 0.272 \pm 0.006 $&$ 0.754 \pm 0.004 $&$ 0.633 \pm 0.000 $&$ 0.491 \pm 0.016 $&$ 4.8$ \\
CMA-ES &$ 0.543 \pm 0.007$&$ 0.483 \pm 0.011 $&$ 0.376 \pm 0.004 $&$ -0.051 \pm 0.004 $&$ 0.685 \pm 0.018 $&$ 0.633 \pm 0.000 $&$ 0.466 \pm 0.020 $&$5.2$ \\
REINFORCE &$0.450 \pm 0.003 $&$0.472 \pm 0.000 $&$0.470 \pm 0.017 $&$0.146 \pm 0.009 $&$0.307 \pm 0.002 $&$0.633 \pm 0.000 $&$ 0.083 \pm 0.004$&$ 5.7$ \\
\midrule
Gradient Ascent &$0.572 \pm 0.024 $&$0.470 \pm 0.004 $&$0.463 \pm 0.022 $&$0.141 \pm 0.010 $&$ 0.637 \pm 0.148 $&$ 0.633 \pm 0.000 $&$ 0.478 \pm 0.029$&$ 5.1$ \\ 
COMs &$ 0.492 \pm 0.009 $&$0.472 \pm 0.012 $&$0.365 \pm 0.026 $&$0.525 \pm 0.018 $&$0.885 \pm 0.002 $&$0.633 \pm 0.000 $&$ 0.522 \pm 0.034 $&$ 4.6$ \\ 
BONET &$0.505 \pm 0.055 $&$0.496 \pm 0.037 $&$0.369 \pm 0.015 $&$\mathbf{0.819 \pm 0.032}$&$\mathbf{0.907 \pm 0.020}$&$0.630 \pm 0.000 $&$ 0.614 \pm 0.035 $&$3.4$ \\
\midrule
CbAS &$ 0.422 \pm 0.007 $&$ 0.458 \pm 0.001 $&$ 0.111 \pm 0.009 $&$ 0.384 \pm 0.010 $&$ 0.752 \pm 0.003 $&$ 0.633 \pm 0.000 $&$ 0.436 \pm 0.008 $&$ 6.5$ \\
MINs &$0.425 \pm 0.011 $&$0.471 \pm 0.004 $&$0.330 \pm 0.011 $&$0.651 \pm 0.010 $&$0.890 \pm 0.003 $&$0.633 \pm 0.000 $&$ 0.547 \pm 0.005 $&$ 4.8$ \\ 
\rowcolor{gray}
Ours &$\mathbf{0.896 \pm 0.033} $&$ \mathbf{0.507 \pm 0.012} $&$ \mathbf{0.373 \pm 0.002}$&$0.715 \pm 0.004 $&$ \mathbf{0.907 \pm 0.008}$&$\mathbf{0.630 \pm 0.000}$&$ \mathbf{0.672 \pm 0.009}$&$\mathbf{1.2}$\\
\bottomrule
\end{tabular}}
\label{tab:appx_q256_50_table}
\end{table*}

\subsubsection{Additional results for $Q=128$}
\paragraph{$25$-th, $50$-th, $75$-th and $100$-th percentile results} We present the normalized results with $Q=128$ at $100$-th and $50$-th percentile to further demonstrate the effectiveness of our approach in \cref{tab:main_table_128,tab:appx_q128_50_table}. Additionally, we provide results at $25$-th and $75$-th percentiles for reference in \cref{tab:appx_q128_25_75_table}.

\begin{table*}[!htb]
\centering
\caption{\textbf{Normalized design-bench results with $Q = 128$ at 100-th percentile.} Baseline numbers from \cite{kim2024bootstrapped,yuan2024importance,nguyen2024expt,chen2024parallel}. \textbf{ChEMBL} dataset is excluded following baseline set-ups. * indicates that the baseline has slightly different set-ups; the numbers are for rough reference (see \cref{appx:db_baselines} for more details).}
\vskip 0.1in
\resizebox{0.9\textwidth}{!}{
\begin{tabular}{llllllllc}
\toprule
\multicolumn{1}{c}{\textbf{BASELINE}}  &\multicolumn{1}{c}{\textbf{TFBIND8}} 
&\multicolumn{1}{c}{\textbf{TFBIND10}}
&\multicolumn{1}{c}{\textbf{SUPERCON.}}
&\multicolumn{1}{c}{\textbf{ANT}}  
&\multicolumn{1}{c}{\textbf{D'KITTY}}  
&\multicolumn{1}{c}{\textbf{MEAN SCORE}$^\uparrow$} 
&\multicolumn{1}{c}{\textbf{MNR}$^\downarrow$}\\
\midrule

$\cD$ (best) & $0.439$ & $0.467$ & $0.399$ & $0.565$ & $0.884$ & {-} & {-}\\
\midrule
GP-qEI & $0.798 \pm 0.083$ & $0.652 \pm 0.038$ 
       & $0.402 \pm 0.034$ & $0.819 \pm 0.000$
       & $0.896 \pm 0.000$ & $0.713$ & $14.2$ \\
CMA-ES & $0.953 \pm 0.022$ & $0.670 \pm 0.023$
       & $0.465 \pm 0.024$ & $\mathbf{1.214 \pm 0.732}$ & $0.724 \pm 0.001$ & $0.805$ & $8.6$ \\
REINFORCE & $0.948 \pm 0.028$ & $0.663 \pm 0.034$ & $0.481 \pm 0.013$ & $0.266 \pm 0.032$ & $0.562 \pm 0.196$ & $0.584$ & $12.4$ \\
\midrule
*BOOTGEN &$\underline{0.979 \pm 0.001}$& {-}&{-}&{-}&{-}&{-}&{-} \\
*ExPT &$0.933 \pm 0.036$&$0.677 \pm 0.048$&{-}&$\underline{0.970 \pm 0.004}$&$\mathbf{0.973 \pm 0.005}$&{-}&{-} \\
\midrule
Grad. Ascent & $0.886 \pm 0.035$ & $0.647 \pm 0.021$ & $0.495 \pm 0.011$ & $0.934 \pm 0.011$ & $ 0.944 \pm 0.017$ & $0.781$ & $10.2$ \\
COMS & $0.496 \pm 0.065$&$0.622 \pm 0.003$&$0.491 \pm 0.028$&$0.856 \pm 0.040$&$0.938 \pm 0.015$&$0.681$&$13.4$ \\ 
BONET & $0.911 \pm 0.005$&$\underline{0.756 \pm 0.006}$&$0.500 \pm 0.002$& $0.927 \pm 0.002$&$0.954 \pm 0.000$& $0.810$ & $6.8$ \\
ROMA &$0.924 \pm 0.040$&$0.666 \pm 0.035$&$0.510 \pm 0.015$&$0.917 \pm 0.030$&$0.927 \pm 0.013$&$0.789$&$8.2$ \\
NEMO &$0.943 \pm 0.005$&$0.711 \pm 0.021$&$0.502 \pm 0.002$&$0.958 \pm 0.011$&$0.954 \pm 0.007$&$0.814$&$5.2$ \\
BDI &$0.870 \pm 0.000$&$ 0.605 \pm 0.000$&$0.513 \pm 0.000$&$0.906 \pm 0.000$&$0.919 \pm 0.000$&$0.763$&$11.6$ \\
IOM &$0.878 \pm 0.069$&$0.648 \pm 0.023$&$\underline{0.520 \pm 0.018}$&$0.918 \pm 0.031$&$0.945 \pm 0.012$&$0.782$&$8.6$ \\
ICT &$0.958 \pm 0.008$&$0.691 \pm 0.023$&$0.503 \pm 0.017$&$0.961 \pm 0.007$&$\underline{0.968 \pm 0.020}$&$0.816$&$3.4$ \\
Tri-mtring &$\underline{0.970 \pm 0.001}$&$0.722 \pm 0.017$&$0.514 \pm 0.018$&$0.948 \pm 0.014$&$0.966 \pm 0.010$&$\underline{0.824}$&$\underline{3}$ \\
\midrule
CbAS & $0.927 \pm 0.051$&$0.651 \pm 0.060$&$0.503 \pm 0.069$&$0.876 \pm 0.031$&$0.892 \pm 0.008$&$0.770$ & $10.8$ \\
Auto.CbAS &$0.910 \pm 0.044$&$0.630 \pm 0.045$&$0.421 \pm 0.045$&$0.882 \pm 0.045$&$0.906 \pm 0.006$&$0.745$&$13$ \\
MINs & $0.905 \pm 0.052$&$0.616 \pm 0.021$ & $0.499 \pm 0.017$&$0.445 \pm 0.080$&$0.892 \pm 0.011$&$0.671$ & $13.6$ \\
DDOM & $0.957 \pm 0.006$&$0.657 \pm 0.006$&$0.495 \pm 0.012$&$0.940 \pm 0.004$&$0.935 \pm 0.001$&$0.797$ & $8$ \\
\rowcolor{gray}
Ours &$\mathbf{0.983 \pm 0.011}$&$\mathbf{0.771 \pm 0.096}$&$\mathbf{0.559 \pm 0.024}$&$0.954 \pm 0.008$&$0.954 \pm 0.004$&$\mathbf{0.844}$&$\mathbf{2}$ \\
\bottomrule
\end{tabular}}    
\vskip 0.1in
\label{tab:main_table_128}
\end{table*}

\begin{table*}[htb!]
    \centering
    \caption{\textbf{Normalized design-bench $25$-th and $75$-th results with $Q = 128$}.}
    \vskip 0.1in
    \resizebox{0.7\textwidth}{!}{
    \begin{tabular}{clllll}
    \toprule
     \multicolumn{1}{c}{\textbf{Percentile}}  
    &\multicolumn{1}{c}{\textbf{TFBIND8}} 
    &\multicolumn{1}{c}{\textbf{TFBIND10}}
    &\multicolumn{1}{c}{\textbf{SUPERCON.}}
    &\multicolumn{1}{c}{\textbf{ANT}}  
    &\multicolumn{1}{c}{\textbf{D'KITTY}}\\
    \midrule
$25$-th&$0.635 \pm 0.020$&$0.453 \pm 0.007$&$0.351 \pm 0.019$&$0.527 \pm 0.033$&$0.878 \pm 0.002$\\
$75$-th&$0.873 \pm 0.006$&$0.945 \pm 0.012$&$0.409 \pm 0.013$&$0.832 \pm 0.012$&$0.910 \pm 0.002$\\
    \bottomrule
    \end{tabular}}
    \label{tab:appx_q128_25_75_table}
\end{table*}

\begin{table*}[ht!]
\centering
\caption{\textbf{Normalized design-bench $50$-th percentile results with $Q = 128$ at 100-th percentile.} Baseline numbers from \cite{kim2024bootstrapped,yuan2024importance,nguyen2024expt,chen2024parallel}. \textbf{ChEMBL} dataset is excluded following baseline set-ups. * indicates that the baseline has slightly different set-ups; the numbers are for rough reference (see \cref{appx:db_baselines} for more details).}
\vskip 0.1in
\resizebox{0.9\textwidth}{!}{
\begin{tabular}{llllllllc}
\toprule
\multicolumn{1}{c}{\textbf{BASELINE}}  &\multicolumn{1}{c}{\textbf{TFBIND8}} 
&\multicolumn{1}{c}{\textbf{TFBIND10}}
&\multicolumn{1}{c}{\textbf{SUPERCON.}}
&\multicolumn{1}{c}{\textbf{ANT}}  
&\multicolumn{1}{c}{\textbf{D'KITTY}}  
&\multicolumn{1}{c}{\textbf{MEAN SCORE}$^\uparrow$} 
&\multicolumn{1}{c}{\textbf{MNR}$^\downarrow$}\\
\midrule
$\cD$ (best)&$0.439$&$0.467$&$0.399$&$0.565$&$0.884$& {-} & {-}\\
\midrule
GP-qEI&$0.439 \pm 0.000$&$0.467 \pm 0.000$&$0.300 \pm 0.015$&$0.567 \pm 0.000$&$0.883 \pm 0.000$&$0.531$&$11.2$ \\
CMA-ES&$0.537 \pm 0.014$&$0.484 \pm 0.014$&$0.379 \pm 0.003$&$-0.045 \pm 0.004$&$0.684 \pm 0.016$&$0.408$&$10.2$ \\
REINFORCE&$0.462 \pm 0.021$&$0.475 \pm 0.008$&$0.463 \pm 0.016$&$0.138 \pm 0.032$&$0.356 \pm 0.131$&$0.379$&$10.8$ \\
\midrule
*BOOTGEN&$0.833 \pm 0.007$&{-}&{-}&{-}&{-}&{-}&{-}\\
*ExPT&$0.473 \pm 0.014$&$0.477 \pm 0.014$&{-}&$\mathbf{0.705 \pm 0.018}$&$\mathbf{0.902 \pm 0.006}$&{-}&{-}\\
\midrule
Grad. Ascent&$0.532 \pm 0.017$&$0.529 \pm 0.027$&$0.339 \pm 0.015$&$0.564 \pm 0.014$&$0.877 \pm 0.005$&$0.568$&$7.4$ \\
COMS&$0.439 \pm 0.000$&$0.466 \pm 0.002$&$0.316 \pm 0.022$&$0.568 \pm 0.002$&$0.883 \pm 0.002$&$0.534$&$10.6$\\
BONET&$0.505 \pm 0.004$&$0.465 \pm 0.002$&$\mathbf{0.470 \pm 0.004}$&$0.620 \pm 0.003$&$0.897 \pm 0.000$&$0.591$&$5.4$\\
ROMA&$0.555 \pm 0.020$&$0.512 \pm 0.020$&$0.372 \pm 0.019$&$0.479 \pm 0.041$&$0.853 \pm 0.007$&$0.554$&$7.6$\\
NEMO&$0.548 \pm 0.017$&$0.516 \pm 0.020$&$0.322 \pm 0.008$&$0.593 \pm 0.000$&$0.885 \pm 0.000$&$0.539$&$10$\\
BDI&$0.439 \pm 0.000$&$0.476 \pm 0.000$&$0.412 \pm 0.000$&$0.474 \pm 0.000$&$0.855 \pm 0.000$&$0.546$&$9.4$\\
IOM&$0.439 \pm 0.000$&$0.477 \pm 0.010$&$0.352 \pm 0.021$&$0.509 \pm 0.033$&$ 0.876 \pm 0.006$&$0.531$&$9.8$\\
ICT&$0.551 \pm 0.013$&$\mathbf{0.541 \pm 0.004}$&$0.399 \pm 0.012$&$0.592 \pm 0.025$&$0.874 \pm 0.005$&$0.591$&$5.2$\\
Tri-mtring&$0.609 \pm 0.021$&$0.527 \pm 0.008$&$0.355 \pm 0.003$&$0.606 \pm 0.007$&$0.886 \pm 0.001$&$0.597$&$4.2$\\
\midrule
CbAS&$0.428 \pm 0.010$&$0.463 \pm 0.007$&$0.111 \pm 0.017$&$0.384 \pm 0.016$&$0.753 \pm 0.008$&$0.428$&$15$\\
Auto.CbAS&$0.419 \pm 0.007$&$0.461 \pm 0.007$&$0.131 \pm 0.010$&$0.364 \pm 0.014$&$0.736 \pm 0.025$&$0.422$&$15.8$\\
MINs&$0.421 \pm 0.015$&$0.468 \pm 0.006$&$0.336 \pm 0.016$&$0.618 \pm 0.040$&$0.887 \pm 0.004$&$0.546$&$8.8$\\
DDOM&$0.553 \pm 0.002$&$0.488 \pm 0.001$&$0.295 \pm 0.001$&$0.590 \pm 0.003$&$0.870 \pm 0.001$&$0.559$&$8.6$\\
\rowcolor{gray}
Ours&$\mathbf{0.873 \pm 0.006}$&$0.499 \pm 0.002$&$0.370 \pm 0.004$&$\mathbf{0.704 \pm 0.013}$&$\mathbf{0.897 \pm 0.003}$&$\mathbf{0.669}$&$\mathbf{3}$ \\
\bottomrule
\end{tabular}}    
\label{tab:appx_q128_50_table}
\end{table*}

\subsubsection{Extended discussion for ablation study} 
As mentioned in \cref{sec:exp_bbo}, we provide ablative results on noise magnitude $M$, number of intermediate stages $K$ and query budget $Q$ in \cref{tab:abl_table,tab:abl_K_M,tab:abl_Q}. Of note, the $K=1, M=100$ row in \cref{tab:abl_table} means learning a \ac{lebm} w/ \cref{equ:n2ce_obj} alone, w/o ratio decomposition. We can see that the variant delivers significantly better results than \ac{mle}-\ac{lebm} w/ \ac{svgd} sampler. This indicates the effectiveness of the noise-intensified objective. $K=6, M=1$ row means learning with the original \ac{nce} objective but with ratio decomposition, w/o intensifying the noise distributions. We can see that simply combining the \ac{nce} objective with variational learning of \ac{lebm} only achieves inferior performance compared with learning \ac{lebm} with the \ac{n2ce} objective. 

We further provide detailed ablation of $M$ in \cref{tab:abl_K_M}. We can see that greater $M$, \eg, $M=100$ or $M=1000$ is significantly better than smaller $M$, \eg, $M=1$ or $M=10$. However, larger $M$ incurs larger memory consumption based on our implementation in \cref{appx:tr_details_n_arch}. It also has a pitfall of leading to larger approximation variance based on our analysis in \cref{prop:grad_finite_ver}. We therefore choose $M=100$ for all experiments. These results further confirm our assumption about the noise magnitude.
We can see from \cref{tab:abl_K_M} that larger number of intermediate stages $K$ typically brings better performance. These results demonstrate the need for ratio decomposition, aligned with our analysis in \cref{sec:error_n_reg}. In addition, since we are using the same network architecture for ablating different $K$s, the network capacity is fixed. Larger number of intermediate stages may require a larger network to perform well (circumventing competition between stages), as we are learning the ensemble of ratio estimators using the shared network. \ac{svgd} in the data space typically requires a relatively large population of initial samples to perform well. In our set-up, we set the query budget $Q$ to the population size and run \ac{svgd} in the latent space. We can see from \cref{tab:abl_Q} that our method is robust to $Q$ by pulling \ac{svgd} sampling back to the latent space. 

\begin{table}[!htb]
\begin{minipage}{0.47\textwidth}
\caption{\textbf{Ablation of $K, M$ on SPRCON. (avg. over 5 runs).}}
\label{tab:abl_K_M}
\begin{center}
\begin{small}
\begin{sc}
\resizebox{\linewidth}{!}{
\begin{tabular}{cccCc}
\toprule
\rowcolor{white}
\multirow{2.5}{*}{$M=100$} & $K=2$ & $K=4$ & {\color{tgray}$K=6$} & $K=8$ \\
\cmidrule(l){2-5}
& $0.427_{\pm 0.013}$
& $0.557_{\pm 0.029}$
& $\mathbf{0.567}_{\pm 0.017}$
& $0.527_{\pm 0.049}$ \\
\midrule
\rowcolor{white}
\multirow{2.5}{*}{$K=6$} & $M=1$ & $M=10$ & {\color{tgray}$M=100$} & $M=1000$ \\
\cmidrule(l){2-5}
& $0.497_{\pm 0.044}$
& $0.486_{\pm 0.023}$
& $\mathbf{0.567}_{\pm 0.017}$
& $0.548_{\pm 0.046}$ \\
\bottomrule
\end{tabular}}
\end{sc}
\end{small}
\end{center}
\vskip -0.1in
\end{minipage} \quad\quad
\begin{minipage}{0.47\textwidth}
\caption{\textbf{Ablation of $Q$ on D'Kitty (avg. over 5 runs).}}
\label{tab:abl_Q}
% \vskip 0.1in
\begin{center}
\begin{small}
\begin{sc}
\resizebox{\linewidth}{!}{
\begin{tabular}{ccccC}
\toprule
\rowcolor{white}
$Q=2$ & $Q=8$ & $Q=32$ & $Q=128$ & {\color{tgray}$Q=256$} \\
\midrule
$0.906_{\pm 0.009}$
& $0.925_{\pm 0.013}$
& $0.937_{\pm 0.006}$
& $0.954_{\pm 0.004}$ 
& $\mathbf{0.961}_{\pm 0.006}$ \\
\bottomrule
\end{tabular}}
\end{sc}
\end{small}
\end{center}
\vskip -0.1in
\end{minipage}
\end{table}

\paragraph{Random baselines} 
One valid concern is that our framework solves the offline \ac{bbo} problem by simply memorizing the best points in the offline dataset and proposing new points close to those best points during evaluation utilizing the randomness in the sampling process. To see whether this is the case, we follow \citet{mashkaria2023generative} to perform a simple experiment on the \textbf{D’Kitty} task. We include random baselines by similarly choosing a small hypercube domain around the optimal point in the offline dataset. These baselines then uniformly sample 256 random points in the cube as the proposed candidates; this can be seen as a offline dataset oracle w/ noise model. In \cref{tab:appx_randb}, we show the results for different widths of this hypercube ranging from $0$ to $0.1$. $0$ width means returning the best point in the offline dataset.
We can see that the best result (226.10) from these noisy oracle models is significantly lower than the result from our model (304.36). We can also see from the trend in max, mean and std. values that simply extending the width of cube for random search is hopeless for solving the offline \ac{bbo} task. This is very likely due to the high-dimensionality and highly-multimodal nature of the black-box function input space. The comparison between these random baselines and our method suggests that our method does find an informative latent space for effectively extrapolating beyond the offline dataset best designs. The prior model in our formulation provides meaningful gradient for exploring the underlying high-function-value manifold.

\begin{table}[!htb]
\caption{\textbf{Random baseline results on the D'Kitty dataset.} Results from 5 runs.}
\label{tab:appx_randb}
\vskip 0.1in
\begin{center}
\begin{small}
\begin{sc}
% \resizebox{\linewidth}{!}{
\begin{tabular}{ccccccC}
\toprule
\rowcolor{white}
Cube Width & $0.00$ & $0.005$ & $0.01$ 
           & $0.05$ & $0.1$ & {\color{tgray} Ours} \\
\midrule
Max$^\uparrow$  & $199.23$ & $212.66$ & $222.44$
     & $226.10$ & $209.00$ & $\mathbf{304.36}$ \\ 
Mean$^\uparrow$ & $199.23$ & $190.68$ & $182.13$ 
     & $-169.62$ & $-368.71$ & $\mathbf{206.76}$\\ 
Std.$^\downarrow$ & $0.00$ & $9.28$ & $12.21$ 
     & $331.00$ & $261.37$ & $\mathbf{63.26}$ \\ 
\bottomrule
\end{tabular}
% }
\end{sc}
\end{small}
\end{center}
\vskip -0.1in
\end{table}

\section*{Use of Large Language Models}
During the preparation of this manuscript, we used Large Language Models (LLMs) to assist with language polishing, organization of sections, and improving clarity of exposition. All technical contributions, including theoretical results, experimental design, and analysis, were conceived and carried out by the authors. The LLMs were not used to generate novel research content, proofs, or experimental results.

%%%%%%%%%%%%%%%%%%%%%%%%%%%%%%%%%%%%%%%%%%%%%%%%%%%%%%%%%%%%%%%%%%%%%%%%%%%%%%%
%%%%%%%%%%%%%%%%%%%%%%%%%%%%%%%%%%%%%%%%%%%%%%%%%%%%%%%%%%%%%%%%%%%%%%%%%%%%%%%

\end{document}